\documentclass[10pt]{article}

\ifdefined\pdfminorversion \pdfminorversion=7 \fi
\ifdefined\pdfvariable \pdfvariable minorversion=7 \fi

\usepackage[preprint]{tmlr}

\usepackage{amsmath,amsfonts,bm}

\def\eqref#1{equation~\ref{#1}}

\def\1{\bm{1}}

\DeclareMathAlphabet{\mathsfit}{\encodingdefault}{\sfdefault}{m}{sl}
\SetMathAlphabet{\mathsfit}{bold}{\encodingdefault}{\sfdefault}{bx}{n}

\usepackage{hyperref}

\usepackage{amsmath,amssymb,amsthm}
\usepackage{upgreek}
\usepackage{float}
\newtheorem{lemma}{Lemma}
\newtheorem{theorem}{Theorem}
\newtheorem{proposition}{Proposition}
\newtheorem{corollary}{Corollary}
\theoremstyle{definition}
\newtheorem{assumption}{Assumption}
\theoremstyle{remark}
\newtheorem{remark}{Remark}
\usepackage{booktabs}
\usepackage{multirow}
\usepackage{graphicx}
\usepackage{url,ulem}
\usepackage{xspace}
\usepackage{xcolor}
\graphicspath{{../}{./}}
\newcommand{\method}{PAGE\xspace}

\newcommand{\pci}[2]{{\scriptsize$[#1,#2]$}}
\newcommand{\yes}{\checkmark}
\newcommand{\no}{--}

\title{PAGE: Partition-Aware Gated KV-Cache Eviction}

\author{\name Pankaj Kumar \email pankaj.kumar@niser.ac.in \\
      \addr School of Computer Sciences\\
      National Institute of Science Education and Research \\
      An OCC of Homi Bhabha National Institute, India
      \AND
      \name Subhankar Mishra \email smishra@niser.ac.in \\
      \addr School of Computer Sciences\\
      National Institute of Science Education and Research \\
      An OCC of Homi Bhabha National Institute, India
      }

\begin{document}
\maketitle

\begin{abstract}


  KV-cache eviction can do more than compress. In long-context LLMs, keeping only some cached tokens sometimes matches or exceeds full-cache accuracy, because many redundant prefill tokens otherwise dilute attention away from the tokens that carry the answer. This benefit is not uniform, and evicting the wrong tokens drops accuracy to zero on tasks that need precise retrieval, so the useful question is not only which tokens to keep but whether to evict this input at all. We show that one label-free number computed from the prefill attention, the drop between early and late layers in how much attention heads agree on which tokens to read, predicts per input, before any decoding, which of the two cases an input falls under. We build this into PAGE (Partition-Aware Gated Eviction), a wrapper that runs any SnapKV-style evictor when the drop is large and keeps the full cache when it is small, with no training, labels, or fine-tuning. PAGE is a safety mechanism rather than a compressor, so we measure it by the failures it prevents. It cuts the harm rate on capacity-bound inputs from 0.75 to 0.026, and on multi-key retrieval with Mistral-7B plain SnapKV falls from 99\% to 0\% as the budget shrinks, while PAGE holds it at 89\%. Elsewhere it passes the base evictor through unchanged, which is the intended behaviour and is what we observe in 8 of 16 cells. The prediction is one-sided, identifying evictable inputs reliably and resolving the single-near-tie-distractor boundary only with per-input information beyond the mean drop. The head-agreement drop orders tasks the same way across four model families, but the threshold does not transfer, and a roughly 100-input unlabeled pilot per model is needed outside Qwen2.5 and Mistral. Realized compression is 1.8$\times$ to 3.4$\times$ against a nominal 16$\times$ budget and decays toward unity by batch 16, and at matched memory a trained evictor outperforms PAGE. Code is available at \url{https://github.com/pankajkumar6002/PAGE}.

\end{abstract}

\section{Introduction}
\label{sec:intro}

Large language models (LLMs) are increasingly used for long-context tasks such as document understanding, retrieval, and extended reasoning. As context length grows, the KV-cache grows linearly with the input and becomes a memory bottleneck during inference. KV-cache eviction addresses this bottleneck by retaining only a subset of the states under a fixed memory budget~\citep{H20,streamingLLM,snapkv,PyramidKV}. Existing methods largely share the same formulation: given a fixed budget, determine \textit{which} KV states to retain. When discarded states are redundant, this can reduce attention dilution and even improve long-context performance~\citep{DBtrimKV,CapKV,indexmem}; however, when the answer depends on a small set of critical tokens, aggressive eviction can instead destroy essential information and severely degrade retrieval.

Recent work has therefore focused on improving which states are retained through stronger importance signals, learned predictors, and information-theoretic objectives~\citep{snapkv,PyramidKV,CapKV,DBtrimKV,learn-to-evict}. While these methods can improve the accuracy--memory trade-off and even outperform full-cache inference by mitigating attention dilution~\citep{DBtrimKV}, they still assume that eviction is appropriate for every input. In particular, adaptive approaches either learn a scorer to estimate which states should be retained~\citep{DBtrimKV,indexmem,ForesightKV,learn-to-evict} or derive an unconditional capacity measure from query statistics~\citep{CapKV}. This leaves two fundamental questions open: \textbf{(Q1)}~\textit{For which inputs does KV-cache eviction help or hurt?} and \textbf{(Q2)}~\textit{Can this distinction be predicted a priori from a cheap inference-time signal before decoding?} These questions are important because a benchmark-level accuracy--memory trade-off can mask a subset of inputs for which eviction is catastrophic.

We answer both affirmatively by reframing KV-cache eviction as an \textit{input-dependent admission decision}: rather than assuming that eviction should be applied to every input, we ask whether the model's prefill attention contains a signal indicating when eviction is likely to be beneficial or harmful. We find that the early-to-late drop in mean pairwise top-$k$ head agreement, denoted by $D$, provides a useful label-free signal that partitions inputs into a dilution-prone class $\mathcal{D}$ and a capacity-bound class $\mathcal{C}$ before any decoding (Figure~\ref{fig:PAGE_setup}). Based on this observation, we introduce \textbf{PAGE} (\textbf{P}artition-\textbf{A}ware \textbf{G}ated \textbf{E}viction), a training-free wrapper that applies an existing eviction method when $D$ exceeds a threshold and otherwise retains the full cache. 
Since a benchmark mean can hide exactly this failure mode, we evaluate \method by what it prevents on the capacity-bound regime rather than by an average over tasks. It cuts the rate at which eviction destroys a correct answer by $29\times$ at $16\times$ compression, and at matched \textit{achieved} cache it lifts the accuracy-memory frontier above a crossover near $1/3$ kept cache, losing to plain eviction below it (Figure~\ref{fig:pareto-gated}). The same selectivity has a converse. The gate is inactive wherever eviction is already safe, so on inputs that eviction does not harm the gated and plain systems become identical and their measured difference is exactly zero. \method is designed to prevent harmful eviction, not to maximize average compression.

\begin{figure}[t]
  \centering
  \includegraphics[width=0.875\linewidth]{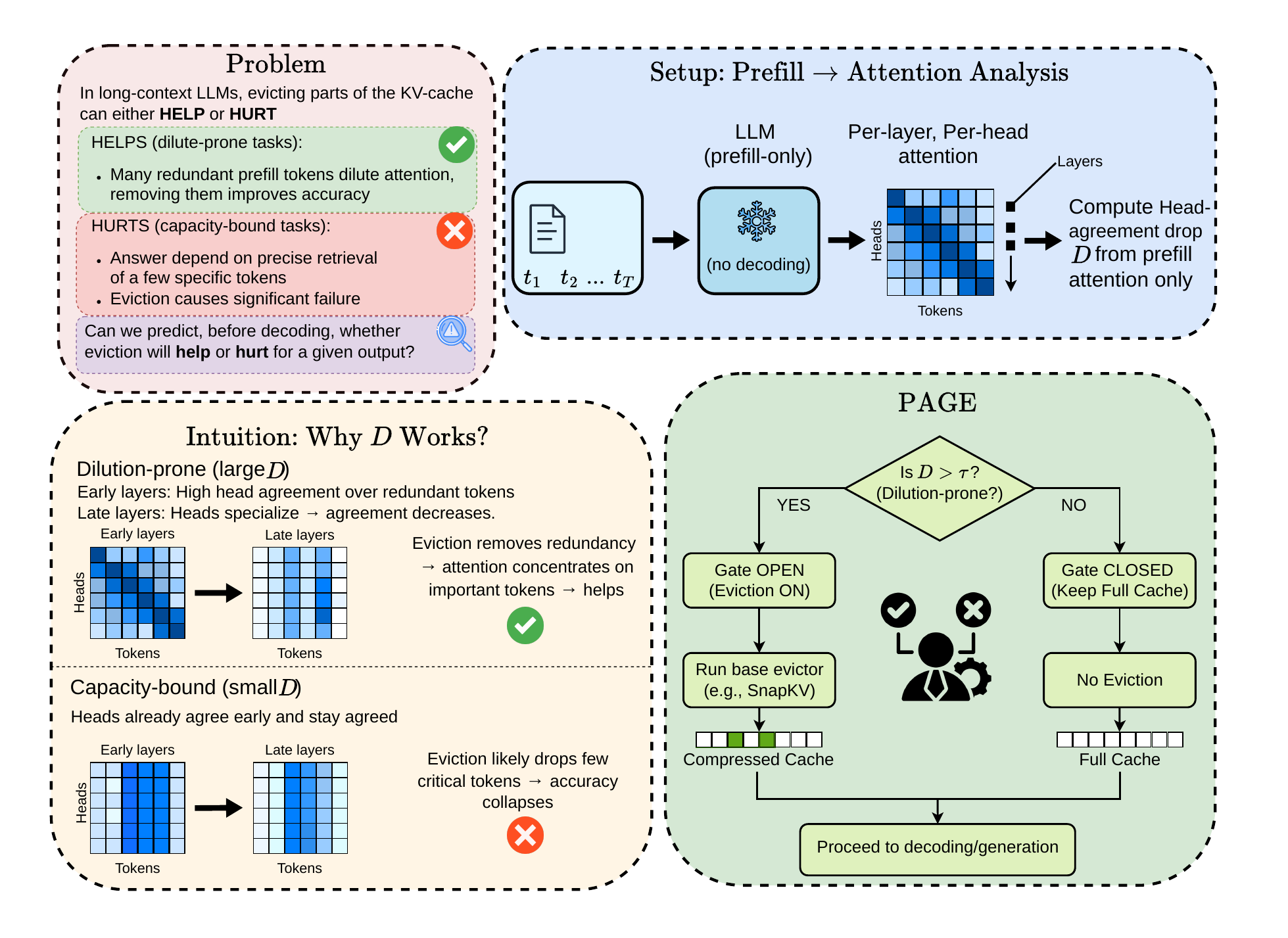}
\caption{\method uses prefill-only attention to compute the head-agreement drop $D$ and decide \textit{whether} to evict. A large $D$ indicates dilution-prone inputs, for which \method opens the gate and applies a base evictor, and a small $D$ indicates capacity-bound inputs, for which the gate remains closed and the full cache is retained.}
  \label{fig:PAGE_setup}
\end{figure}

\paragraph{Contributions}
\begin{enumerate}
  \item \textbf{An empirical partition} of inputs into $\mathcal{C}$ and $\mathcal{D}$, demonstrated across eight (model, context) cells on RULER~\citep{Ruler} and confirmed by a controlled distractor sweep. The early-to-late head-agreement drop $D := \mathrm{agr}_{\text{early}} - \mathrm{agr}_{\text{late}}$ orders tasks by eviction safety and separates the two classes at the task level (Section~\ref{subsec:analysis}).

  \item \textbf{\method, a training-free gating wrapper} that applies any base evictor when $D \ge \tau$ and retains the full cache otherwise, with a $z$-score calibration recipe that transfers across four architecture families (Section~\ref{sec:method}).

  \item \textbf{A substrate-dependent capacity limit}: the three scorer families we test (attention-based, geometric, attention-free) all collapse on capacity-bound inputs at aggressive budgets, evidence of an information floor that no scorer among these families removes (Section~\ref{subsec:scorer-independence}).

  \item \textbf{A scaling relation} $\rho \approx (1 - A_{\mathrm{full}}) \cdot p_{\mathrm{recov}}$, validated by head-to-heads with CapKV and DBTrimKV and a pre-registered 32K test that preserves ordinal content but falsifies constant-ratio proportionality (Sections~\ref{subsec:analysis},~\ref{sec:theory}).
\end{enumerate}

We adopt the dilution mechanism of \citet{DBtrimKV} (Proposition~3.1, Corollary~3.2) and repurpose a known head-agreement primitive as a gating signal; Section~\ref{sec:related} positions these choices. \method is a safeguard, not a compressor: at matched memory a trained evictor wins (Section~\ref{sec:limitations}), it is inactive wherever eviction is already safe ($\Delta = 0$ in half the matrix cells), and the $+22.9$pp headline is driven by one task family (excluding NIAH-MK3: $+4.5$pp; task-label oracle: $+20.1$pp). Realized compression is $1.8$--$3.4\times$ (mean $2.9\times$) and decays toward $1\times$ by batch~16 under static provisioning.

\section{Related Work}
\label{sec:related}

KV-cache eviction methods primarily differ in how they identify states to retain under a fixed budget. H2O~\citep{H20}, StreamingLLM~\citep{streamingLLM}, SnapKV~\citep{snapkv}, and PyramidKV~\citep{PyramidKV} use attention, recency, or layer-dependent statistics for retention. Recent methods develop more specialized selection criteria: DBTrimKV~\citep{DBtrimKV} characterizes attention dilution from near-tie distractors and uses selective retention to mitigate it, while CapKV~\citep{CapKV} derives an information-theoretic eviction criterion. IndexMem~\citep{indexmem} combines learned indexing with latent memory, and ForesightKV~\citep{ForesightKV} and Learning-to-Evict~\citep{learn-to-evict} predict future KV utility. VECTOR~\citep{simplepluginimprovingevictionbased} and CriticalKV~\citep{criticalKV} instead provide plug-and-play mechanisms over existing eviction policies. These methods retain the conventional formulation of selecting which states to preserve once eviction is applied.

Adaptive eviction methods further condition retention or budget allocation on the input or task. DynamicKV~\citep{DynamicKV} uses task-dependent attention statistics to allocate layer-wise KV budgets, while EvicPress~\citep{evicpressjointkvcachecompression} incorporates a context-level eviction decision into a profiled quality-latency optimization. Attention-Gate~\citep{zeng2025incontextkvcacheevictionllms} and Fast KVzip~\citep{fastkvzip} learn adaptive gating or budget policies. SAGE-KV~\citep{sagekv} is particularly close in its use of post-prefill attention, using a one-shot attention signal to select tokens and heads for retention. PAGE instead uses the prefill signal to characterize whether the eviction operation should be applied, while the underlying retention policy remains unchanged.

A complementary line of work studies the internal attention structure associated with retrieval and compression. Retrieval-head analyses identify sparse subsets of heads that drive long-context retrieval~\citep{wu2025retrieval}, while \citet{understandingphysicskeyvaluecache} analyze how head consensus evolves across layers and relate these profiles to compression tolerance. Chen et al.~\citep{pitfalls} document task-dependent degradation under eviction, and Zhang et al.~\citep{valueawarekveviction} study when value-aware eviction is beneficial. VaSE~\citep{chang2026valueaware} attributes some eviction failures to high-magnitude value states, whereas ManifoldKV~\citep{manifoldkv} considers geometric structure in key representations. PAGE draws on this diagnostic perspective by using the early-to-late change in head agreement as an input-level prefill statistic rather than as a retention score or an offline architectural descriptor.


\section{The Gating Method: PAGE}
\label{sec:method}

\paragraph{Problem Formulation}
\label{subsec:formulation}

For task $\mathsf{T}$, model $M$, context length $T$, let $A(b, x)$ denote the accuracy on input $x$ at cache budget $b$, with $b_{\max} = 1$ denoting the full, unevicted cache. We call an input \textit{capacity-bound} if eviction at every tested budget degrades its accuracy relative to the full cache, and \textit{dilution-prone} if there exists some budget at which eviction improves accuracy. The key evaluation criterion is the per-input recovery rate:
\begin{equation}
\rho_{\mathrm{KV}}(\mathsf{T}, M, T) := \Pr_{x}\bigl[A(b^{\ast}(x), x) > A(b_{\max}, x)\bigr],
\qquad b^{\ast}(x) := \arg\max_{b} A(b, x).
\label{eq:rho}
\end{equation}
This per-input strict-Pareto improvement frequency is a stronger statement than ``mean accuracy is higher at some $b$.'' We pre-registered $\rho \ge 0.05$ as ``dilution-prone'' definition before any sweep.

\paragraph{Algorithm}
\label{subsec:algorithm}

\method wraps any SnapKV-style evictor in one conditional. Given a prompt $x$, the base evictor's budget $b \in (0, 1]$ (the fraction of prompt positions, the evictor is allowed to retain), a sink-and-recency protection $(n_{\text{sink}}, w)$ that always preserves the first $n_{\text{sink}}$ and last $w$ positions, and a gate threshold $\tau$ on the head-agreement drop:
\begin{equation}
\boxed{\
\method(x;\, \text{base}, b, \tau) =
\begin{cases}
\text{base.evict}(\text{cache}, b) & \text{if } D(\text{cache}.\text{attentions}) \ge \tau \\
\text{cache} & \text{otherwise}
\end{cases}\
}
\end{equation}
where the prefill produces \texttt{cache} together with its prefill attentions, from which we compute $D$ as in Eq.~\ref{eq:D}. The gate fires once per input, at prefill, before any decode step, and involves no backward pass, no training, and no per-token decision.


\begin{figure}[t]
  \centering
  \includegraphics[width=0.60\linewidth]{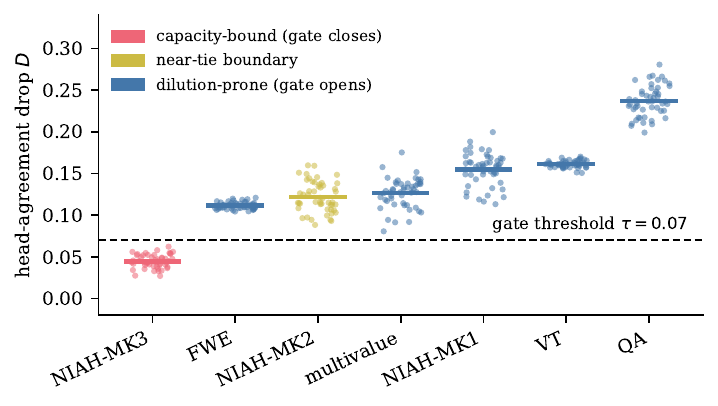}
  \caption{$D$ separates the two task classes (Qwen2.5-1.5B, RULER 4K, $N = 50$ per task; dots are inputs, bars per-task means). NIAH-MK3 sits below $\tau = 0.07$ and the dilution-prone tasks above it. The near-tie MK1/MK2 tasks fall on the boundary, tracking the distractor gradient (App.~\ref{app:deferred}).}
  \label{fig:phenomenon}
\end{figure}



\paragraph{Head-Agreement Drop $D$}
\label{subsec:D-definition}

Take the prefill attention tensor $[L, H, T, T]$, for $L$ layers, $H$ heads per layer and prompt length $T$. Project onto the last $w = 32$ rows, the observation queries, and average across them for a per-(layer, head, key) attention vector. For each layer and head pair we take the Jaccard similarity of the two heads' top-$k$ key sets ($k = 32$) and average over pairs, giving a per-layer agreement $a_\ell$. Binning the $L$ layers into thirds and writing $a_{\text{early}}$, $a_{\text{late}}$ for the mean agreement in the first and last bins, we set
\begin{equation}
\label{eq:D}
D := a_{\text{early}} - a_{\text{late}}.
\end{equation}
$D$ is one scalar per input, costing $O(LH^2 k)$: $17$ms on Qwen2.5-1.5B and $89$ms on Mistral-7B at 4K (Table~\ref{tab:latency}). It is not free even when the base evictor already scores by attention, since that scoring uses head-\textit{averaged} attention whereas $D$ needs the per-head pattern, which requires materializing attention that fused FlashAttention~\citep{flash-attention} does not produce. A same-inputs ablation against five other cheap prefill statistics justifies the cost: $D$ is the only one that separates the partition, every alternative failing task-level separation on the clear-cut NIAH-MK3 anchor. Table~\ref{tab:signal-ablation} also tracks how that separation narrows as the anchor moves toward the near-tie MK2 boundary.

\label{subsec:signal-ablation}


\paragraph{Calibration}
\label{subsec:calibration}

We distinguish two senses of label-free. \method is accuracy-label-free because the gate never observes correctness. A calibration variant is additionally task-label-free when it also requires no task identities to set $\tau$. The two senses come apart across the variants below.

The fixed $\tau = 0.07$ used throughout is a within-family convenience. It mis-classifies Qwen2.5-3B at 16K and fails outright on Qwen3. The recommended variant standardizes $D$ per model over an unlabeled ${\sim}100$-input pilot, and is task-label-free. It transfers a single global threshold zero-shot to every held-out family tested: Llama-architecture models (Yi-1.5-9B, Llama-3.1-8B) and Qwen3-4B~\citep{Yi_tc,llama_tc,qwen_tc}. The intermediate 40-input variant of Eq.~\ref{eq:tau-app} is accuracy-label-free but not task-label-free, because it uses task identities to form $\bar D_{\mathcal C}$ and $\bar D_{\mathcal D}$. We report fixed-$\tau$ numbers as the conservative default. Table~\ref{tab:tau-variants} maps the three variants to their label requirements and to the results that use each, so every gated number can be traced to the variant that produced it. The one number computed at a threshold other than its reported value is the Qwen2.5-3B 4K SnapKV matrix cell, run at $\tau = 0.04$ and reported at $\tau = 0.07$ by exact post-hoc re-evaluation (Appendix~\ref{app:repro}). Further calibration details appear in Table~\ref{tab:recipe} and Figure~\ref{fig:tau-sensitivity}, Appendix~\ref{app:deferred}.


\section{Experiments}
\label{sec:experiments}



\paragraph{Setup}
\label{subsec:setup}

We use RULER~\citep{Ruler} via the \texttt{simonjegou/ruler}~\citep{ruler-hf} port distributed with kvpress~\citep{kvpress}, whose four-evictor lineup our base evictors follow, at context lengths $T \in \{4\mathrm{K}, 16\mathrm{K}\}$ on four models: Qwen2.5-1.5B-Instruct, Qwen2.5-3B-Instruct, Qwen2.5-14B-Instruct~\citep{qwen2.5}, and Mistral-7B-Instruct-v0.3~\citep{mistral7b}. Eviction is SnapKV-style by default: we score each prompt position by the average attention from the last $w = 32$ queries across heads and layers, then keep the top-$b$ positions with $b \in \{0.0625, 0.125, 0.25, 0.375, 0.5, 0.625, 0.75, 0.875, 1.0\}$, and always preserve the first $n_{\text{sink}} = 4$ sink positions~\citep{streamingLLM} and the last $w = 32$ recent positions. Cache slicing uses explicit \texttt{position\_ids} to preserve rotary position embedding (RoPE) alignment~\citep{roformer}, rather than the attention-mask trick that silently shifts RoPE on evicted positions. Task codes are RULER's: VT (variable tracking), FWE (frequent-words extraction), QA\_1/QA\_2, and NIAH-MK$k$ = \texttt{niah\_multikey\_}$k$ ($k$ near-tie keys, one carrying the answer). 
The matrix uses RULER 4K with the four-task mixed suite ($N = 100$ each) across all sixteen cells. Gating deltas carry $95\%$ cluster-bootstrap intervals, all 16 positive with smallest lower bound $+9.3$pp, and remain significant under Benjamini-Hochberg correction at $q = 0.05$ across the 16 cells (Appendix~\ref{app:ci}). Full protocol and deferred tables are in Appendices~\ref{app:deferred} and~\ref{app:ci}.

\paragraph{The partition} We define the partition as a pooled comparison: across the eight (model, context) cells, only $2/526$ NIAH-MK3 inputs recover under eviction ($\rho = 0.004$, Wilson $[0.001, 0.014]$), whereas five task families reach $\rho \ge 0.05$ at some context length. Because individual cells at these sample sizes are underpowered, we interpret Table~\ref{tab:partition} as this pooled comparison rather than as $48$ separate calls. We restrict the definition of the partition to 4K and 16K. The single 32K test-case yields $\rho=0.04$ for NIAH-MK3, just below the $0.05$ threshold that defines the class (Section~\ref{subsec:analysis}), and consequently falls outside our fitting range. These below-threshold cells arise from two mechanisms. In saturated cells, $A_{\mathrm{full}}\to 1$ leaves little headroom for eviction. Longer contexts create headroom, as predicted by Eq.~\ref{eq:scaling}. The remaining three exceptions are budget-insensitive: headroom exists, yet the same inputs remain incorrect across budgets. These three are all QA cells, consistent with failures arising from parametric knowledge rather than cache allocation (Appendix~\ref{app:deferred}).

\begin{table}[t]
\centering
\caption{Per-input recovery rate $\rho$ on RULER with Wilson $95\%$ intervals (small type). Bold marks point estimates $\rho \ge 0.05$. NIAH-MK3 (precise multi-key retrieval) is the only task that systematically fails the threshold in every cell tested. $N = 100$ per (task, cell) except Qwen-14B ($N = 50$ at 4K, $N = 30$ at 16K), Mistral 16K VT/FWE/QA\_1/NIAH-MK3 ($N = 50$), Qwen-3B 16K NIAH-MK3 ($N = 50$), and Qwen-1.5B 16K NIAH-MK3 ($N = 46$). At these $N$ most single cells straddle the $0.05$ threshold and no per-cell call is individually powered; the claim rests on the pooled comparison, which is unambiguous (NIAH-MK3: $2/526$, $\rho = 0.004$, Wilson $[0.001, 0.014]$). Per-cell $N$ varies with compute budget and with the number of examples available at each context length, not with any post-hoc selection.}
\label{tab:partition}
\resizebox{\textwidth}{!}{%
\begin{tabular}{lcccccccc}
\toprule
 & \multicolumn{2}{c}{Qwen 1.5B} & \multicolumn{2}{c}{Qwen 3B} & \multicolumn{2}{c}{Qwen 14B} & \multicolumn{2}{c}{Mistral} \\
\cmidrule(lr){2-3}\cmidrule(lr){4-5}\cmidrule(lr){6-7}\cmidrule(lr){8-9}
Task & 4K & 16K & 4K & 16K & 4K & 16K & 4K & 16K \\
\midrule
VT               & \textbf{0.06} \pci{.03}{.13} & \textbf{0.19} \pci{.13}{.28} & 0.00 \pci{.00}{.04} & \textbf{0.05} \pci{.02}{.11} & 0.00 \pci{.00}{.07}          & 0.03 \pci{.01}{.17}          & 0.00 \pci{.00}{.04}          & \textbf{0.32} \pci{.21}{.46} \\
FWE              & \textbf{0.06} \pci{.03}{.13} & 0.04 \pci{.02}{.10}          & 0.01 \pci{.00}{.05} & \textbf{0.32} \pci{.24}{.42} & 0.00 \pci{.00}{.07}          & \textbf{0.20} \pci{.10}{.37} & 0.02 \pci{.01}{.07}          & \textbf{0.22} \pci{.13}{.35} \\
QA\_1            & \textbf{0.08} \pci{.04}{.15} & \textbf{0.14} \pci{.09}{.22} & 0.03 \pci{.01}{.09} & \textbf{0.08} \pci{.04}{.15} & \textbf{0.06} \pci{.02}{.16} & \textbf{0.10} \pci{.03}{.26} & 0.04 \pci{.02}{.10}          & 0.00 \pci{.00}{.07} \\
QA\_2            & \textbf{0.08} \pci{.04}{.15} & \textbf{0.05} \pci{.02}{.11} & 0.04 \pci{.02}{.10} & 0.02 \pci{.01}{.07}          & 0.02 \pci{.00}{.11}          & 0.03 \pci{.01}{.17}          & 0.01 \pci{.00}{.05}          & 0.00 \pci{.00}{.04} \\
niah\_multivalue & \textbf{0.07} \pci{.03}{.14} & \textbf{0.10} \pci{.06}{.17} & 0.02 \pci{.01}{.07} & \textbf{0.13} \pci{.08}{.21} & \textbf{0.08} \pci{.03}{.19} & \textbf{0.07} \pci{.02}{.21} & \textbf{0.06} \pci{.03}{.13} & \textbf{0.21} \pci{.14}{.30} \\
NIAH-MK3         & 0.01 \pci{.00}{.05}          & 0.00 \pci{.00}{.08}          & 0.00 \pci{.00}{.04} & 0.02 \pci{.00}{.11}          & 0.00 \pci{.00}{.07}          & 0.00 \pci{.00}{.11}          & 0.00 \pci{.00}{.04}          & 0.00 \pci{.00}{.07} \\
\bottomrule
\end{tabular}%
}
\end{table}

\paragraph{Boundary and Confound Analysis}
\label{subsec:boundary} A single capacity-bound task leaves open whether NIAH-MK3 names a class or an isolated construction, so we tested the boundary on four further RULER tasks chosen by mechanism rather than by outcome (Table~\ref{tab:capbound}, Appendix~\ref{app:capbound-detail}). The results confirm a second capacity-bound member (niah\_multiquery) and show that neither distractor count nor output length defines the class; what matters is whether distractors are near-tie surface-form variants of the answer. Furthermore, three controls rule out the obvious confounds. First, a distractor sweep over the RULER NIAH-MultiKey family confirms the dilution mechanism of \citet{DBtrimKV}: mean $D$ shrinks monotonically as near-tie distractors accumulate (Table~\ref{tab:distractor}). Second, re-running the partition under the $10\%$ prefix and suffix protection of \citet{garcia2026protectionnearlyneedstructural} leaves it unchanged. Third, replacing SnapKV scoring with uniform-random scoring destroys accuracy at every budget, ruling out a generic regularization explanation (Appendix~\ref{app:random}, Table~\ref{tab:random-control}).


\paragraph{External validation}
The partition is not an artifact of synthetic RULER data. On three LongBench subtasks with Qwen2.5-14B (Table~\ref{tab:realistic-workload}, Appendix~\ref{app:deferred}), the gate calls all three classes correctly \textit{a priori} from $D$ alone: lcc (capacity-bound, $\Delta = +0.146$), repobench-p~\citep{liu2024repobench} (mixed), and hotpotqa~\citep{hotpotqa} (evict-robust, $\Delta = 0.000$).
The partition is sufficiently reproducible to motivate asking whether it can be detected before decoding.

\subsection{Analysis of PAGE}
\label{subsec:analysis}

\paragraph{Predicting Unsafe Eviction}
\label{subsec:signal-validation}

Figure~\ref{fig:phenomenon} shows that $D$ separates the two task classes on the fitting cell (Qwen2.5-1.5B, RULER 4K): NIAH-MK3 sits below $\tau = 0.07$ and the dilution-prone tasks above it. Per-input AUC reaches $1.000$ on that cell, dropping to $0.978$, $0.899$, and $0.741$ on three held-out cells (Appendix~\ref{subsec:heldout-ablation}). $D$ is a useful one-sided prefill signal: reliable at closing the gate on capacity-bound inputs, only weakly discriminative on dilution-prone ones. On Yi-1.5-9B and Llama-3.1-8B the drop ordering transfers but $\tau = 0.07$ does not; the $z$-scored variant with an unlabeled pilot maintains AUC $\approx 0.80$ (Table~\ref{tab:drops}).




\paragraph{Catastrophic Failure Prevention}
\label{subsec:showcase}

The cleanest single cell demonstrates the safety property without aggregation. On Mistral 4K NIAH-MK3 (Figure~\ref{fig:mistral-showcase}), plain SnapKV collapses from $0.99$ to $0.00$ at $b = 0.0625$, while gated SnapKV holds $0.89$ at every budget. The gate closes on $90/100$ inputs, with $10/100$ false positives where it opens and eviction destroys the needle. Pooled over the $7$ SnapKV cells at 4K and 16K, gating cuts the fixed-budget harm rate from $0.75$ to $0.026$ on the capacity-bound task (Table~\ref{tab:harm-rate}), yielding a $29\times$ reduction in harmful eviction that stands as the primary safety result.

\begin{figure}[t]
  \centering
  \includegraphics[width=0.55\linewidth]{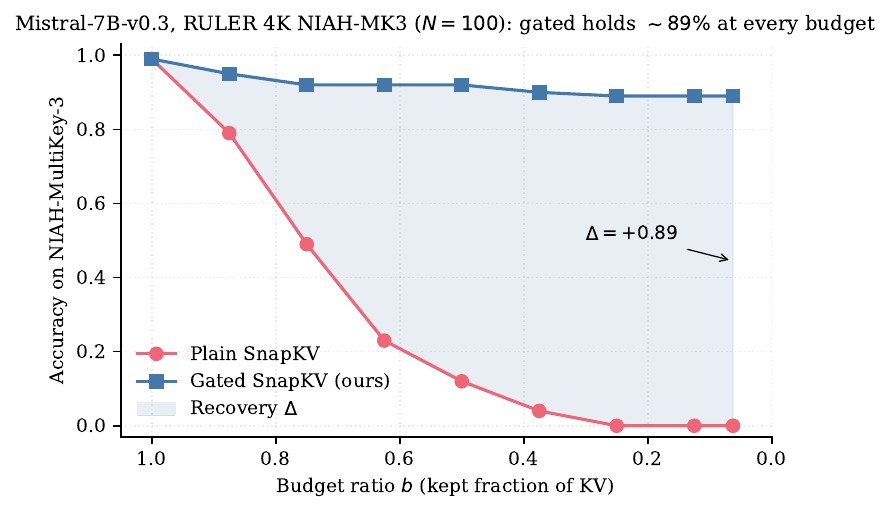}
  \caption{Mistral-7B 4K NIAH-MK3, $N = 100$: accuracy vs.\ KV budget for plain SnapKV (solid, collapses to 0) and gated SnapKV (dashed, 0.89 flat across budgets). The gate prevents the collapse without retraining the base evictor or seeing labels.}
  \label{fig:mistral-showcase}
\end{figure}



\paragraph{Transfer Across Eviction Mechanisms}
\label{subsec:matrix-results}

Table~\ref{tab:matrix} reports a $4 \times 4$ matrix of four base evictors and four models under the same gate, same $\tau$, and no per-method tuning. Each $\Delta$ is gated minus plain accuracy at matched \textit{nominal} budget, not a compression win. The gate composes with several tested eviction mechanisms, which means it works alongside them rather than advancing the frontier on its own. Dropping NIAH-MK3 takes the grand mean from $+22.9$pp to $+4.5$pp. In half the cells the non-MK3 $\Delta$ is exactly $0.000$ because the gate opens on every non-MK3 input there, which is the intended behavior for a gate that should do nothing when eviction is safe. The defensible headline is the harm reduction (Table~\ref{tab:harm-rate}) and the matched-cache frontier (Section~\ref{subsec:pareto}). Composition also extends outside the attention-score family, lifting a geometry-based evictor by $\Delta = +16.2$pp (Figure~\ref{fig:matrix}). Cross-architecture results appear in Tables~\ref{tab:crossarch} and~\ref{tab:bestplain}. Full per-task decomposition, budget-axis breakdown, and cross-architecture tables are in Appendix~\ref{subsec:per-task-delta}.

\begin{table}[t]
\centering
\caption{Headline $4 \times 4$ method-agnostic matrix. Same gate ($\tau = 0.07$), no per-base-evictor tuning; $\Delta$ averages over budgets $b < 1.0$. All $16$ cells are positive with grand mean $+22.9$pp, but the second $\Delta$ row shows what that is made of: excluding NIAH-MK3 the grand mean falls to $+4.5$pp and is exactly $0.000$ in $8$ of $16$ cells, where the gate opens on every non-MK3 input and the two arms are the same system. The Qwen2.5-3B SnapKV cell ran at $\tau = 0.04$ and is reported at $\tau = 0.07$ by exact post-hoc re-evaluation (Appendix~\ref{app:repro}). $^\dagger$Qwen2.5-14B H2O is not an independent cell (see text). Per-task decomposition: Table~\ref{tab:per-task-delta}.}
\label{tab:matrix}
\begin{tabular}{lcccccccc}
\toprule
 & \multicolumn{2}{c}{SnapKV} & \multicolumn{2}{c}{H2O} & \multicolumn{2}{c}{StreamingLLM} & \multicolumn{2}{c}{PyramidKV} \\
\cmidrule(lr){2-3}\cmidrule(lr){4-5}\cmidrule(lr){6-7}\cmidrule(lr){8-9}
Model & plain & gated & plain & gated & plain & gated & plain & gated \\
\midrule
Qwen2.5-1.5B & 0.438 & 0.559 & 0.196 & 0.357 & 0.033 & 0.195 & 0.417 & 0.570 \\
$\Delta$      & \multicolumn{2}{c}{$\!+\!0.121$} & \multicolumn{2}{c}{$\!+\!0.162$} & \multicolumn{2}{c}{$\!+\!0.163$} & \multicolumn{2}{c}{$\!+\!0.152$} \\
{\scriptsize$\Delta$ no MK3} & \multicolumn{2}{c}{\scriptsize$+0.000$} & \multicolumn{2}{c}{\scriptsize$+0.000$} & \multicolumn{2}{c}{\scriptsize$+0.000$} & \multicolumn{2}{c}{\scriptsize$+0.000$} \\
\midrule
Qwen2.5-3B   & 0.580 & 0.839 & 0.339 & 0.642 & 0.052 & 0.482 & 0.536 & 0.865 \\
$\Delta$      & \multicolumn{2}{c}{$\!+\!0.259$} & \multicolumn{2}{c}{$\!+\!0.303$} & \multicolumn{2}{c}{$\!+\!0.430$} & \multicolumn{2}{c}{$\!+\!0.329$} \\
{\scriptsize$\Delta$ no MK3} & \multicolumn{2}{c}{\scriptsize$+0.101$} & \multicolumn{2}{c}{\scriptsize$+0.098$} & \multicolumn{2}{c}{\scriptsize$+0.263$} & \multicolumn{2}{c}{\scriptsize$+0.144$} \\
\midrule
Qwen2.5-14B  & 0.709 & 0.853 & 0.668$^\dagger$ & 0.868$^\dagger$ & 0.080 & 0.330 & 0.632 & 0.860 \\
$\Delta$      & \multicolumn{2}{c}{$\!+\!0.143$} & \multicolumn{2}{c}{$\!+\!0.200$} & \multicolumn{2}{c}{$\!+\!0.250$} & \multicolumn{2}{c}{$\!+\!0.228$} \\
{\scriptsize$\Delta$ no MK3} & \multicolumn{2}{c}{\scriptsize$+0.000$} & \multicolumn{2}{c}{\scriptsize$+0.000$} & \multicolumn{2}{c}{\scriptsize$+0.000$} & \multicolumn{2}{c}{\scriptsize$+0.000$} \\
\midrule
Mistral-7B    & 0.623 & 0.810 & 0.400 & 0.636 & 0.102 & 0.365 & 0.593 & 0.823 \\
$\Delta$      & \multicolumn{2}{c}{$\!+\!0.187$} & \multicolumn{2}{c}{$\!+\!0.236$} & \multicolumn{2}{c}{$\!+\!0.263$} & \multicolumn{2}{c}{$\!+\!0.231$} \\
{\scriptsize$\Delta$ no MK3} & \multicolumn{2}{c}{\scriptsize$+0.015$} & \multicolumn{2}{c}{\scriptsize$+0.020$} & \multicolumn{2}{c}{\scriptsize$+0.053$} & \multicolumn{2}{c}{\scriptsize$+0.020$} \\
\bottomrule
\end{tabular}
\end{table}

\paragraph{Accuracy-Memory Frontier}
\label{subsec:pareto}

The gate falls back to the full cache when it closes, so the honest comparison is at matched \textit{achieved} cache. At matched achieved kept-KV (Figure~\ref{fig:pareto-gated}), gating beats plain SnapKV by $3-19$pp once the target retains more than about a third of the cache, and loses below that where the full-KV fallback dominates. \method recovers $69\%$ of the headroom against a per-input oracle, almost entirely from protecting the capacity-bound anchor. This recovery is bounded by how often the gate opens, since $1 - p_{\text{open}}$ of the cache stays resident regardless of the nominal budget $b$, and so memory should be planned against that floor (Table~\ref{tab:achieved-compression}).

The same floor sets the realized memory benefit. Against a $16\times$ nominal budget, realized compression is $1.8$--$3.4\times$ with mean $2.9\times$, because the resident $1 - p_{\text{open}}$ fraction limits what eviction removes. Under static provisioning, where one gate-closed sequence forces the batch to the full cache, this benefit decays toward $1\times$ by batch~16 (Table~\ref{tab:achieved-compression}). \method is therefore a per-input safeguard rather than a compressor, with a memory benefit that is real at moderate batch and small kept-KV targets and negligible once either the batch is large or the target is aggressive.

\paragraph{Scorer Independent of Capacity Floor}
\label{subsec:scorer-independence}

The gate's value depends on whether the capacity-bound regime is a property of the input rather than the scorer. On Qwen2.5-1.5B RULER 4K, a faithful ManifoldKV+Ada-KV implementation~\citep{manifoldkv,Adakv} and the attention-free KeyDiff~\citep{keydiff} both fall to chance on NIAH-MK3 below $25\%$ budget, like the four attention-score evictors. A geometry scorer that reaches $0.72$ on two-key NIAH against SnapKV's $0.16$ still collapses on three-key retrieval. The capacity limit is therefore substrate-dependent rather than universal; full tables and the Llama-3.1-8B rescue result are in Appendix~\ref{subsec:scorer-independence-app} (Table~\ref{tab:scorer-independence}).

\paragraph{Per-Input Gating Beyond Task-Level Routing}
\label{subsec:oracle}

A natural objection is that a simple task-level policy could achieve the same result. The task-label oracle closes on every NIAH-MK3 input and opens on every other, the best any task-level policy can achieve without test-time labels. In 8 of 16 cells \method reproduces the oracle exactly; in 5 cells it \textit{beats} the oracle by up to $+0.198$ by closing on dilution-prone inputs the oracle wrongly opens, and in 3 Mistral cells it loses by at most $0.009$ through false positives on MK3. Overall the oracle reaches $+20.1$pp against \method's $+22.9$pp. Task-level routing captures a substantial portion of the effect, but can fail on heterogeneous tasks where within-task variation makes per-input resolution useful. Full per-cell results appear in Appendix~\ref{subsec:per-task-delta} (Table~\ref{tab:task-oracle}).

\paragraph{Comparison with Published Methods and a Pre-Registered 32K Test}
\label{subsec:head-to-heads}
\label{subsec:prereg32k}

Wrapping CapKV~\citep{CapKV} on LongBench yields $+1.5$/$+8.3$pp on two subtasks ($N = 68$ and $24$), though sample sizes are small and no formal claim is made. Wrapping pretrained DBTrimKV~\citep{DBtrimKV} improves matched-budget accuracy by $+23.3$pp, but the gate fires on only $1.7\%$ of records, so at matched \textit{memory} plain DBTrimKV wins. \method therefore does not advance the frontier over a strong trained method at matched memory, and its value lies in the label-free per-input decision to retain the full cache when eviction is unsafe. Details appear in Appendix~\ref{app:experiment-details}. A pre-registered out-of-range test at 32K (Qwen2.5-1.5B, $N=100$) fits $\rho/(1-A_{\mathrm{full}})$ on 4K/16K and hashes predicted intervals before sweeping. One of three predictions lands strictly inside its interval while both misses over-recover in different ways. What survives is ordinal rather than quantitative, as MK3's ratio stays far below VT's at every context length even though constant-ratio proportionality fails. We accordingly claim a directional relation rather than a quantitative law (Table~\ref{tab:prereg32k}, Appendix~\ref{app:deployment-details}).

\begin{figure}[t]
  \centering
  \includegraphics[width=0.6\linewidth]{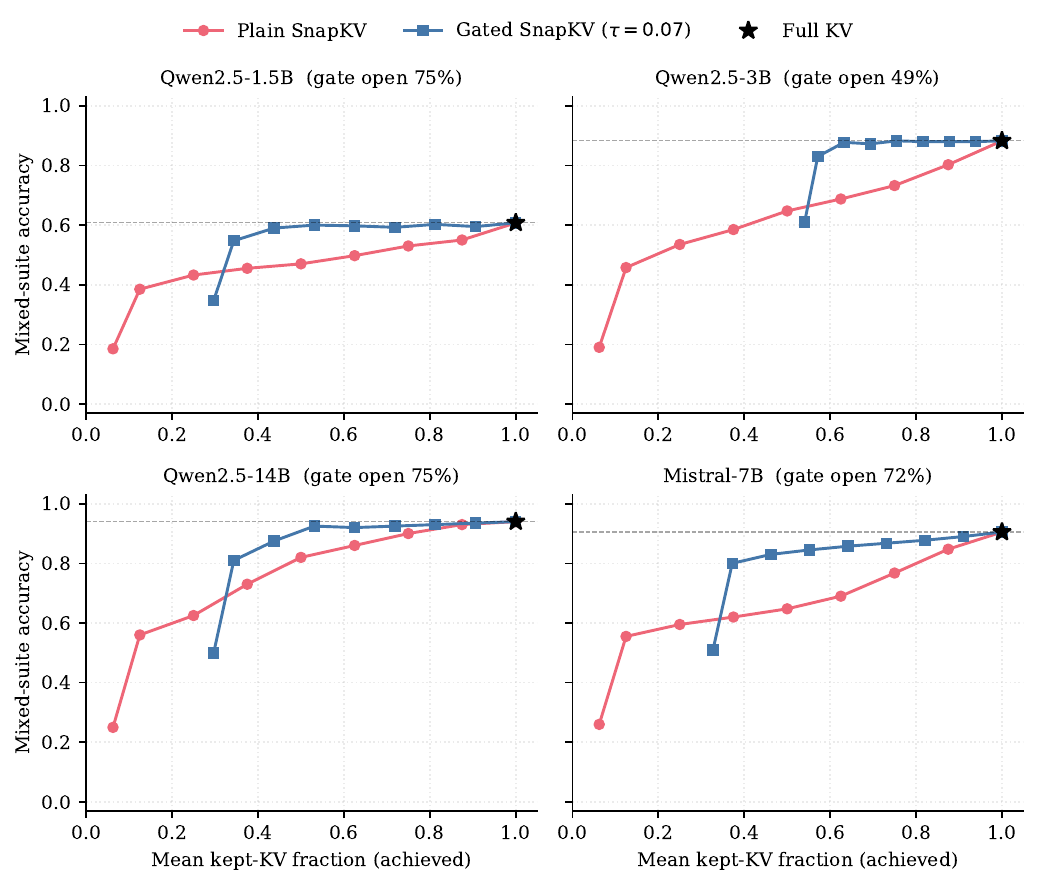}
  \caption{Accuracy vs.\ \textit{achieved} kept-KV fraction on the mixed 4K suite, SnapKV. The gate falls back to the full cache when it closes, so gated points sit at larger kept-KV than their nominal budget and the honest comparison is at matched achieved cache. Above a target of $\approx 1/3$ kept-KV (Qwen3B: $0.57$) the gated curve is the Pareto frontier by up to $+19$pp; at the leftmost gated point plain eviction dominates it. Gated results use the post-hoc $\tau = 0.07$ reconstruction (Appendix~\ref{app:repro}).}
  \label{fig:pareto-gated}
\end{figure}
\section{Mechanistic Interpretation and Scaling Law}
\label{sec:theory}


\paragraph{Dilution Mechanism}
\label{subsec:dilution}

We adopt the dilution mechanism of \citet{DBtrimKV} directly and claim no theoretical novelty. For a head reading a relevant set $\mathcal{R}$ against noise $\mathcal{N}$ with pre-eviction signal mass $\alpha_{\mathcal{R}}$, near-tie distractors force the dilution $\delta_t$ large and preferential retention reduces it. Their Proposition~3.1 establishes that near-tie distractors force attention dilution in a single head, and Corollary~3.2 shows preferential retention of high-attention positions reduces it. The partition follows from the mechanism: capacity-bound tasks concentrate the signal ($\alpha_{\mathcal{R}} \to 1$), so imperfect scoring can only \textit{lose} signal and eviction strictly degrades the head; dilution-prone tasks spread it ($\alpha_{\mathcal{R}} \ll 1$), so eviction of noise positions monotonically improves the head Signal-Noise Ratio (SNR). 


\paragraph{Scaling Relation}
\label{subsec:scaling-relation}

Combining the dilution mechanism with the headroom $1 - A_{\mathrm{full}}(M,T)$ gives the scaling formula:
\begin{equation}
\label{eq:scaling}
\rho_{\mathrm{KV}}(\mathsf{T}, M, T) \;\approx\;
\bigl(1 - A_{\mathrm{full}}(M, T)\bigr) \cdot p_{\mathrm{recov}}\bigl(\rho_{\mathcal{R}}^2 / \gamma\bigr),
\qquad p_{\mathrm{recov}}: [0, \infty) \to [0, 1],
\end{equation}
where $\rho_{\mathcal{R}}$ is the retained fraction of relevant mass and $\gamma$ the noise-retention ratio. The alignment between $D$ and the partition is empirical: $D$ ranks NIAH-MK3 smallest in every Qwen and Mistral cell (Table~\ref{tab:drops}), and $\rho / (1 - A_{\mathrm{full}})$ fits a through-origin slope on dilution-prone tasks while remaining near zero on NIAH-MK3 regardless of headroom (Table~\ref{tab:scaling-formula}, Figure~\ref{fig:scaling}). The pre-registered 32K test preserves the ordinal structure of these results, yet it disproves the constant-ratio proportionality implied by a direct interpretation of Eq.~\ref{eq:scaling}.

The theoretical bridge from $D$ to $\delta_t$ requires a margin-to-noise ratio roughly three times the empirically observed values. The sufficient-condition analysis therefore motivates the mechanism, but does not validate the deployed gate. Consequently, the paper's empirical claims do not depend on this theoretical certification, instead rest on the observed partition and the direct gating results. We provide full derivations in Appendices~\ref{app:theory-derivations}, \ref{app:proof-scaling}, and~\ref{app:sufficient}.


\section{Limitations}
\label{sec:limitations}

\paragraph{Architectural scope}
$D$ presumes softmax attention over a materialized key-value cache and does not apply to linear-attention or state-space architectures. Attempting to run it on layers that expose no key-value cache in the standard form fails outright rather than degrading, which bounds the family of models these results speak to.

\paragraph{The gate is inactive on dilution-prone-only workloads}
On the dilution-prone LongBench subset (qasper, triviaqa, trec, multifieldqa\_en) the pooled gated $\Delta$ is within $\pm 1$pp of plain on all four models (Qwen2.5-1.5B $0.000$, Qwen2.5-3B $+0.002$, Mistral-7B $+0.003$, Qwen2.5-14B $-0.008$). The partition predicts neutrality where a workload contains no capacity-bound inputs, and we observe it, including the mildly negative tail.

\paragraph{One-sided prediction and remaining scope gaps}
On LongBench passage\_count at Qwen2.5-14B the gate closes on all inputs (mean $D = -0.013$) yet eviction is empirically safe, producing a harmless false positive that costs compression rather than accuracy. The symmetric failure arises at the single-distractor boundary: NIAH-MK2 is empirically capacity-bound yet its mean $D = 0.122$ co-locates with genuinely dilution-prone tasks, so the gate opens and eviction proceeds as a false negative that the mean-$D$ threshold cannot resolve without per-input recovery information (Appendix~\ref{app:distractor-detail}). $D$ is therefore a one-sided predictor that is reliable at detecting dilution-prone inputs but not a perfect classifier. A related consequence for deployment is that the gate-open rate $p_{\text{open}}$ is workload-dependent and not known before running, so static provisioning cannot size the cache from $p_{\text{open}}$ in advance and must either profile it per workload or plan against the $1 - p_{\text{open}}$ resident floor. Multilingual and code tasks remain out of scope.

\paragraph{Threshold portability}
The portability of the empirical partition depends on the fixed threshold that operationalizes it. On Yi-1.5-9B and Llama-3.1-8B, the drop ordering transfers but $\tau = 0.07$ does not. It mis-classifies niah\_multivalue on both models and inverts the drop-vs-$\rho$ prediction on Yi. A $z$-scored variant repairs this with an unlabeled pilot, but it remains a calibration step, and per-input separation stays at AUC $\approx 0.80$. The residual gap traces to interior profile structure that the early-minus-late difference cannot distinguish. The statistic $D$ summarizes the per-layer agreement sequence $a_\ell$ by its endpoints, discarding the shape between them. A statistic that reads the full depth-profile rather than just its endpoints would be a stronger per-input predictor and is an open problem (Appendix~\ref{subsec:layer-profile}).

\section{Conclusion}
\label{sec:conclusion}

We show that KV-cache eviction is heterogeneous: inputs fall into a capacity-bound regime where compression causes severe degradation and a dilution-prone regime where selective retention can preserve or improve accuracy. The early-to-late head-agreement drop $D$ separates these regimes from the prefill alone, supporting this claim strongly on Qwen and Mistral and partially across architectures with per-model calibration. \method wraps any base evictor in a gate on $D$ and prevents catastrophic failures: on Mistral-7B NIAH-MK3 a $99\% \to 0\%$ collapse under plain eviction becomes a flat $89\%$ under the gate, with the same protection on realistic lcc code completion and a lifted accuracy-memory frontier where moderate compression is the target. These gains come with real costs: the safeguard requires one attention-exposing prefill pass, its memory benefit is bounded by $1 - p_{\text{open}}$ and decays toward $1\times$ at batch $16$ under static provisioning. The open problem it leaves is a signal that transfers across model families without a per-model pilot.


\subsection*{AI Use Statement}

In this work, we used generative AI tools to assist in the writing of proofs. We formulated every mathematical claim and supplied the ingredients of its argument, and AI assistance was limited to writing up the resulting steps. Among the remaining tasks carrying required disclosure, we did not use generative AI tools to generate synthetic data sets, develop theoretical models or conceptual frameworks, propose or refine hypotheses, design or provide feedback on research methodology or experiments, implement methods, clean or reformat datasets, or interpret results. Translation and qualitative or thematic data analysis are not applicable to this work. Among tasks carrying recommended disclosure, we used generative AI tools to draft parts of this paper and to edit the manuscript for readability. We have reviewed all AI-assisted work: every proof in Appendices~\ref{app:proof-scaling},~\ref{app:sufficient}, and~\ref{app:expected-agreement} was checked line by line by the authors against the assumptions and scope statements of Appendix~\ref{app:theory-derivations}, and all AI-drafted prose was revised by the authors and verified against the logs and tables it describes. We take responsibility for the final content of this work, including text, claims or artifacts produced with the aid of generative AI.

\subsection*{Reproducibility Statement}

Code and per-input logs are available at \url{https://anonymous.4open.science/r/PAGE-018239}. All runs use NVIDIA A100 80GB GPUs with PyTorch 2.11, Transformers 5.9, and Python 3.13, in bfloat16 with greedy (argmax) decoding, \texttt{max\_new\_tokens} $= 128$, and a fixed seed. Models are the public Hugging Face checkpoints of Qwen2.5-1.5B/3B/14B/32B-Instruct, Mistral-7B-Instruct-v0.3, Yi-1.5-9B-Chat, Llama-3.1-8B-Instruct, and Qwen3-4B-Instruct-2507, and tasks come from \texttt{simonjegou/ruler}~\citep{ruler-hf,kvpress} and LongBench~\citep{longbench}.

The remaining material is in the appendix. Appendix~\ref{app:repro} collects the implementation details needed for exact reproduction, including the cache-slicing and long-context conventions and the layer-binning convention behind the reported values of $D$. Appendix~\ref{app:experiment-details} describes how inputs are drawn and how each experiment is constructed, and Appendix~\ref{app:wrapper-provenance} documents the diagnostic wrappers together with the code-path and hardware effects measured on them, including one wrapper that failed validation and whose numbers are therefore not reported. Variation is bounded in two places. Appendix~\ref{subsec:seed-variance} reports three input-draw replicates on five SnapKV cells, and Appendix~\ref{app:ci} gives confidence intervals for the headline numbers. Each table caption states its per-cell sample sizes. The two published head-to-heads run in the baselines' own pinned environments, which Appendix~\ref{app:deferred} records with the released logs.

For the theoretical claims, Appendix~\ref{app:theory-derivations} states the mechanism, the assumptions, and the notation, with full proofs in Appendices~\ref{app:proof-scaling}, \ref{app:sufficient}, and \ref{app:expected-agreement}, and sharpened treatments of three deferred items in Appendix~\ref{app:sharpened}. Each part ends with a scope statement saying what it does and does not establish. The gate is one thresholded scalar (Eq.~\ref{eq:D}) and decoding is greedy, so every gated number in the paper can be recomputed from the released per-input logs without rerunning a model.

\subsection*{Ethics Statement}

This work studies inference-time efficiency of publicly released language models on public benchmarks. It involves no human subjects, no personal data, and no new model training. We see no ethical risk specific to this method beyond those of the underlying models.

\bibliographystyle{tmlr}
\bibliography{main}

@misc{BudgetThinker,
  title         = {BudgetThinker: Empowering Budget-aware LLM Reasoning with Control Tokens},
  author        = {Hao Wen and Xinrui Wu and Yi Sun and Feifei Zhang and Liye Chen and Jie Wang and Yunxin Liu and Yunhao Liu and Ya-Qin Zhang and Yuanchun Li},
  year          = {2025},
  eprint        = {2508.17196},
  archiveprefix = {arXiv},
  primaryclass  = {cs.LG},
  url           = {https://arxiv.org/abs/2508.17196}
}

@misc{qwen_tc,
  title         = {Qwen3 Technical Report},
  author        = {An Yang and Anfeng Li and Baosong Yang and Beichen Zhang and Binyuan Hui and Bo Zheng and Bowen Yu and Chang Gao and Chengen Huang and Chenxu Lv and Chujie Zheng and Dayiheng Liu and Fan Zhou and Fei Huang and Feng Hu and Hao Ge and Haoran Wei and Huan Lin and Jialong Tang and Jian Yang and Jianhong Tu and Jianwei Zhang and Jianxin Yang and Jiaxi Yang and Jing Zhou and Jingren Zhou and Junyang Lin and Kai Dang and Keqin Bao and Kexin Yang and Le Yu and Lianghao Deng and Mei Li and Mingfeng Xue and Mingze Li and Pei Zhang and Peng Wang and Qin Zhu and Rui Men and Ruize Gao and Shixuan Liu and Shuang Luo and Tianhao Li and Tianyi Tang and Wenbiao Yin and Xingzhang Ren and Xinyu Wang and Xinyu Zhang and Xuancheng Ren and Yang Fan and Yang Su and Yichang Zhang and Yinger Zhang and Yu Wan and Yuqiong Liu and Zekun Wang and Zeyu Cui and Zhenru Zhang and Zhipeng Zhou and Zihan Qiu},
  year          = {2025},
  eprint        = {2505.09388},
  archiveprefix = {arXiv},
  primaryclass  = {cs.CL},
  url           = {https://arxiv.org/abs/2505.09388}
}

@misc{Yi_tc,
  title         = {Yi: Open Foundation Models by 01.AI},
  author        = {01. AI and : and Alex Young and Bei Chen and Chao Li and Chengen Huang and Ge Zhang and Guanwei Zhang and Guoyin Wang and Heng Li and Jiangcheng Zhu and Jianqun Chen and Jing Chang and Kaidong Yu and Peng Liu and Qiang Liu and Shawn Yue and Senbin Yang and Shiming Yang and Wen Xie and Wenhao Huang and Xiaohui Hu and Xiaoyi Ren and Xinyao Niu and Pengcheng Nie and Yanpeng Li and Yuchi Xu and Yudong Liu and Yue Wang and Yuxuan Cai and Zhenyu Gu and Zhiyuan Liu and Zonghong Dai},
  year          = {2025},
  eprint        = {2403.04652},
  archiveprefix = {arXiv},
  primaryclass  = {cs.CL},
  url           = {https://arxiv.org/abs/2403.04652}
}

@misc{llama_tc,
  title         = {The Llama 3 Herd of Models},
  author        = {Aaron Grattafiori and Abhimanyu Dubey and Abhinav Jauhri and Abhinav Pandey and Abhishek Kadian and Ahmad Al-Dahle and Aiesha Letman and Akhil Mathur and Alan Schelten and Alex Vaughan and Amy Yang and Angela Fan and Anirudh Goyal and Anthony Hartshorn and Aobo Yang and Archi Mitra and Archie Sravankumar and Artem Korenev and Arthur Hinsvark and Arun Rao and Aston Zhang and Aurelien Rodriguez and Austen Gregerson and Ava Spataru and Baptiste Roziere and Bethany Biron and Binh Tang and Bobbie Chern and Charlotte Caucheteux and Chaya Nayak and Chloe Bi and Chris Marra and Chris McConnell and Christian Keller and Christophe Touret and Chunyang Wu and Corinne Wong and Cristian Canton Ferrer and Cyrus Nikolaidis and Damien Allonsius and Daniel Song and Danielle Pintz and Danny Livshits and Danny Wyatt and David Esiobu and Dhruv Choudhary and Dhruv Mahajan and Diego Garcia-Olano and Diego Perino and Dieuwke Hupkes and Egor Lakomkin and Ehab AlBadawy and Elina Lobanova and Emily Dinan and Eric Michael Smith and Filip Radenovic and Francisco Guzmán and Frank Zhang and Gabriel Synnaeve and Gabrielle Lee and Georgia Lewis Anderson and Govind Thattai and Graeme Nail and Gregoire Mialon and Guan Pang and Guillem Cucurell and Hailey Nguyen and Hannah Korevaar and Hu Xu and Hugo Touvron and Iliyan Zarov and Imanol Arrieta Ibarra and Isabel Kloumann and Ishan Misra and Ivan Evtimov and Jack Zhang and Jade Copet and Jaewon Lee and Jan Geffert and Jana Vranes and Jason Park and Jay Mahadeokar and Jeet Shah and Jelmer van der Linde and Jennifer Billock and Jenny Hong and Jenya Lee and Jeremy Fu and Jianfeng Chi and Jianyu Huang and Jiawen Liu and Jie Wang and Jiecao Yu and Joanna Bitton and Joe Spisak and Jongsoo Park and Joseph Rocca and Joshua Johnstun and Joshua Saxe and Junteng Jia and Kalyan Vasuden Alwala and Karthik Prasad and Kartikeya Upasani and Kate Plawiak and Ke Li and Kenneth Heafield and Kevin Stone and Khalid El-Arini and Krithika Iyer and Kshitiz Malik and Kuenley Chiu and Kunal Bhalla and Kushal Lakhotia and Lauren Rantala-Yeary and Laurens van der Maaten and Lawrence Chen and Liang Tan and Liz Jenkins and Louis Martin and Lovish Madaan and Lubo Malo and Lukas Blecher and Lukas Landzaat and Luke de Oliveira and Madeline Muzzi and Mahesh Pasupuleti and Mannat Singh and Manohar Paluri and Marcin Kardas and Maria Tsimpoukelli and Mathew Oldham and Mathieu Rita and Maya Pavlova and Melanie Kambadur and Mike Lewis and Min Si and Mitesh Kumar Singh and Mona Hassan and Naman Goyal and Narjes Torabi and Nikolay Bashlykov and Nikolay Bogoychev and Niladri Chatterji and Ning Zhang and Olivier Duchenne and Onur Çelebi and Patrick Alrassy and Pengchuan Zhang and Pengwei Li and Petar Vasic and Peter Weng and Prajjwal Bhargava and Pratik Dubal and Praveen Krishnan and Punit Singh Koura and Puxin Xu and Qing He and Qingxiao Dong and Ragavan Srinivasan and Raj Ganapathy and Ramon Calderer and Ricardo Silveira Cabral and Robert Stojnic and Roberta Raileanu and Rohan Maheswari and Rohit Girdhar and Rohit Patel and Romain Sauvestre and Ronnie Polidoro and Roshan Sumbaly and Ross Taylor and Ruan Silva and Rui Hou and Rui Wang and Saghar Hosseini and Sahana Chennabasappa and Sanjay Singh and Sean Bell and Seohyun Sonia Kim and Sergey Edunov and Shaoliang Nie and Sharan Narang and Sharath Raparthy and Sheng Shen and Shengye Wan and Shruti Bhosale and Shun Zhang and Simon Vandenhende and Soumya Batra and Spencer Whitman and Sten Sootla and Stephane Collot and Suchin Gururangan and Sydney Borodinsky and Tamar Herman and Tara Fowler and Tarek Sheasha and Thomas Georgiou and Thomas Scialom and Tobias Speckbacher and Todor Mihaylov and Tong Xiao and Ujjwal Karn and Vedanuj Goswami and Vibhor Gupta and Vignesh Ramanathan and Viktor Kerkez and Vincent Gonguet and Virginie Do and Vish Vogeti and Vítor Albiero and Vladan Petrovic and Weiwei Chu and Wenhan Xiong and Wenyin Fu and Whitney Meers and Xavier Martinet and Xiaodong Wang and Xiaofang Wang and Xiaoqing Ellen Tan and Xide Xia and Xinfeng Xie and Xuchao Jia and Xuewei Wang and Yaelle Goldschlag and Yashesh Gaur and Yasmine Babaei and Yi Wen and Yiwen Song and Yuchen Zhang and Yue Li and Yuning Mao and Zacharie Delpierre Coudert and Zheng Yan and Zhengxing Chen and Zoe Papakipos and Aaditya Singh and Aayushi Srivastava and Abha Jain and Adam Kelsey and Adam Shajnfeld and Adithya Gangidi and Adolfo Victoria and Ahuva Goldstand and Ajay Menon and Ajay Sharma and Alex Boesenberg and Alexei Baevski and Allie Feinstein and Amanda Kallet and Amit Sangani and Amos Teo and Anam Yunus and Andrei Lupu and Andres Alvarado and Andrew Caples and Andrew Gu and Andrew Ho and Andrew Poulton and Andrew Ryan and Ankit Ramchandani and Annie Dong and Annie Franco and Anuj Goyal and Aparajita Saraf and Arkabandhu Chowdhury and Ashley Gabriel and Ashwin Bharambe and Assaf Eisenman and Azadeh Yazdan and Beau James and Ben Maurer and Benjamin Leonhardi and Bernie Huang and Beth Loyd and Beto De Paola and Bhargavi Paranjape and Bing Liu and Bo Wu and Boyu Ni and Braden Hancock and Bram Wasti and Brandon Spence and Brani Stojkovic and Brian Gamido and Britt Montalvo and Carl Parker and Carly Burton and Catalina Mejia and Ce Liu and Changhan Wang and Changkyu Kim and Chao Zhou and Chester Hu and Ching-Hsiang Chu and Chris Cai and Chris Tindal and Christoph Feichtenhofer and Cynthia Gao and Damon Civin and Dana Beaty and Daniel Kreymer and Daniel Li and David Adkins and David Xu and Davide Testuggine and Delia David and Devi Parikh and Diana Liskovich and Didem Foss and Dingkang Wang and Duc Le and Dustin Holland and Edward Dowling and Eissa Jamil and Elaine Montgomery and Eleonora Presani and Emily Hahn and Emily Wood and Eric-Tuan Le and Erik Brinkman and Esteban Arcaute and Evan Dunbar and Evan Smothers and Fei Sun and Felix Kreuk and Feng Tian and Filippos Kokkinos and Firat Ozgenel and Francesco Caggioni and Frank Kanayet and Frank Seide and Gabriela Medina Florez and Gabriella Schwarz and Gada Badeer and Georgia Swee and Gil Halpern and Grant Herman and Grigory Sizov and Guangyi and Zhang and Guna Lakshminarayanan and Hakan Inan and Hamid Shojanazeri and Han Zou and Hannah Wang and Hanwen Zha and Haroun Habeeb and Harrison Rudolph and Helen Suk and Henry Aspegren and Hunter Goldman and Hongyuan Zhan and Ibrahim Damlaj and Igor Molybog and Igor Tufanov and Ilias Leontiadis and Irina-Elena Veliche and Itai Gat and Jake Weissman and James Geboski and James Kohli and Janice Lam and Japhet Asher and Jean-Baptiste Gaya and Jeff Marcus and Jeff Tang and Jennifer Chan and Jenny Zhen and Jeremy Reizenstein and Jeremy Teboul and Jessica Zhong and Jian Jin and Jingyi Yang and Joe Cummings and Jon Carvill and Jon Shepard and Jonathan McPhie and Jonathan Torres and Josh Ginsburg and Junjie Wang and Kai Wu and Kam Hou U and Karan Saxena and Kartikay Khandelwal and Katayoun Zand and Kathy Matosich and Kaushik Veeraraghavan and Kelly Michelena and Keqian Li and Kiran Jagadeesh and Kun Huang and Kunal Chawla and Kyle Huang and Lailin Chen and Lakshya Garg and Lavender A and Leandro Silva and Lee Bell and Lei Zhang and Liangpeng Guo and Licheng Yu and Liron Moshkovich and Luca Wehrstedt and Madian Khabsa and Manav Avalani and Manish Bhatt and Martynas Mankus and Matan Hasson and Matthew Lennie and Matthias Reso and Maxim Groshev and Maxim Naumov and Maya Lathi and Meghan Keneally and Miao Liu and Michael L. Seltzer and Michal Valko and Michelle Restrepo and Mihir Patel and Mik Vyatskov and Mikayel Samvelyan and Mike Clark and Mike Macey and Mike Wang and Miquel Jubert Hermoso and Mo Metanat and Mohammad Rastegari and Munish Bansal and Nandhini Santhanam and Natascha Parks and Natasha White and Navyata Bawa and Nayan Singhal and Nick Egebo and Nicolas Usunier and Nikhil Mehta and Nikolay Pavlovich Laptev and Ning Dong and Norman Cheng and Oleg Chernoguz and Olivia Hart and Omkar Salpekar and Ozlem Kalinli and Parkin Kent and Parth Parekh and Paul Saab and Pavan Balaji and Pedro Rittner and Philip Bontrager and Pierre Roux and Piotr Dollar and Polina Zvyagina and Prashant Ratanchandani and Pritish Yuvraj and Qian Liang and Rachad Alao and Rachel Rodriguez and Rafi Ayub and Raghotham Murthy and Raghu Nayani and Rahul Mitra and Rangaprabhu Parthasarathy and Raymond Li and Rebekkah Hogan and Robin Battey and Rocky Wang and Russ Howes and Ruty Rinott and Sachin Mehta and Sachin Siby and Sai Jayesh Bondu and Samyak Datta and Sara Chugh and Sara Hunt and Sargun Dhillon and Sasha Sidorov and Satadru Pan and Saurabh Mahajan and Saurabh Verma and Seiji Yamamoto and Sharadh Ramaswamy and Shaun Lindsay and Shaun Lindsay and Sheng Feng and Shenghao Lin and Shengxin Cindy Zha and Shishir Patil and Shiva Shankar and Shuqiang Zhang and Shuqiang Zhang and Sinong Wang and Sneha Agarwal and Soji Sajuyigbe and Soumith Chintala and Stephanie Max and Stephen Chen and Steve Kehoe and Steve Satterfield and Sudarshan Govindaprasad and Sumit Gupta and Summer Deng and Sungmin Cho and Sunny Virk and Suraj Subramanian and Sy Choudhury and Sydney Goldman and Tal Remez and Tamar Glaser and Tamara Best and Thilo Koehler and Thomas Robinson and Tianhe Li and Tianjun Zhang and Tim Matthews and Timothy Chou and Tzook Shaked and Varun Vontimitta and Victoria Ajayi and Victoria Montanez and Vijai Mohan and Vinay Satish Kumar and Vishal Mangla and Vlad Ionescu and Vlad Poenaru and Vlad Tiberiu Mihailescu and Vladimir Ivanov and Wei Li and Wenchen Wang and Wenwen Jiang and Wes Bouaziz and Will Constable and Xiaocheng Tang and Xiaojian Wu and Xiaolan Wang and Xilun Wu and Xinbo Gao and Yaniv Kleinman and Yanjun Chen and Ye Hu and Ye Jia and Ye Qi and Yenda Li and Yilin Zhang and Ying Zhang and Yossi Adi and Youngjin Nam and Yu and Wang and Yu Zhao and Yuchen Hao and Yundi Qian and Yunlu Li and Yuzi He and Zach Rait and Zachary DeVito and Zef Rosnbrick and Zhaoduo Wen and Zhenyu Yang and Zhiwei Zhao and Zhiyu Ma},
  year          = {2024},
  eprint        = {2407.21783},
  archiveprefix = {arXiv},
  primaryclass  = {cs.AI},
  url           = {https://arxiv.org/abs/2407.21783}
}

@inproceedings{ForesightKV,
  title     = {Foresight{KV}: Optimizing {KV} Cache Eviction for Reasoning Models by Learning Long-Term Contribution},
  author    = {Zican Dong and Peiyu Liu and Junyi Li and Zhipeng Chen and Han Peng and Shuo Wang and Xin Zhao},
  booktitle = {Forty-third International Conference on Machine Learning},
  year      = {2026},
  url       = {https://openreview.net/forum?id=znV8JHv8b8}
}

@inproceedings{tang2026predicting,
  title     = {Predicting Future Utility: Global Combinatorial Optimization for Task-Agnostic {KV} Cache Eviction},
  author    = {Ziyao Tang and Pengkun Jiao and Xinhang Chen and Wei Liu and Shiyong Li and Jingjing Chen},
  booktitle = {Forty-third International Conference on Machine Learning},
  year      = {2026},
  url       = {https://openreview.net/forum?id=FQLxcBsKIb}
}

@inproceedings{learn-to-evict,
  title     = {Learning to Evict from Key-Value Cache},
  author    = {Luca Moschella and Laura Manduchi and Ozan Sener},
  booktitle = {Forty-third International Conference on Machine Learning},
  year      = {2026},
  url       = {https://openreview.net/forum?id=0OevIlRMYN}
}

@misc{UT-ACA,
  title         = {UT-ACA: Uncertainty-Triggered Adaptive Context Allocation for Long-Context Inference},
  author        = {Lang Zhou and Shuxuan Li and Zhuohao Li and Shi Liu and Wei-Shi Zheng and Zhilin Zhao},
  year          = {2026},
  eprint        = {2603.18446},
  archiveprefix = {arXiv},
  primaryclass  = {cs.CL},
  url           = {https://arxiv.org/abs/2603.18446}
}

@misc{adaptivetesttimecompute,
  title         = {Adaptive Test-Time Compute Allocation for Reasoning LLMs via Constrained Policy Optimization},
  author        = {Zhiyuan Zhai and Bingcong Li and Bingnan Xiao and Ming Li and Xin Wang},
  year          = {2026},
  eprint        = {2604.14853},
  archiveprefix = {arXiv},
  primaryclass  = {cs.LG},
  url           = {https://arxiv.org/abs/2604.14853}
}

@misc{Later,
  title         = {LaTER: Efficient Test-Time Reasoning via Latent Exploration and Explicit Verification},
  author        = {Xuan Li and Yining Wang and Yuchen Liu and Guanjun Liu and Delai Qiu and Shengping Liu and Jiaen Liang and Wei Huang and Jun Yu and Junnan Zhu},
  year          = {2026},
  eprint        = {2605.07315},
  archiveprefix = {arXiv},
  primaryclass  = {cs.CL},
  url           = {https://arxiv.org/abs/2605.07315}
}

@misc{DBtrimKV,
  title         = {Make Each Token Count: Towards Improving Long-Context Performance with KV Cache Eviction},
  author        = {Ngoc Bui and Hieu Trung Nguyen and Arman Cohan and Rex Ying},
  year          = {2026},
  eprint        = {2605.09649},
  archiveprefix = {arXiv},
  primaryclass  = {cs.LG},
  url           = {https://arxiv.org/abs/2605.09649}
}

@inproceedings{H20,
  author    = {Zhang, Zhenyu and Sheng, Ying and Zhou, Tianyi and Chen, Tianlong and Zheng, Lianmin and Cai, Ruisi and Song, Zhao and Tian, Yuandong and R{\'e}, Christopher and Barrett, Clark and Wang, Zhangyang and Chen, Beidi},
  title     = {H2O: heavy-hitter oracle for efficient generative inference of large language models},
  year      = {2023},
  publisher = {Curran Associates Inc.},
  address   = {Red Hook, NY, USA},
  booktitle = {Proceedings of the 37th International Conference on Neural Information Processing Systems},
  articleno = {1506},
  numpages  = {50},
  location  = {New Orleans, LA, USA},
  series    = {NIPS '23}
}

@inproceedings{Ruler,
  title     = {{RULER}: What{\textquoteright}s the Real Context Size of Your Long-Context Language Models?},
  author    = {Cheng-Ping Hsieh and Simeng Sun and Samuel Kriman and Shantanu Acharya and Dima Rekesh and Fei Jia and Boris Ginsburg},
  booktitle = {First Conference on Language Modeling},
  year      = {2024},
  url       = {https://openreview.net/forum?id=kIoBbc76Sy}
}

@inproceedings{snapkv,
  title     = {Snap{KV}: {LLM} Knows What You are Looking for Before Generation},
  author    = {Yuhong Li and Yingbing Huang and Bowen Yang and Bharat Venkitesh and Acyr Locatelli and Hanchen Ye and Tianle Cai and Patrick Lewis and Deming Chen},
  booktitle = {The Thirty-eighth Annual Conference on Neural Information Processing Systems},
  year      = {2024},
  url       = {https://openreview.net/forum?id=poE54GOq2l}
}

@inproceedings{PyramidKV,
  title     = {Pyramid{KV}: Dynamic {KV} Cache Compression based on Pyramidal Information Funneling},
  author    = {Zefan Cai and Yichi Zhang and Bofei Gao and Yuliang Liu and Yucheng Li and Tianyu Liu and Keming Lu and Wayne Xiong and Yue Dong and Junjie Hu and Wen Xiao},
  booktitle = {Second Conference on Language Modeling},
  year      = {2025},
  url       = {https://openreview.net/forum?id=ayi7qezU87}
}

@inproceedings{Adakv,
  title     = {Ada-{KV}: Optimizing {KV} Cache Eviction by Adaptive Budget Allocation for Efficient {LLM} Inference},
  author    = {Yuan Feng and Junlin Lv and Yukun Cao and Xike Xie and S Kevin Zhou},
  booktitle = {The Thirty-ninth Annual Conference on Neural Information Processing Systems},
  year      = {2025},
  url       = {https://openreview.net/forum?id=tcisuhGsQZ}
}

@inproceedings{DynamicKV,
  author    = {Xiabin Zhou and Wenbin Wang and Minyan Zeng and Jiaxian Guo and Xuebo Liu and Li Shen and Min Zhang and Liang Ding},
  title     = {DynamicKV: Task-Aware Adaptive KV Cache Compression for Long Context LLMs},
  year      = {2025},
  cdate     = {1735689600000},
  pages     = {8042-8057},
  url       = {https://aclanthology.org/2025.findings-emnlp.426/},
  booktitle = {EMNLP (Findings)}
}

@inproceedings{CapKV,
  title     = {Rethinking {KV} Cache Eviction via a Unified Information-Theoretic Objective},
  author    = {Jiaming Yang and Chenwei Tang and Liangli Zhen and Jiancheng Lv},
  booktitle = {Forty-third International Conference on Machine Learning},
  year      = {2026},
  url       = {https://openreview.net/forum?id=2fPbFUJ87k}
}

@misc{garcia2026protectionnearlyneedstructural,
  title         = {Protection Is (Nearly) All You Need: Structural Protection Dominates Scoring in Globally Capped KV Eviction},
  author        = {Gabriel Garcia},
  year          = {2026},
  eprint        = {2605.18053},
  archiveprefix = {arXiv},
  primaryclass  = {cs.LG},
  url           = {https://arxiv.org/abs/2605.18053}
}

@inproceedings{indexmem,
  title     = {IndexMem: Learned {KV}-Cache Eviction with Latent Memory for Long-Context {LLM} Inference},
  author    = {Xintong Yang and Hao Gu and Binxing Xu and Lujun Li and Bei Liu and Jiacheng Liu and Qiyuan Zhu and Yike Guo and Sirui Han},
  booktitle = {Forty-third International Conference on Machine Learning},
  year      = {2026},
  url       = {https://openreview.net/forum?id=SKXkGLbXdB}
}

@inproceedings{longbench,
  title     = {{L}ong{B}ench: A Bilingual, Multitask Benchmark for Long Context Understanding},
  author    = {Bai, Yushi  and
               Lv, Xin  and
               Zhang, Jiajie  and
               Lyu, Hongchang  and
               Tang, Jiankai  and
               Huang, Zhidian  and
               Du, Zhengxiao  and
               Liu, Xiao  and
               Zeng, Aohan  and
               Hou, Lei  and
               Dong, Yuxiao  and
               Tang, Jie  and
               Li, Juanzi},
  editor    = {Ku, Lun-Wei  and
               Martins, Andre  and
               Srikumar, Vivek},
  booktitle = {Proceedings of the 62nd Annual Meeting of the Association for Computational Linguistics (Volume 1: Long Papers)},
  month     = aug,
  year      = {2024},
  address   = {Bangkok, Thailand},
  publisher = {Association for Computational Linguistics},
  url       = {https://aclanthology.org/2024.acl-long.172/},
  doi       = {10.18653/v1/2024.acl-long.172},
  pages     = {3119--3137}
}

@misc{mistral7b,
  title         = {Mistral 7B},
  author        = {Albert Q. Jiang and Alexandre Sablayrolles and Arthur Mensch and Chris Bamford and Devendra Singh Chaplot and Diego de las Casas and Florian Bressand and Gianna Lengyel and Guillaume Lample and Lucile Saulnier and Lélio Renard Lavaud and Marie-Anne Lachaux and Pierre Stock and Teven Le Scao and Thibaut Lavril and Thomas Wang and Timothée Lacroix and William El Sayed},
  year          = {2023},
  eprint        = {2310.06825},
  archiveprefix = {arXiv},
  primaryclass  = {cs.CL},
  url           = {https://arxiv.org/abs/2310.06825}
}

@misc{qwen2.5,
  title         = {Qwen2.5 Technical Report},
  author        = {Qwen and : and An Yang and Baosong Yang and Beichen Zhang and Binyuan Hui and Bo Zheng and Bowen Yu and Chengyuan Li and Dayiheng Liu and Fei Huang and Haoran Wei and Huan Lin and Jian Yang and Jianhong Tu and Jianwei Zhang and Jianxin Yang and Jiaxi Yang and Jingren Zhou and Junyang Lin and Kai Dang and Keming Lu and Keqin Bao and Kexin Yang and Le Yu and Mei Li and Mingfeng Xue and Pei Zhang and Qin Zhu and Rui Men and Runji Lin and Tianhao Li and Tianyi Tang and Tingyu Xia and Xingzhang Ren and Xuancheng Ren and Yang Fan and Yang Su and Yichang Zhang and Yu Wan and Yuqiong Liu and Zeyu Cui and Zhenru Zhang and Zihan Qiu},
  year          = {2025},
  eprint        = {2412.15115},
  archiveprefix = {arXiv},
  primaryclass  = {cs.CL},
  url           = {https://arxiv.org/abs/2412.15115}
}

@misc{ruler-hf,
  author       = {Simon Jegou},
  title        = {RULER},
  year         = {2024},
  howpublished = {Hugging Face Datasets},
  url          = {https://huggingface.co/datasets/simonjegou/ruler},
  note         = {Dataset}
}

@misc{kvpress,
  title         = {Expected Attention: KV Cache Compression by Estimating Attention from Future Queries Distribution},
  author        = {Alessio Devoto and Maximilian Jeblick and Simon Jégou},
  year          = {2025},
  eprint        = {2510.00636},
  archiveprefix = {arXiv},
  primaryclass  = {cs.AI},
  url           = {https://arxiv.org/abs/2510.00636}
}

@inproceedings{criticalKV,
  title     = {Critical{KV}: Optimizing {KV} Cache Eviction from an Output Perturbation Perspective},
  author    = {Yuan Feng and Junlin Lv and Haoyu Guo and Yukun Cao and S Kevin Zhou and Xike Xie},
  booktitle = {Forty-third International Conference on Machine Learning},
  year      = {2026},
  url       = {https://openreview.net/forum?id=vMwlcLyI4b}
}

@inproceedings{pitfalls,
  title     = {The Pitfalls of {KV} Cache Compression},
  author    = {Chen, Alex  and
               Geh, Renato  and
               Grover, Aditya  and
               Van Den Broeck, Guy  and
               Israel, Daniel Mingyi},
  editor    = {Liakata, Maria  and
               Moreira, Viviane P.  and
               Zhang, Jiajun  and
               Jurgens, David},
  booktitle = {Proceedings of the 64th Annual Meeting of the {A}ssociation for {C}omputational {L}inguistics (Volume 1: Long Papers)},
  month     = jul,
  year      = {2026},
  address   = {San Diego, California, United States},
  publisher = {Association for Computational Linguistics},
  url       = {https://aclanthology.org/2026.acl-long.1926/},
  doi       = {10.18653/v1/2026.acl-long.1926},
  pages     = {41530--41553},
  isbn      = {979-8-89176-390-6}
}

@misc{fastkvzip,
  title         = {Fast KVzip: Efficient and Accurate LLM Inference with Gated KV Eviction},
  author        = {Jang-Hyun Kim and Dongyoon Han and Sangdoo Yun},
  year          = {2026},
  eprint        = {2601.17668},
  archiveprefix = {arXiv},
  primaryclass  = {cs.LG},
  url           = {https://arxiv.org/abs/2601.17668}
}

@inproceedings{manifoldkv,
  title     = {Manifold{KV}: Training-Free {KV} Cache Compression via Euclidean Outlier Detection},
  author    = {Debajyoti Datta and Trishala Neeraj and Bibek Paudel and Vyom Sharma and Subhabrata Mukherjee},
  booktitle = {Forty-third International Conference on Machine Learning},
  year      = {2026},
  url       = {https://openreview.net/forum?id=fo77TzkulW}
}

@misc{CompilerKV,
  title         = {CompilerKV: Risk-Adaptive KV Compression via Offline Experience Compilation},
  author        = {Ning Yang and Chengzhi Wang and Yibo Liu and Baoliang Tian and Haijun Zhang},
  year          = {2026},
  eprint        = {2602.08686},
  archiveprefix = {arXiv},
  primaryclass  = {cs.LG},
  url           = {https://arxiv.org/abs/2602.08686}
}

@misc{understandingphysicskeyvaluecache,
  title         = {Understanding the Physics of Key-Value Cache Compression for LLMs through Attention Dynamics},
  author        = {Samhruth Ananthanarayanan and Ayan Sengupta and Tanmoy Chakraborty},
  year          = {2026},
  eprint        = {2603.01426},
  archiveprefix = {arXiv},
  primaryclass  = {cs.CL},
  url           = {https://arxiv.org/abs/2603.01426}
}

@misc{valueawarekveviction,
  title         = {When Does Value-Aware KV Eviction Help? A Fixed-Contract Diagnostic for Non-Monotone Cache Compression},
  author        = {Ruijie Zhang and Haozhe Liang and Da Chang and Li Hu and Fanqi Kong and Huaxiao Yin and Yu Li},
  year          = {2026},
  eprint        = {2605.08234},
  archiveprefix = {arXiv},
  primaryclass  = {cs.LG},
  url           = {https://arxiv.org/abs/2605.08234}
}

@misc{simplepluginimprovingevictionbased,
  title         = {A Simple Plug-in for Improving Eviction-Based KV Cache Compression},
  author        = {Yuping Lin and Jiayuan Ding and Yue Xing and Pengfei He and Jiliang Tang and Subhabrata Mukherjee},
  year          = {2026},
  eprint        = {2605.23258},
  archiveprefix = {arXiv},
  primaryclass  = {cs.LG},
  url           = {https://arxiv.org/abs/2605.23258}
}

@misc{MomentKV,
  title         = {MomentKV: Closing the Directional Gap in KV Cache Eviction for Long-Context Inference},
  author        = {Yu Li and Binxu Li and Tian Lan},
  year          = {2026},
  eprint        = {2606.01563},
  archiveprefix = {arXiv},
  primaryclass  = {cs.LG},
  url           = {https://arxiv.org/abs/2606.01563}
}

@inproceedings{chang2026valueaware,
  title     = {Value-Aware Stochastic {KV} Cache Eviction for Reasoning Models},
  author    = {Ting-Yun Chang and Harvey Yiyun Fu and Deqing Fu and Chenghao Yang and Jesse Thomason and Robin Jia},
  booktitle = {COLM 2026 Workshop on Efficient Reasoning},
  year      = {2026},
  url       = {https://openreview.net/forum?id=HysrOiWFIA}
}

@misc{rigiddynamicentropyguidedadaptive,
  title         = {From Rigid to Dynamic: Entropy-Guided Adaptive Inference for Long-Context LLMs},
  author        = {Zhanchao Xu and Haoyang Li and Qingfa Xiao and Fei Teng and Chen Jason Zhang and Lei Chen and Qing Li},
  year          = {2026},
  eprint        = {2606.09508},
  archiveprefix = {arXiv},
  primaryclass  = {cs.AI},
  url           = {https://arxiv.org/abs/2606.09508}
}

@misc{guidedkvcachecompression,
  title         = {CompressKV: Semantic-Retrieval-Guided KV-Cache Compression for Resource-Efficient Long-Context LLM Inference},
  author        = {Xiaolin Lin and Jingcun Wang and Olga Kondrateva and Yiyu Shi and Bing Li and Grace Li Zhang},
  year          = {2026},
  eprint        = {2606.24467},
  archiveprefix = {arXiv},
  primaryclass  = {cs.AI},
  url           = {https://arxiv.org/abs/2606.24467}
}

@misc{hu2026take,
  title  = {{TAKE}: Task-Aware Chunked {KV} Cache Eviction for Efficient Long-Context {LLM} Prefill},
  author = {Long Hu and Nan Jia and Rui Wang and Jiahui Li and Qingyi Yang and Yixue Hao and Xianzhi Li and Xiaofei Liao},
  year   = {2026},
  url    = {https://openreview.net/forum?id=kMLfUshPwo}
}

@misc{zeng2025incontextkvcacheevictionllms,
  title         = {In-context KV-Cache Eviction for LLMs via Attention-Gate},
  author        = {Zihao Zeng and Bokai Lin and Tianqi Hou and Hao Zhang and Zhijie Deng},
  year          = {2025},
  eprint        = {2410.12876},
  archiveprefix = {arXiv},
  primaryclass  = {cs.CL},
  url           = {https://arxiv.org/abs/2410.12876}
}

@misc{EPIkv-arxiv,
  title         = {Epiphany-Aware KV Cache Eviction Without the Attention Matrix},
  author        = {Steven Kolawole and Virginia Smith},
  year          = {2026},
  eprint        = {2606.26472},
  archiveprefix = {arXiv},
  primaryclass  = {cs.LG},
  url           = {https://arxiv.org/abs/2606.26472}
}

@inproceedings{defensive-kv-iclr2026,
  title     = {Defensive{KV}: Taming the Fragility of {KV} Cache Eviction in {LLM} Inference},
  author    = {Yuan Feng and Haoyu Guo and Junlin Lv and S Kevin Zhou and Xike Xie},
  booktitle = {The Fourteenth International Conference on Learning Representations},
  year      = {2026},
  url       = {https://openreview.net/forum?id=nJgS06sX3O}
}

@misc{evicpressjointkvcachecompression,
  title         = {EVICPRESS: Joint KV-Cache Compression and Eviction for Efficient LLM Serving},
  author        = {Shaoting Feng and Yuhan Liu and Hanchen Li and Xiaokun Chen and Samuel Shen and Kuntai Du and Zhuohan Gu and Rui Zhang and Yuyang Huang and Yihua Cheng and Jiayi Yao and Qizheng Zhang and Ganesh Ananthanarayanan and Junchen Jiang},
  year          = {2025},
  eprint        = {2512.14946},
  archiveprefix = {arXiv},
  primaryclass  = {cs.OS},
  url           = {https://arxiv.org/abs/2512.14946}
}

@inproceedings{keydiff,
  title     = {KeyDiff: Key Similarity-Based {KV} Cache Eviction for Long-Context {LLM} Inference in Resource-Constrained Environments},
  author    = {Junyoung Park and Dalton Jones and Matthew J Morse and Raghavv Goel and Mingu Lee and Christopher Lott},
  booktitle = {The Thirty-ninth Annual Conference on Neural Information Processing Systems},
  year      = {2025},
  url       = {https://openreview.net/forum?id=uBaFH7aQnC}
}

@inproceedings{sagekv,
  title     = {{LLM}s Know What to Drop: Self-Attention Guided {KV} Cache Eviction for Efficient Long-Context Inference},
  author    = {Guangtao Wang and Shubhangi Upasani and Chen Wu and Darshan Gandhi and Jonathan Lingjie Li and Changran Hu and Bo Li and Urmish Thakker},
  booktitle = {Sparsity in LLMs (SLLM): Deep Dive into Mixture of Experts, Quantization, Hardware, and Inference},
  year      = {2025},
  url       = {https://openreview.net/forum?id=qg9dlCcNzr}
}

@misc{AnchorKV,
  title         = {AnchorKV: Anchor-Residual KV Cache Compression},
  author        = {Malik Khalaf and Yara Shamshoum and Nitzan Hodos and Yuval Sieradzki and Assaf Schuster},
  year          = {2026},
  eprint        = {2608.02901},
  archiveprefix = {arXiv},
  primaryclass  = {cs.LG},
  url           = {https://arxiv.org/abs/2608.02901}
}

@inproceedings{lookaheadkv,
  title     = {Lookahead{KV}: Fast and Accurate {KV} Cache Eviction by Glimpsing into the Future without Generation},
  author    = {Jinwoo Ahn and Ingyu Seong and Akhil Kedia and Junhan Kim and Hyemi Jang and Kangwook Lee and Yongkweon Jeon},
  booktitle = {The Fourteenth International Conference on Learning Representations},
  year      = {2026},
  url       = {https://openreview.net/forum?id=RVLMGPXt2i}
}

@inproceedings{Cache-to-Cache,
title={Cache-to-Cache: Direct Semantic Communication Between Large Language Models},
author={Tianyu Fu and Zihan Min and Hanling Zhang and Jichao Yan and Guohao Dai and Wanli Ouyang and Yu Wang},
booktitle={The Fourteenth International Conference on Learning Representations},
year={2026},
url={https://openreview.net/forum?id=LeatkxrBCi}
}

@article{Harris_1960,
  title   = {A lower bound for the critical probability in a certain percolation process},
  volume  = {56},
  doi     = {10.1017/S0305004100034241},
  number  = {1},
  journal = {Mathematical Proceedings of the Cambridge Philosophical Society},
  author  = {Harris, T. E.},
  year    = {1960},
  pages   = {13–20}
}

@inproceedings{liu2024repobench,
  title     = {RepoBench: Benchmarking Repository-Level Code Auto-Completion Systems},
  author    = {Tianyang Liu and Canwen Xu and Julian McAuley},
  booktitle = {The Twelfth International Conference on Learning Representations},
  year      = {2024},
  url       = {https://openreview.net/forum?id=pPjZIOuQuF}
}

@inproceedings{hotpotqa,
  title     = {{H}otpot{QA}: A Dataset for Diverse, Explainable Multi-hop Question Answering},
  author    = {Yang, Zhilin  and
               Qi, Peng  and
               Zhang, Saizheng  and
               Bengio, Yoshua  and
               Cohen, William  and
               Salakhutdinov, Ruslan  and
               Manning, Christopher D.},
  editor    = {Riloff, Ellen  and
               Chiang, David  and
               Hockenmaier, Julia  and
               Tsujii, Jun{'}ichi},
  booktitle = {Proceedings of the 2018 Conference on Empirical Methods in Natural Language Processing},
  month     = oct # {-} # nov,
  year      = {2018},
  address   = {Brussels, Belgium},
  publisher = {Association for Computational Linguistics},
  url       = {https://aclanthology.org/D18-1259/},
  doi       = {10.18653/v1/D18-1259},
  pages     = {2369--2380}
}

@inproceedings{wu2025retrieval,
  title     = {Retrieval Head Mechanistically Explains Long-Context Factuality},
  author    = {Wenhao Wu and Yizhong Wang and Guangxuan Xiao and Hao Peng and Yao Fu},
  booktitle = {The Thirteenth International Conference on Learning Representations},
  year      = {2025},
  url       = {https://openreview.net/forum?id=EytBpUGB1Z}
}

@inproceedings{streamingLLM,
title={Efficient Streaming Language Models with Attention Sinks},
author={Guangxuan Xiao and Yuandong Tian and Beidi Chen and Song Han and Mike Lewis},
booktitle={The Twelfth International Conference on Learning Representations},
year={2024},
url={https://openreview.net/forum?id=NG7sS51zVF}
}

@article{roformer,
  author     = {Su, Jianlin and Ahmed, Murtadha and Lu, Yu and Pan, Shengfeng and Bo, Wen and Liu, Yunfeng},
  title      = {RoFormer: Enhanced transformer with Rotary Position Embedding},
  year       = {2024},
  issue_date = {Feb 2024},
  publisher  = {Elsevier Science Publishers B. V.},
  address    = {NLD},
  volume     = {568},
  number     = {C},
  issn       = {0925-2312},
  url        = {https://doi.org/10.1016/j.neucom.2023.127063},
  doi        = {10.1016/j.neucom.2023.127063},
  journal    = {Neurocomput.},
  month      = feb,
  numpages   = {12}
}

@inproceedings{flash-attention,
  author    = {Dao, Tri and Fu, Daniel Y. and Ermon, Stefano and Rudra, Atri and R\'{e}, Christopher},
  title     = {FLASHATTENTION: fast and memory-efficient exact attention with IO-awareness},
  year      = {2022},
  isbn      = {9781713871088},
  publisher = {Curran Associates Inc.},
  address   = {Red Hook, NY, USA},
  booktitle = {Proceedings of the 36th International Conference on Neural Information Processing Systems},
  articleno = {1189},
  numpages  = {16},
  location  = {New Orleans, LA, USA},
  series    = {NIPS '22}
}

\appendix

\section{Deferred Tables, Derivations, and Secondary Experiments}
\label{app:deferred}

\subsection{Extended Related Work}
\label{app:related-full}

Section~\ref{sec:related} positions \method against the decision axes that prior work optimizes.
This subsection gives the full treatment, starting with the five neighbours closest to our setup
before turning to concurrent, older, and adjacent work.

\paragraph{Concurrent 2026 diagnostics and wrappers}
\citet{valueawarekveviction} ask a similar headline question, namely when value-aware KV eviction helps. Their work is complementary to ours along a clear axis. Their score is \textit{value-aware} (a per-input support-coupling score $\phi(x)$ over value states), whereas our signal is \textit{attention topology}. Beyond this signal difference, their unit of analysis is the selector rather than the task, their $\phi(x)$ drives diagnostic grid splits rather than an eviction gate, and they report that diagnostic signals transfer across selectors only conditionally, which makes the single transferred threshold we report non-obvious. 
\citet{pitfalls} document instruction-type-dependent degradation under five eviction methods, with some instruction types becoming entirely ignored, the closest published task-level failure taxonomy. We
cede the hurts-degradation observation to them and claim instead the joint helps-versus-hurts partition plus the a priori predictor. 
\citet{understandingphysicskeyvaluecache} compute head-consensus levels and report how they evolve with depth, describing distinct profiles by architecture (early consensus with late diversification on one family, late convergence on another) and relating those profiles descriptively to compression tolerance. Their statistic is not purely \textit{within-layer}: they too report a depth-wise comparison, substantially the diagnosis Appendix~\ref{subsec:layer-profile} reaches independently. We differentiate instead on the three axes that do hold: per input rather than per architecture, computed online at prefill rather than offline, and thresholded to make a deployment decision rather than used as a descriptive metric. To our knowledge the statistic has not been proposed as an eviction gate before, and the ablation of Table~\ref{tab:signal-ablation} shows it separates the partition cleanly (AUC $1.000$) on the selecting cell and transfers across families. 
On the failure side, \citet{chang2026valueaware} (VaSE) trace catastrophic eviction failures to high-magnitude value states, an axis orthogonal to our attention-topology analysis, and ManifoldKV~\citep{manifoldkv} shows that a Euclidean-outlier key scorer improves NIAH-MultiKey retention under compression, which is why we define ``capacity-bound'' relative to the scorer being gated (Section~\ref{subsec:scorer-independence}).
SAGE-KV~\citep{sagekv} shares our \textit{premise}: after prefilling, the model's own attention already indicates which tokens can be dropped, read once rather than per decode step, with a head-level analysis. The difference is what is done with the premise. SAGE-KV always compresses and uses the signal to pick \textit{which} tokens and heads to keep, whereas we use a signal from the same source to decide \textit{whether} to evict at all, do no head-level selection, and, unlike SAGE-KV, characterize a task class on which eviction cannot be made safe by any choice of tokens.
KeyDiff~\citep{keydiff} is an attention-free, FlashAttention-compatible evictor whose stated motivation is precisely that materializing attention is expensive. We address it empirically in the scorer-independence analysis (Section~\ref{subsec:scorer-independence}). KeyDiff collapses on NIAH-MK3 below 25\% budget, as does every scorer we tested. Its effectiveness therefore does not confer safety on the capacity-bound class.

\begin{table}[t]
\centering
\caption{Gate-signal ablation on Qwen2.5-1.5B-Instruct (RULER 4K, $N = 50$ per task, one shared prefill per input). Each cheap prefill statistic is tested for separating the capacity-bound anchor (NIAH-MK3) from the dilution-prone tasks (vt, fwe, qa\_1, multivalue). ``Sep.'' = a single threshold puts the capacity-bound task mean(s) strictly on one side of all four dilution task means. AUCs are per-input and oriented per signal ($0.5$ = chance). ``MK mono.'' = task means are monotone along the MK1$\to$MK2$\to$MK3 distractor gradient; $r$ = Pearson correlation with $D$.}
\label{tab:signal-ablation}
\resizebox{\textwidth}{!}{%
\begin{tabular}{lcccccc}
\toprule
Signal & Sep.\ MK3 & AUC MK3 & AUC MK3+MK2 & MK mono. & $r$ w/ $D$ \\
\midrule
head-agreement drop $D$ (paper) & \yes  & 1.000 & \textbf{0.859} & \yes & $+1.00$ \\
mean attention entropy (norm.)  & \no & 0.508 & 0.670 & \no & $+0.08$ \\
top-32 attention mass share & \no  & 0.784 & 0.574 & \yes & $+0.22$ \\
max single-position share  & \no & 0.723 & 0.708 & \yes & $+0.32$ \\
key-norm dispersion (std/mean)  & \no & 0.570 & 0.560 & \no & $+0.01$ \\
early$-$late entropy drop  & \no & 0.805 & 0.575 & \no & $-0.64$ \\
\bottomrule
\end{tabular}%
}
\end{table}

On the wrapper side, VECTOR~\citep{simplepluginimprovingevictionbased} and
CriticalKV~\citep{criticalKV} are plug-and-play augmentations over
multiple base evictors, and both always compress under a fixed budget and
adapt \textit{which} tokens are kept, whereas our gate decides
\textit{whether} to evict at all, a priori from a label-free attention
signal. The serving system EvicPress~\citep{evicpressjointkvcachecompression} also chooses
whether to evict per context, but as part of a profiled quality-latency
optimization over storage tiers rather than from an attention statistic,
and without a task-type partition. Cross-family transfer of eviction control is itself known through CompilerKV's~\citep{CompilerKV} portable per-head retention tables. We instead transfer one fixed threshold across the Qwen and Mistral families, extended to Llama-architecture models and Qwen3 via an unlabeled per-model standardization.

Further concurrent evictors and allocators,
EntropyInfer~\citep{rigiddynamicentropyguidedadaptive}, CompressKV~\citep{guidedkvcachecompression},
MomentKV~\citep{MomentKV}, EpiKV~\citep{EPIkv-arxiv},
DefensiveKV~\citep{defensive-kv-iclr2026},
LookaheadKV~\citep{lookaheadkv}, and TAKE~\citep{hu2026take}, adapt
budgets or scores per head, layer, token, or task, but none gates the
eviction decision per input or identifies a task class where eviction
hurts monotonically. AnchorKV~\citep{AnchorKV} is also
``safety-aware'' but in a different sense from ours: it biases retention
scores away from harmful-prompt directions for jailbreak defense while
always compressing, whereas our ``collapse-avoiding'' sense is about not
destroying task accuracy. Learned gates
(Attention-Gate~\citep{zeng2025incontextkvcacheevictionllms}, Fast
KVzip~\citep{fastkvzip}) share the ``gate'' name but train their
gating parameters, unlike our inference-time scalar. A separate line uses
learnable per-layer gates for cross-model KV-cache \textit{fusion}
(C2C~\citep{Cache-to-Cache}), which addresses communication between
models, orthogonal to eviction within one.

Beyond the concurrent methods above, an older set of baselines and training-based
predictors anchors the comparison. SnapKV~\citep{snapkv},
H2O~\citep{H20}, StreamingLLM~\citep{streamingLLM},
PyramidKV~\citep{PyramidKV}, and Ada-KV~\citep{Adakv} are the
standard training-free eviction baselines. We use SnapKV-style scoring as
the base for our gate and demonstrate orthogonality on H2O, StreamingLLM,
and PyramidKV. ForesightKV~\citep{ForesightKV} and
Learning-to-Evict~\citep{learn-to-evict} train predictors of future
utility and require labels or rollouts to do so. UT-ACA~\citep{UT-ACA} and LU-KV~\citep{tang2026predicting} are
calibration-based context allocators on a single substrate. None of these
proposes a binary task-type partition or an a priori predictor.

Finally, three methods adapt compute on axes adjacent to but distinct from ours: AdaCompute-GBM~\citep{adaptivetesttimecompute} on the self-consistency
sample-count axis, BudgetThinker~\citep{BudgetThinker} on the CoT
chain-length axis, and LaTER~\citep{Later} on latent-then-explicit
reasoning. These work on substrates orthogonal to the KV-cache axis we
study.

\subsection{Additional Partition Analyses}
\label{app:partition-details}

This subsection presents the robustness checks, confound analyses, and
boundary diagnostics that support the empirical partition of
Section~\ref{sec:experiments}.

\subsubsection{Saturation, budget-insensitivity, and context-length scaling}

A cell below the dilution-prone threshold for two different
reasons, and only saturation case is harmless. Qwen2.5-3B has $\rho \approx 0$ on VT at 
4K because $A_{\mathrm{full}} \approx 1$, and the same model reaches $\rho = 0.32$ on FWE at 16K once
headroom opens. Mistral-7B repeats the pattern: at 4K it is largely
saturated ($A_{\mathrm{full}} = 1.00$ on VT, $0.82$ on FWE), and at 16K VT
and FWE recover strongly ($\rho = 0.32$ and $0.22$, niah\_multivalue
$0.06 \to 0.21$). Budget-insensitivity is the informative case instead,
because headroom exists and recovery does not occur. On Mistral,
QA\_1 has $\rho = 0.04 / 0.00$ at 4K/16K despite $A_{\mathrm{full}} = 0.81
/ 0.70$, and QA\_2 has $\rho = 0.01 / 0.00$ despite $A_{\mathrm{full}} =
0.50 / 0.56$, and QA\_2 on Qwen-14B behaves the same way. These are all 
QA, where the residual errors resemble parametric-knowledge failures
that no cache budget can repair, a phenomenon distinct from the capacity
limit exhibited by NIAH-MK3.

The headroom argument predicts the behaviour as context grows, and
Qwen2.5-14B tests it directly by re-entering the dilution regime at long
context. At 4K, three dilution-prone tasks (VT, FWE, QA\_2) sit at $\rho
\approx 0$: saturation, since $A_{\mathrm{full}} \in [0.92, 1.00]$ leaves
no headroom for recovery, as Eq.~\ref{eq:scaling} predicts. At 16K, two
of five dilution-prone tasks cross the threshold (FWE $0.00 \to 0.20$,
QA\_1 $0.06 \to 0.10$), niah\_multivalue stays above it ($0.08 \to
0.07$), and the capacity-bound NIAH-MK3 stays at exactly $0.00$ at both
scales. VT ($0.00 \to 0.03$) and QA\_2 ($0.02 \to 0.03$) remain below
threshold because the inputs that the model gets wrong at 4K stay wrong
at every budget at 16K. Headroom opens, dilution returns, and NIAH-MK3
stays pinned: the partition behaves as the scaling formula dictates, not
as task identity dictates.

\subsubsection{The partition is not a task-name heuristic}

Task identity and dilution-proneness could in principle be confounded, but they are not.
The NIAH family contains structurally different subtasks. niah\_multivalue (one key, many values)
is dilution-prone in seven of the eight cells measured
(Table~\ref{tab:partition}), while NIAH-MK3 (three near-tie keys, one of
which is the answer) is capacity-bound in all of them, and the distractor
sweep (Table~\ref{tab:distractor}) places niah\_multikey\_1 (no near-tie distractors) on the
dilution-prone side as well. The \textit{near-tie distractor} structure,
not the surface task category, determines the class, and the head-agreement
drop in Section~\ref{subsec:D-definition} catches the distinction. A non-RULER
test-case confirms this: LongBench passage\_retrieval\_en (pick 1 of 30
paragraphs, $N = 100$, Qwen2.5-1.5B, prompts 10 -- 16K tokens) looks
NIAH-like but has no near-tie distractors, and it is evict-robust in
practice. Accuracy stays flat from $b = 1.0$ to $b = 0.0625$ ($0.29 \to
0.26$, $\rho = 0.01$), while its mean drop $D = 0.105$ sits squarely in
the dilution-prone cluster. The gate opens on all $100/100$ inputs. The predictor correctly classifies this surface-similar but structurally different task a priori.

\subsubsection{Protection-matched ablation}

The partition could be an artifact of our relatively light protection setting.
\citet{garcia2026protectionnearlyneedstructural} argue that structural
protection dominates scorer choice, so we re-ran the Qwen2.5-1.5B 4K
partition tasks with their $10\%$ prefix $+$ $10\%$ suffix protection
($n_{\text{sink}} = w = 410$, $N = 100$ per task). Protection overrides the budget, so all nominal $b \le 0.1875$ share an
effective floor of $824$ kept tokens ($\approx 0.21$ of context). Under
this heavy protection the partition remians unchanged. NIAH-MK3 still
collapses ($0.65$ full-KV $\to 0.09$ at the floor, an $86\%$ relative
loss despite the floor cache being $3.3\times$ our light-protection
cache at nominal $b = 0.0625$), so the capacity-bound failure is not
protection-fixable. The eviction-helps signal survives on the
dilution-prone side ($\rho$: FWE $0.06 \to 0.13$, QA\_1 $0.08 \to 0.07$,
VT $0.06 \to 0.03$, aggregate $0.077 \ge 0.05$), and FWE's floor accuracy
($0.31$) exceeds its full-KV accuracy ($0.22$). Heavy protection is also
not free: protected VT drops to $0.00$ at the floor while our
light-protection run holds $0.75$ at nominal $b = 0.125$, because
protected prefix/suffix tokens crowd out the scorer-selected middle
tokens the task requires. In our harness $w$ is both the recency protection and the SnapKV scoring
window, so this run is a protection-heavy SnapKV variant rather than an
exact replication of their protocol.

\begin{table}[t]
\centering
\caption{Mean head-agreement drop $D$ and gate-open fraction on four additional RULER tasks, $N = 100$ per cell, $\tau = 0.07$. niah\_multiquery (four required keys, no real distractors) closes the gate on two of three models and sits nearest the anchor on the third, confirming a second capacity-bound task. qa\_2 (29 real hard-negative documents), niah\_single\_1 (no distractors), and cwe (ten exact output tokens) all open the gate on most or all inputs, so neither distractor count nor exact-token output length predicts the class.}
\label{tab:capbound}
\begin{tabular}{lccc}
\toprule
Task & Qwen2.5-1.5B & Qwen2.5-3B & Mistral-7B \\
\midrule
niah\_multiquery & $D{=}0.076$, open $32\%$ & $D{=}0.021$, open $100\%$ & $D{=}0.042$, open $100\%$ \\
qa\_2            & $D{=}0.213$, open $0\%$  & $D{=}0.097$, open $10\%$  & $D{=}0.086$, open $18\%$  \\
niah\_single\_1  & $D{=}0.153$, open $0\%$  & $D{=}0.061$, open $87\%$  & $D{=}0.121$, open $0\%$   \\
cwe              & $D{=}0.200$, open $0\%$  & $D{=}0.103$, open $0\%$   & $D{=}0.135$, open $0\%$   \\
\bottomrule
\end{tabular}
\end{table}
\begin{table}[ht]
\centering
\caption{Distractor-count sweep on the NIAH-MultiKey family
(Qwen2.5-1.5B, RULER 4K, $N = 50$ per task). Mean head-agreement drop
$D$ shrinks monotonically as the near-tie distractor count grows.}
\label{tab:distractor}
\begin{tabular}{lccccc}
\toprule
Task & mean $D$ & gate-open frac & $\rho_{\mathrm{plain}}$ & $\rho_{\mathrm{gated}}$ & plain at $b{=}0.0625$ \\
\midrule
niah\_multikey\_1 & $+0.155$ & $1.00$ & $0.04$ & $0.04$ & $0.36$ \\
niah\_multikey\_2 & $+0.122$ & $1.00$ & $0.00$ & $0.00$ & $0.00$ \\
niah\_multikey\_3 & $+0.045$ & $0.00$ & $0.00$ & $0.00$ & $0.00$ \\
\bottomrule
\end{tabular}
\end{table}
\subsubsection{Boundary and confound analysis: full table}
\label{app:capbound-detail}

A single capacity-bound task leaves open whether NIAH-MK3 names a class or an isolated construction. We tested the boundary directly on four further RULER tasks chosen by mechanism rather than by outcome: a second near-tie retrieval task predicted capacity-bound, a real-hard-negative task predicted dilution-prone despite carrying more distractors than any other task we tested, a single-needle control with no distractors at all predicted dilution-prone, and a code-completion task requiring ten exact low-redundancy output tokens predicted capacity-bound. Table~\ref{tab:capbound} reports the outcome at $\tau = 0.07$, $N = 100$ per cell.

The results confirm that the partition is not defined by distractor count or output length. niah\_multiquery is capacity-bound on two of three models, giving the class a confirmed second member. qa\_2 carries the most distractors of any task tested yet sits comfortably on the dilution-prone side. What separates the two classes is whether distractors are near-tie surface-form variants of the answer: the distractor sweep isolates this directly (Table~\ref{tab:distractor}). The class has two entry conditions: a near-tie surface-form distractor or a conjunctive multi-key answer. Both conditions concentrate the answer into a small, non-redundant token set, which is the property the gate responds to.

\subsubsection{Near-tie distractor count drives \texorpdfstring{$D$}{D}}
\label{app:distractor-detail}

The capacity-bound classification rests on a mechanism claim: that
near-tie distractors, not task identity, distractor count, or output
length in general, drive $D$ downward. Section~\ref{app:partition-details} already
shows that real hard-negative distractors do not have this effect, since
qa\_1 and qa\_2 carry the most distractors among the tasks we tested and exhibits the largest
$D$ (Table~\ref{tab:drops}). Exact-token output length also does not, since cwe requires ten exact low-redundancy tokens yet keeps the
gate open on every input tested (Table~\ref{tab:capbound}). This
subsubsection isolates the near-tie case directly. The RULER NIAH-MultiKey family lets us sweep the near-tie
distractor count with the number of documents held fixed:
niah\_multikey\_1 (no near-tie distractors), niah\_multikey\_2 (one),
niah\_multikey\_3 (two). The dilution mechanism of
\citet{DBtrimKV} predicts that more near-tie distractors should
produce larger expected dilution $\delta_t$ and, via
App.~\ref{app:sufficient}, smaller head-agreement drop $D$. On
Qwen2.5-1.5B at RULER 4K, $N = 50$ per task, we observe exactly this: a
clean monotone gradient (Table~\ref{tab:distractor}).

$D$ shrinks by $\sim 3.5\times$ from $0 \to 2$ near-tie distractors,
and per-input ranges between MK2 and MK3 do not overlap
($\max_{\text{MK3}} D = 0.062 < \min_{\text{MK2}} D = 0.088$). The
partition predictor separates 2-distractor inputs from
$\leq 1$-distractor inputs $100\%$ of the time at $\tau = 0.07$. The
monotone gradient replicates on Mistral-7B-v0.3 ($D = 0.081, 0.072,
0.060$ for MK1/MK2/MK3, $N = 50$ per task), confirming the mechanism
is not Qwen-specific, though the finer per-input MK2/MK3 non-overlap is not reproduced
(their ranges overlap on Mistral).

This clean separation between $0$ and $2$ distractors leaves open what
happens at the single-distractor boundary, and MK2 is exactly that case.
It is empirically capacity-bound at low budgets (plain SnapKV collapses
from $80\%$ at $b{=}1$ to $0\%$ at $b{=}0.0625$) yet is classified as
dilution-prone by the predictor at $\tau = 0.07$: MK2's mean $D = 0.122$
sits between the MK3 cluster ($D \leq 0.062$) and the MK1 cluster ($D
\geq 0.099$), co-located with genuinely dilution-prone FWE ($D = 0.098$)
and niah\_multivalue ($D = 0.107$). No single mean-$D$ threshold
separates MK2 from the dilution-prone cluster on this model, so the
predictor falls back to its weakest form here. It correctly classifies
MK3 and the standard dilution-prone tasks but mis-classifies single-near-tie
MK2. Per-input recovery information beyond mean $D$ would be needed to
catch it at deployment.

\subsection{Signal and Method Details}
\label{app:signal-details}


\subsubsection{Gate-signal ablation: extended analysis}

Table~\ref{tab:signal-ablation} gives the headline comparison; Section~\ref{subsec:signal-ablation} summarises the key findings. This subsection provides the extended analysis of the signal's two free parameters.

\paragraph{The signal's two free parameters} $D$ depends on a top-$k$ parameter and a bin
fraction. The justification for $k = 32$ was that
it is commensurate with $|\mathcal{R}|$, which is the quantity the gate
exists to avoid knowing, so an empirical check tests that choice.
Over $k \in \{8, 16, 32, 64, 128\}$ on the selecting cell and on a held-out
Mistral-7B cell ($N = 400$ inputs each), $k = 32$ maximises both
the separation margin and the per-input AUC on \textit{both} cells
(Qwen2.5-1.5B: margin $+0.068$, AUC $1.000$; Mistral-7B: margin $+0.020$,
AUC $0.977$). This empirical validation replaces the original heuristic. The choice is not robust upward. The MK3-lowest ordering on which the gate depends breaks at
$k = 128$ on Qwen2.5-1.5B and at $k \ge 64$ on Mistral-7B. 

\begin{table}[t]
\centering
\caption{Gate-signal ablation repeated on held-out cells. ``sep.'' counts the
signals whose mean on NIAH-MK3 lies strictly below their mean on every
dilution-prone task, which is the property the gate thresholds, and ``margin''
is that gap for $D$. ``margin$^-$'' excludes niah\_multivalue, the exclusion
Appendix~\ref{subsec:layer-profile} places on the Llama-architecture ordering claim.
AUC is per-input separability of NIAH-MK3 from the pooled dilution-prone tasks. The wrapper column records which memory workaround
produced the cell (Appendix~\ref{app:wrapper-provenance}).}
\label{tab:heldout-ablation}
\small
\begin{tabular}{llccccccc}
\toprule
Cell & Status & Wrap. & $D$ MK3 & min dil. & margin & margin$^-$ & AUC($D$) & sep. \\
\midrule
Qwen2.5-1.5B 4K & fitting  & stock & 0.0448 & 0.1114 & $+0.0667$ & $+0.0667$ & 1.000 & 1/6 \\
Qwen2.5-14B 4K  & held out & v1    & 0.0455 & 0.0487 & $+0.0032$ & $+0.0436$ & 0.899 & 1/6 \\
Llama-3.1-8B 4K & held out & v1    & 0.0710 & 0.0607 & $\mathbf{-0.0103}$ & $+0.0073$ & 0.741 & \textbf{0/6} \\
Mistral-7B 16K  & held out & v3    & 0.0544 & 0.0756 & $+0.0212$ & $+0.0393$ & 0.978 & 1/6 \\
\bottomrule
\end{tabular}
\end{table}

The bin fraction is the less sensitive parameter, but only where $D$ works at all. Recomputing
$D$ with early and late bins at $1/4$, $1/3$ and $1/2$ of the layer stack,
from the stored per-layer agreements at no additional inference cost, leaves
the task ordering intact on every Qwen cell we retained. The margin
narrows as the bins widen but never changing sign (Qwen2.5-1.5B 4K:
$+0.084$, $+0.068$, $+0.041$; Qwen2.5-3B 4K: $+0.047$, $+0.042$, $+0.025$;
Qwen2.5-3B 16K: $+0.087$, $+0.067$, $+0.045$). On the two
Llama-architecture cells the ordering is inverted at \textit{every} bin
fraction (Llama-3.1-8B $-0.002$, $-0.011$, $-0.024$; Yi-1.5-9B $-0.033$,
$-0.031$, $-0.032$). This is the same transfer failure
Section~\ref{subsec:signal-validation} reports for the fixed $\tau$ on those families
rather than a separate defect of the binning. No choice of bin fraction
repairs it. The thirds convention is therefore not load-bearing where the
signal transfers, and is not the reason the signal fails where it does not. The
Mistral per-layer logs are corrupt and excluded, so this sweep covers five of
the seven cells.

One methodological point falls out of the $k$ sweep.
Sweeping $k$ at the deployed fixed $\tau = 0.07$ inverts the conclusion.
$k = 8$ then posts $\Delta = +0.239$ on Qwen2.5-1.5B and $+0.391$ on
Mistral-7B, apparently beating $k = 32$, while its gate-open fraction is
$0.00$. The gate has silently become full-KV fallback, and the apparent
gain is the full cache. The location and scale of $D$ move with $k$ (pooled mean $D$ is
$+0.138$ at $k = 32$ but negative at $k = 8$ and $k = 128$), so any sweep of a
parameter that rescales $D$ must recalibrate $\tau$ alongside it, or it
ablates the threshold rather than the parameter. This is the same failure mode
as the Qwen3 fixed-$\tau$ cell of Section~\ref{subsec:calibration}, and we report
gate-open fractions beside every $\Delta$ in the sweep for that reason.

\subsubsection{Held-out signal ablation}
\label{subsec:heldout-ablation}

Table~\ref{tab:signal-ablation} is measured on the cell $D$ was selected. An AUC of $1.000$ there is the number a reader should trust least. We
therefore repeat the ablation on three cells that played no part in selecting
the statistic (Table~\ref{tab:heldout-ablation}), and the result splits into
what survives held-out testing and what does not.

What survives is the qualitative claim: $D$ remains the only one of the six
statistics that achieves task-level separation anywhere, and it does so on
two of the three held-out cells. Every alternative fails on every cell,
held out or not, so the claim Table~\ref{tab:signal-ablation} makes now
rests on more than the measurement that selected the signal.

What does not survive is the fitting cell's exact number. The per-input
AUC of $1.000$ is a property of that cell alone. Held out, it is $0.899$,
$0.978$, and $0.741$, so we report $1.000$ as a fitting-cell figure rather
than as the separability of $D$ in general.

The two weak held-out cells share a common cause: both are decided by a single task, niah\_multivalue. On Llama-3.1-8B, NIAH-MK3 sits at $D =
0.0710$ while niah\_multivalue sits at $0.0607$, giving an unrestricted
margin of $-0.0103$. Consequently, $D$ separates on two of the
three held-out cells, not three, once niah\_multivalue is counted. The
other three dilution-prone tasks lie above MK3, and against the nearest of them (VT
at $0.0783$) the margin is $+0.0073$ if niah\_multivalue is excluded.
That figure is the one Appendix~\ref{subsec:layer-profile} derives from
the depth profile, reached here from an independent measurement and it already places this exclusion on the task-level ordering claim for a stated reason: niah\_multivalue is
eviction-fragile on Llama-architecture models, so a low $D$ there is
arguably the correct call rather than a mis-ordering. We did
not register that exclusion before running this ablation, so we report
it here as a post-hoc diagnostic rather than a confirmed prediction, and
Table~\ref{tab:heldout-ablation} gives both margins so either case is
checkable. On Qwen2.5-14B the margin of $+0.0032$ is likewise set by
niah\_multivalue at $0.0487$, and excluding it the margin is $+0.0436$
against FWE.

\begin{figure}[ht]
\centering
\includegraphics[width=\linewidth]{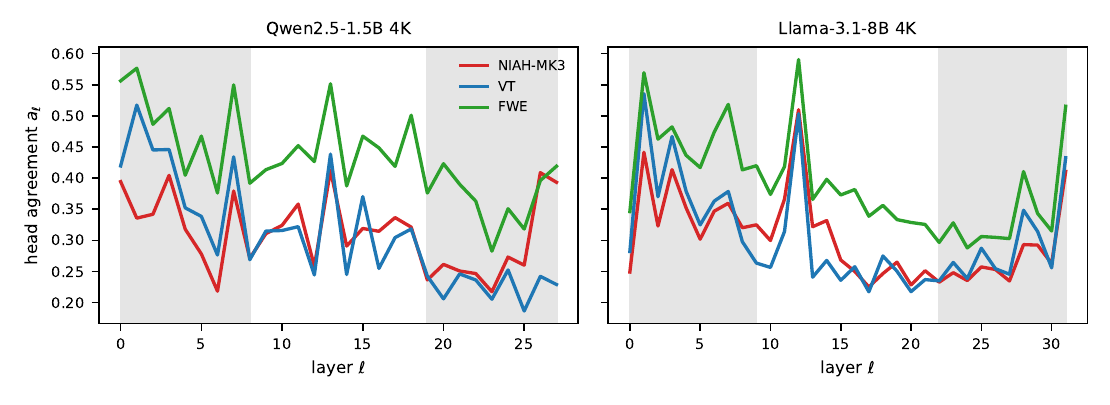}
\caption{Per-layer head agreement $a_\ell$, averaged over inputs
($N = 100$ per task, RULER 4K). Shaded bands are the early and late thirds
that define $D$. Left: on Qwen2.5-1.5B the profile decays with depth and the
capacity-bound task holds the highest late plateau, as (A3) predicts. Right:
on Llama-3.1-8B the profile is non-monotone and NIAH-MK3 is nearly
indistinguishable from VT under the same difference, which is why the fixed
$\tau$ does not transfer to that family.}
\label{fig:layer-profile}
\end{figure}

\subsubsection{The per-layer profile behind \texorpdfstring{$D$}{D}}
\label{subsec:layer-profile}

The small Llama-3.1-8B margin above is a symptom. This subsubsection diagnoses the cause by examining the per-layer profile that $D$ summarizes. Assumption (A3) in App.~\ref{app:sufficient} asserts a layer structure: early layers respond to local context, late layers to task structure, which is why an early-minus-late difference should separate the two classes. The paper elsewhere reports only the binned scalar $D$. We plot the profile $a_\ell$ itself (Figure~\ref{fig:layer-profile}). It is recomputed from the released per-layer logs at zero GPU cost.

On Qwen2.5-1.5B the profile behaves as (A3) describes. Agreement decays with depth, and NIAH-MK3 holds a markedly higher late plateau ($0.283$) than the dilution-prone tasks ($0.227$ on VT). The difference separates them by $+0.068$. On Llama-3.1-8B it does not. The profile is non-monotone, with a
pronounced mid-stack spike, and NIAH-MK3 and VT become nearly indistinguishable under the difference ($D = +0.071$ against $+0.078$, a margin of $+0.007$, an order of magnitude smaller than on Qwen).

We report this because it is the mechanistic content behind a failure the paper otherwise leaves unexplained. The fixed threshold does not transfer to the Llama architecture. Per-input AUC there is $\approx 0.80$ rather than $1.000$, not because the signal is absent but because the depth profile it summarises has a different shape. An early-minus-late difference is the right summary of a monotone profile and the wrong summary of a profile with interior structure, which is also why the depth-correlation statistic comes out anti-correlated on Llama rather than merely weak. A statistic that reads the profile shape rather than its endpoints is the natural next step, and this is
the concrete form of the ``per-input statistic beyond mean $D$'' that Section~\ref{sec:limitations} names as future work. We attempted it, from the stored per-layer agreements at no additional inference cost, and report the outcome because it is informative in both directions.

Three endpoint-free variants of the profile were tested against $D$ on five
cells: the late-plateau level alone, the regression slope of $a_\ell$ on
depth, and the fraction of adjacent layer pairs that decrease. The
late-plateau level, which the difference between MK3 at $0.283$ and VT at
$0.227$ above suggests as the natural candidate, \textit{fails}: its per-input
AUC is at or below chance on every cell ($0.412$, $0.201$, $0.203$, $0.006$,
$0.298$), and on Llama-3.1-8B it is almost perfectly anti-correlated. The
observation that motivated it is real on Qwen and does not generalise. The
regression slope matches $D$ but never exceeds it ($1.000$, $1.000$, $1.000$,
$0.714$, $0.800$ against $D$'s $1.000$, $1.000$, $1.000$, $0.786$, $0.803$),
which indicates the endpoint difference is already extracting most of what the
profile carries. A monotonicity statistic reads higher than $D$ on one cell
(Yi-1.5-9B, $0.912$ against $0.803$) and at chance on another
(Llama-3.1-8B, $0.552$). At $N = 50$ capacity-bound inputs per cell, with
three statistics tested across five cells and no correction for that
multiplicity, we treat the single high value as unexplained rather than as a
result, and do not build on it.

The conclusion we draw is narrow and negative: none of the three shape
statistics transfers uniformly, and none improves on $D$ in a way that
replicates, so the open problem of Section~\ref{sec:limitations} stands
unresolved.

\subsubsection{Calibration recipe details}
\label{subsec:recipe-details}

Section~\ref{subsec:calibration} summarizes the calibration recipe. This
subsubsection gives the full recipe, its validation, its sensitivity to
$\tau$, and the architecture-normalized variant.

\paragraph{Two senses of ``label-free''} \method never observes correctness,
making it \textit{accuracy-label-free}, which is weaker than
\textit{task-label-free}. The fixed-$\tau$ recipe below needs $20$ inputs from
a known capacity-bound task and $20$ from a known dilution-prone one. The
$z$-scored variant assumes neither, and we use the stronger term only for it.

\paragraph{Recipe} For a new (model, context):
\begin{enumerate}
  \item Pick $20$ inputs from a known capacity-bound task (NIAH-MK3) and
  compute the head-agreement drop $D_i$ on each. Average to get
  $\bar D_{\mathcal C}$.
  \item Pick $20$ inputs from a known dilution-prone task (VT) and compute
  $\bar D_{\mathcal D}$.
  \item Set
  \begin{equation}
  \label{eq:tau-app}
  \tau = \tfrac{1}{2}\bigl(\bar D_{\mathcal C} + \bar D_{\mathcal D}\bigr).
  \end{equation}
\end{enumerate}
$40$ prefills, no labels, no decode, $<5$ minutes on a single A100.

\paragraph{Validation} On four cells (Table~\ref{tab:recipe}). The recipe
lifts Qwen2.5-3B 16K from $+3.4$pp to $+5.4$pp (the only failure cell turned
into a win) and lifts Mistral 4K from $+18.7$pp to $+22.3$pp. It is
neutral on average across the four cells but with smaller variance.

\begin{table}[ht]
\centering
\caption{Calibration-recipe validation. $\bar D_{\mathcal C}$ uses NIAH-MK3
($N = 20$); $\bar D_{\mathcal D}$ uses VT ($N = 20$).
Bolded $\tau$-choice column is the winner per cell. The pilot means here
are computed on $20$-input samples and therefore differ slightly from the
$N \ge 50$ per-task means of Table~\ref{tab:drops}.}
\label{tab:recipe}
\begin{tabular}{lccccc}
\toprule
Cell & $\bar D_{\mathcal C}$ & $\bar D_{\mathcal D}$ & $\tau_{\mathrm{cal}}$
& gated@$\tau_{\mathrm{cal}}$ & gated@$\tau\!=\!0.07$ \\
\midrule
Qwen2.5-1.5B 4K & $+0.045$ & $+0.161$ & $+0.103$ & $+0.121$ & $+0.121$ \\
Qwen2.5-3B 4K   & $-0.002$ & $+0.076$ & $+0.037$ & $+0.213$ & $\boldsymbol{+0.259}$ \\
Qwen2.5-3B 16K  & $-0.021$ & $+0.068$ & $+0.024$ & $\boldsymbol{+0.054}$ & $+0.034$ \\
Mistral 4K       & $+0.060$ & $+0.096$ & $+0.078$ & $\boldsymbol{+0.223}$ & $+0.187$ \\
\bottomrule
\end{tabular}
\end{table}

\begin{table}[ht]
\centering
\caption{The three $\tau$ variants used in the paper, their label requirements, and where each is applied, referenced from Section~\ref{subsec:calibration}. Accuracy labels are never used by any variant. The fixed value produces all headline numbers; the sole exception is the Qwen2.5-3B 4K SnapKV matrix cell, run at $\tau = 0.04$ and reported at $\tau = 0.07$ by exact post-hoc re-evaluation (Appendix~\ref{app:repro}).}
\label{tab:tau-variants}
\resizebox{\textwidth}{!}{
\begin{tabular}{lccl}
\toprule
Variant & Accuracy-label-free & Task-label-free & Used for \\
\midrule
Fixed $\tau = 0.07$ & \yes & \yes & Qwen2.5 and Mistral headline results \\
$z$-scored per-model pilot & \yes & \yes & Cross-family transfer (Yi, Llama, Qwen3) \\
40-input midpoint (Eq.~\ref{eq:tau-app}) & \yes & \no & Calibration validation (Table~\ref{tab:recipe}) \\
\bottomrule
\end{tabular}}
\end{table}

\paragraph{Threshold sensitivity}
Sweeping $\tau \in \{0.01, 0.025, 0.04, 0.055, 0.07, 0.085, 0.10, 0.13, 0.16, 0.20\}$ on Qwen2.5-1.5B at RULER 4K (mixed suite, $N=100$, budgets $\{0.5, 0.25, 0.125, 0.0625\}$, SnapKV) shows an
operational plateau $\tau \in [0.055, 0.10]$ inside which $\Delta$ varies by under $1$pp. Both false-positive (gate fires on capacity-bound MK3) and false-negative (gate stays closed on dilution-prone VT/FWE/QA\_1) rates stay at or below $9\%$. The deployed $\tau = 0.07$ sits in the middle of this plateau (Figure~\ref{fig:tau-sensitivity}). Two scope conditions attach to that
statement. The sweep is on the cell that selected the signal, and we do not claim the plateau transfers. It is a fitting-cell measurement, reported here because the deployed constant should not be a critical threshold on the cell it came from. $\Delta$ alone is also the wrong quantity to look for a plateau in, since raising $\tau$ closes the gate more often and substitutes full-KV fallback for
eviction, which raises $\Delta$ monotonically while destroying compression. We therefore read the plateau jointly with the false-positive and false-negative rates above, and report achieved kept-KV alongside $\Delta$ wherever the threshold varies. Outside it the behaviour changes character. Below $0.04$, FP $\geq 64\%$ and $\Delta \leq +0.06$. Above $0.13$, FN $\geq 33\%$ and the method degenerates toward ``always full KV,'' which on this full-KV-strong task set still yields a large $\Delta$, since the sweep's raw maximum is in fact at $\tau = 0.20$, but only at the price of never compressing, so it is not a useful operating point. The deployed threshold is not critically sensitive to the exact value.

\begin{figure}[h]
  \centering
  \includegraphics[width=0.62\linewidth]{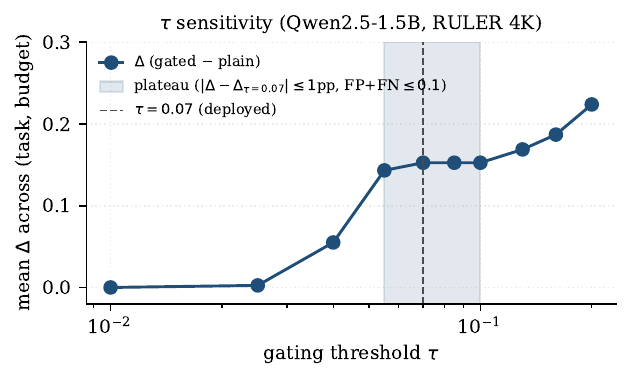}
  \caption{$\tau$ sensitivity on Qwen2.5-1.5B, RULER 4K. Mean
  $\Delta$ (gated $-$ plain) over (task, budget) cells. Shaded band:
  operational plateau where $|\Delta - \Delta_{\tau=0.07}| \leq 1$pp
  and FP$+$FN $\leq 0.10$. The deployed $\tau = 0.07$ sits in the
  middle of the plateau.}
  \label{fig:tau-sensitivity}
\end{figure}

\paragraph{Architecture-normalized variant: the z-scored drop} 
The recipe above calibrates a per-cell threshold from two labeled pilot tasks. A stronger variant needs no task labels at all: standardize the drop per model, $z(x) = (D(x) - \mu_M) / \sigma_M$, with $\mu_M, \sigma_M$ pooled over an unlabeled $\sim 100$-input pilot, and use one global threshold $\theta_z$. Fitting $\theta_z = -0.69$ on the five Qwen/Mistral cells \textit{only} and transferring it zero-shot to the two held-out Llama-architecture models: task-level separation of NIAH-MK3 from $\{$QA\_1, QA\_2, VT, FWE$\}$ succeeds in $7$ of $7$ cells (pooled cross-cell AUC $0.965$), where raw $D$ with one threshold manages $4$ of $7$. The reason is that raw $D$ is an architecture-invariant \textit{ordering} whose location and scale are model-specific: $\tau = 0.07$ sits at $z = +0.69$ on Yi but $z = -0.97$ on Llama. These are two opposite failure modes that standardization removes. Downstream, the z-gate lifts Llama-3.1-8B from $\Delta = +0.109$ to $\boldsymbol{+0.174}$ (NIAH-MK3 gate-open $0.56 \to 0.30$, MK3 accuracy $0.46 \to 0.72$) and gives Yi $\Delta = +0.145$ with \textit{real} compression (kept-KV $0.22$--$0.32$ versus $0.76-0.89$ under the fixed-$\tau$ fallback). Two honest scopes: niah\_multivalue must be excluded from the task-level ordering claim (its mean $D$ sits below MK3 on both Llama-architecture models, which is empirically the right call there, since it is eviction-fragile on both), and per-input MK3-vs-rest AUC on the Llama-architecture models is intrinsically $\approx 0.80$, since no monotone rescaling changes within-model ranking. A depth-correlation statistic over the per-layer profile scores AUC $0.22$ on Llama. This does not mean the depth profile is uninformative there. An AUC of $0.22$ is \textit{anti}-correlated, and inverting its sign gives $0.78$, comparable to $D$'s own $\approx 0.80$ on that family. The profile carries signal on Llama with the opposite sign to the one assumption (A3) predicts, which Appendix~\ref{subsec:layer-profile} shows directly.

\paragraph{A fourth family where fixed $\tau$ fails outright: Qwen} The starkest transfer test is Qwen3-4B-Instruct, a different model generation from the Qwen2.5 calibration set, whose drop distribution is compressed to a different scale entirely ($\mu = 0.034$, $\sigma = 0.017$ pooled over
$400$ inputs). There the fixed $\tau = 0.07$ sits at $z = +2.0$ and the gate essentially never fires. It keeps $99.4\%$ of the cache and its apparent $\Delta = +0.43$ is a pure full-KV-fallback artifact with $0.6\%$ compression. The z-scored threshold ($\tau_{\text{Qwen3}} = \mu + \theta_z \sigma = 0.022$, same $\theta_z = -0.69$) restores a working gate. NIAH-MK3 stays closed (open fraction $0.02$, accuracy $0.05 \to 0.98$) while the three dilution-prone tasks open ($0.94-1.00$), yielding $2.3\times$ compression ($43\%$ mean kept-KV) at $\Delta = +0.235$ over plain and recovering $54.5\%$ of the per-input oracle headroom. The NIAH-MK3-smallest ordering transfers cleanly. Qwen3 is thus the cleanest evidence for the paper's central claim about the predictor: the drop is an architecture-invariant \textit{ordering} whose scale must be standardized per model. Here fixed-$\tau$ does not merely underperform. It does nothing.

The recipe is the deployment-friendly by default as it always fixes the failure cell. The z-scored variant extends it label-free across architecture families. We report fixed-$\tau$ numbers throughout as the conservative default to show both the robustness of one constant and what standardization buys.

\begin{table}[t]
\centering
\caption{Head-pair subsampling (Qwen2.5-1.5B, RULER 4K, $H = 12$, $66$ pairs,
$L = 28$, $N = 30$ inputs per task over four tasks, $120$ input-task rows in
total). ``Margin'' is the gap between mean $D$ on NIAH-MK3 and the
smallest mean $D$ among the dilution-prone tasks; ``ordering'' counts the draws
in which that margin stayed positive.}
\label{tab:head-pair-subsample}
\small
\begin{tabular}{lccccc}
\toprule
Pairs kept & count & mean $D$ MK3 & min mean $D$ dil. & margin & ordering \\
\midrule
$100\%$ & 66 & 0.0413 & 0.1131 & $+0.0717$ & 1/1 \\
$50\%$  & 33 & 0.0404 & 0.1109 & $+0.0705$ & 20/20 \\
$25\%$  & 16 & 0.0456 & 0.1179 & $+0.0724$ & 20/20 \\
$10\%$  & 7  & 0.0280 & 0.1035 & $+0.0754$ & 20/20 \\
$5\%$   & 3  & 0.0541 & 0.1206 & $+0.0665$ & 20/20 \\
$2\%$   & 1  & 0.0529 & 0.1264 & $+0.0735$ & 20/20 \\
\bottomrule
\end{tabular}
\end{table}

\subsubsection{Head-pair and layer subsampling}
\label{subsec:head-pair-subsample}

\paragraph{How many head pairs does $D$ need?}
Layer count is only half of the $O(L H^2 k)$ cost. The other half is the
pair count, and it is the half that scales badly: $66$ pairs at $H = 12$
but $780$ at $H = 40$. The released logs cannot answer this, because the
test averages over pairs before writing, so we re-ran the test-case
retaining the per-pair Jaccards and subsampled offline
(Table~\ref{tab:head-pair-subsample}). Random subsets are drawn $20$
times per fraction, since a single unrepresentative draw would understate
what subsampling can do.

The ordering on which the gate depends survives \textit{every} draw down to a single pair out of
$66$. At $2\%$ of the pairs the margin is $+0.0735$, against $+0.0717$ using all
of them. We read this as a statement about the cost model rather than a
recommendation to gate on a single pair: the $H^2$ term is not load-bearing, so
the gate can be made substantially cheaper on exactly the wide-head models where
it is most expensive, which is the deployment objection of
Section~\ref{sec:limitations}. The caveats are that this is one model at one
context length, that it is an ordering result on task means rather than a
per-input one. The margin is not monotone in the fraction (the $5\%$
row dips to $+0.0665$), which is sampling noise at $120$ inputs.

\paragraph{How many layers does $D$ need?} The gate's cost grows with the
number of layers scanned, so we ask whether a subset suffices. Recomputing $D$
from layer subsets of the stored per-layer agreements needs no GPU. We first
check validity: recomputing over all $L$ layers reproduces the \texttt{drop}
the deployed gate actually thresholded, bit-exactly on the two Qwen cells and
to $4 \times 10^{-4}$ on Mistral, which is two orders of magnitude below
$\tau$. Table~\ref{tab:layer-subsample} then reports whether each subset
preserves the ordering the gate relies on, NIAH-MK3 ranking below every other
task. Halving the layers is safe on all three cells, though the margin shrinks
by roughly a quarter on Qwen2.5-1.5B and by an order of magnitude on
Qwen2.5-3B. Beyond that the signal degrades and is cell-dependent: every-4th
holds only on Qwen2.5-3B, the early bin alone inverts the ordering everywhere,
and the single first-and-last-layer pair happens to work on Qwen2.5-1.5B while
failing badly on Mistral. We therefore do not recommend subsampling below
$L/2$, and we read the incidental successes as a caution against tuning the
layer set per cell.

\begin{table}[h]
\centering
\caption{Effect of computing $D$ from a subset of layers rather
than all of them. ``Margin'' is the gap between mean $D$ on NIAH-MK3 and the
smallest mean $D$ among the other tasks; a positive margin means a separating
threshold still exists on that subset, and ``kept'' marks whether the task
ordering is preserved.}
\label{tab:layer-subsample}
\small
\begin{tabular}{lccccc}
\toprule
 & \multicolumn{2}{c}{Qwen2.5-1.5B ($L{=}28$)} & \multicolumn{2}{c}{Qwen2.5-3B ($L{=}36$)} & Mistral-7B ($L{=}32$) \\
\cmidrule(lr){2-3}\cmidrule(lr){4-5}\cmidrule(lr){6-6}
Layer subset & margin & kept & margin & kept & margin \\
\midrule
all layers        & $+0.0676$ & \yes & $+0.0415$ & \yes & $+0.0174$ \\
every 2nd         & $+0.0520$ & \yes & $+0.0125$ & \yes & $+0.0152$ \\
every 4th         & $-0.0017$ & \no  & $+0.0035$ & \yes & $-0.0194$ \\
early bin only    & $-0.0410$ & \no  & $-0.0280$ & \no  & $-0.0225$ \\
first+last layer  & $+0.0660$ & \yes & $-0.0532$ & \no  & $-0.1541$ \\
\bottomrule
\end{tabular}
\end{table}

\subsubsection{Gate-signal diagnostic}
\label{app:wrapper-provenance}

The diagnostic requests every layer's full $[H, T, T]$ attention, which reaches roughly
$120$\,GiB for Qwen2.5-14B at 4K and cannot be run directly. Two
wrappers make it fit. The retention wrapper (v1) calls the stock
attention and slices the returned tensor, and the query-chunk wrapper
(v3) chunks the query axis while calling the stock function unchanged.
We validated each against the stock script on the fitting cell before
using it: v1 agrees exactly ($D = 0.0412$ against $0.0412$ on the same
host) and v3 agrees to $+0.000029$, which is floating-point
reassociation. We report this validation because a third wrapper failed
the same check. Rewriting attention through SDPA moved mean $D$ from
$0.0448$ to $0.0300$, a $33\%$ shift, reproducibly on two different
GPUs, so the cause is the output path perturbing hidden states and
compounding across layers rather than a hardware difference. No
reported number uses this third wrapper. Separating that $0.0148$ code
artifact from the $0.0035$ hardware effect required varying code path
and hardware together, since either comparison alone against the
released numbers confounds the two.

\subsection{Experimental Details and Secondary Results}
\label{app:experiment-details}

This subsection collects the full tables, decompositions, and secondary
experiments that support the experimental validation of Section~\ref{sec:experiments},
including the random-eviction control (Appendix~\ref{app:random},
Table~\ref{tab:random-control}) and the per-layer eviction ablation
(Appendix~\ref{app:per-layer}, Table~\ref{tab:perlayer}).

\subsubsection{Per-task decomposition of the matrix}
\label{subsec:per-task-delta}

Table~\ref{tab:matrix} reports one number per cell; decomposing that number by
task explains where it comes from. Table~\ref{tab:per-task-delta} breaks
down every matrix cell by task and reports the gate-open fraction that
explains the zeros, and two structural facts follow from it. First,
$\Delta_{\text{non-MK3}}$ is exactly $0.000$ in $8$ of $16$ cells, and in
each of those the gate opens on $100\%$ of non-MK3 inputs, so the gated
and plain arms execute identically: the zero is an identity, not a null
result. Second, the mean $\Delta_{\mathrm{MK3}}$ is $+0.780$ while the
grand mean is $+0.229 \approx \tfrac{1}{4}\Delta_{\mathrm{MK3}}$, the
MK3 share of the mixed suite. Since per-task accuracy is additive under reweighting,
the headline for a workload whose capacity-bound share is $f$ is
$\Delta(f) = f \cdot \Delta_{\mathrm{MK3}} + (1 - f) \cdot
\Delta_{\text{rest}}$, which is near-linear through the origin
(Figure~\ref{fig:mixture-sweep}): $\Delta(0) = +0.045$, $\Delta(0.25) =
+0.229$ (the reported suite), $\Delta(0.5) = +0.412$. A deployment
therefore sees the gate's benefit in proportion to how much
capacity-bound traffic it actually serves, and the realistic-workload
estimate of that share is in Table~\ref{tab:realistic-workload}.

\begin{table}[t]
\centering
\caption{Per-task $\Delta$ (gated $-$ plain, budgets $b < 1.0$, $\tau = 0.07$)
for all $16$ matrix cells, with the gate-open fraction on non-MK3 inputs.
Rows marked $\dagger$ have $\Delta_{\text{non-MK3}}$ exactly $0.000$ with the
gate open on every non-MK3 input, so gated $\equiv$ plain there by
construction. Recomputed from the released per-input logs.}
\label{tab:per-task-delta}
\resizebox{\textwidth}{!}{%
\begin{tabular}{llccccccc}
\toprule
Base evictor & Model & MK3 & VT & FWE & QA\_1 & all & no MK3 & open (non-MK3) \\
\midrule
SnapKV & Qwen2.5-1.5B & $+0.484$ & $0.000$ & $0.000$ & $0.000$ & $+0.121$ & $+0.000^\dagger$ & $1.000$ \\
       & Qwen2.5-3B   & $+0.734$ & $+0.004$ & $+0.300$ & $0.000$ & $+0.259$ & $+0.101$ & $0.653$ \\
       & Qwen2.5-14B  & $+0.573$ & $0.000$ & $0.000$ & $0.000$ & $+0.143$ & $+0.000^\dagger$ & $1.000$ \\
       & Mistral-7B   & $+0.701$ & $0.000$ & $+0.046$ & $0.000$ & $+0.187$ & $+0.015$ & $0.923$ \\
\midrule
H2O & Qwen2.5-1.5B & $+0.647$ & $0.000$ & $0.000$ & $0.000$ & $+0.162$ & $+0.000^\dagger$ & $1.000$ \\
    & Qwen2.5-3B   & $+0.920$ & $+0.017$ & $+0.277$ & $0.000$ & $+0.303$ & $+0.098$ & $0.653$ \\
    & Qwen2.5-14B  & $+0.800$ & $0.000$ & $0.000$ & $0.000$ & $+0.200$ & $+0.000^\dagger$ & $1.000$ \\
    & Mistral-7B   & $+0.883$ & $0.000$ & $+0.043$ & $+0.017$ & $+0.236$ & $+0.020$ & $0.923$ \\
\midrule
StreamingLLM & Qwen2.5-1.5B & $+0.650$ & $0.000$ & $0.000$ & $0.000$ & $+0.163$ & $+0.000^\dagger$ & $1.000$ \\
             & Qwen2.5-3B   & $+0.930$ & $+0.030$ & $+0.760$ & $0.000$ & $+0.430$ & $+0.263$ & $0.653$ \\
             & Qwen2.5-14B  & $+1.000$ & $0.000$ & $0.000$ & $0.000$ & $+0.250$ & $+0.000^\dagger$ & $1.000$ \\
             & Mistral-7B   & $+0.890$ & $0.000$ & $+0.140$ & $+0.020$ & $+0.263$ & $+0.053$ & $0.923$ \\
\midrule
PyramidKV & Qwen2.5-1.5B & $+0.610$ & $0.000$ & $0.000$ & $0.000$ & $+0.152$ & $+0.000^\dagger$ & $1.000$ \\
          & Qwen2.5-3B   & $+0.883$ & $0.000$ & $+0.433$ & $0.000$ & $+0.329$ & $+0.144$ & $0.653$ \\
          & Qwen2.5-14B  & $+0.913$ & $0.000$ & $0.000$ & $0.000$ & $+0.228$ & $+0.000^\dagger$ & $1.000$ \\
          & Mistral-7B   & $+0.863$ & $0.000$ & $+0.060$ & $0.000$ & $+0.231$ & $+0.020$ & $0.923$ \\
\midrule
\textbf{mean} & & $+0.780$ & $+0.003$ & $+0.129$ & $+0.002$ & $\mathbf{+0.229}$ & $\mathbf{+0.045}$ & $0.894$ \\
\bottomrule
\end{tabular}%
}
\end{table}

\begin{table}[ht]
\centering
\caption{The partition on a non-synthetic workload: three LongBench subtasks, Qwen2.5-14B, SnapKV, $\tau = 0.07$. $D$ is the mean head-agreement drop and ``open'' the fraction of inputs on which the gate fires. ``Sensitivity'' is the fraction of full-cache-correct answers plain eviction destroys at $b = 0.0625$ (the capacity signal) and $\rho$ the fraction of inputs on which eviction recovers a full-cache failure (the dilution signal). The gate's class call is made from $D$ alone, before decoding, and matches the empirical class on all three.}
\label{tab:realistic-workload}
\begin{tabular}{lcccccccl}
\toprule
Subtask & $N$ & $A_{\mathrm{full}}$ & sens. & $\rho$ & $D$ & open & empirical class & $\Delta$ \\
\midrule
lcc         & 96  & 0.292 & 0.64 & 0.031 & 0.0543 & 0.250 & capacity-bound & $+0.146$ \\
repobench-p & 82  & 0.317 & 0.42 & 0.061 & 0.0541 & 0.256 & mixed          & $+0.049$ \\
hotpotqa    & 100 & 0.560 & 0.12 & 0.020 & 0.0973 & 0.860 & evict-robust   & $+0.000$ \\
\bottomrule
\end{tabular}
\end{table}

\paragraph{Budget convention} $\Delta$ averages over the eviction budgets
$b < 1.0$. Including $b = 1.0$ leaves the value unchanged because the gated and
plain arms coincide there. Restricting instead to $\{0.125, 0.25, 0.5\}$ gives
$+0.240$, and to the four budgets shared by all cells $+0.243$. We report
$b < 1.0$ throughout.

\begin{figure}[h]
\centering
\includegraphics[width=0.55\linewidth]{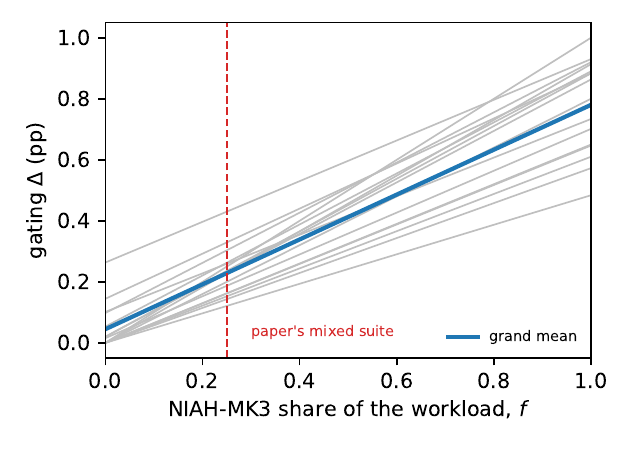}
\caption{Gating $\Delta$ as a function of the capacity-bound share $f$ of the
workload, $\Delta(f) = f \Delta_{\mathrm{MK3}} + (1 - f) \Delta_{\text{rest}}$,
mean-line slope $0.7354$. Grey lines are the $16$ matrix cells, the blue line
their mean, and the dashed
marker the $f = 0.25$ composition of our mixed suite. Because
$\Delta_{\text{rest}}$ is exactly zero in half the cells, the mean is close to
linear through the origin, so the headline $+22.9$pp should be read as a
statement about a workload that is one quarter capacity-bound rather than as a
property of the method.}
\label{fig:mixture-sweep}
\end{figure}

\subsubsection{Budget-axis decomposition}

Table~\ref{tab:qwen3b-curve} decomposes the Qwen2.5-3B 4K
SnapKV cell along the budget axis, complementing the budget-averaged
$\Delta$ values reported in Table~\ref{tab:matrix} and Figure~\ref{fig:matrix}.

\begin{table}[h]
\centering
\caption{Qwen2.5-3B 4K, SnapKV, mixed suite. Plain accuracy collapses
from $0.88$ to $0.19$ as the budget shrinks; gated holds $0.61-0.88$
across the entire range.}
\label{tab:qwen3b-curve}
\begin{tabular}{rccc}
\toprule
budget $b$ & plain & gated & $\Delta$ \\
\midrule
1.000  & 0.882 & 0.882 & $+0.000$ \\
0.875  & 0.802 & 0.880 & $+0.078$ \\
0.750  & 0.733 & 0.880 & $+0.147$ \\
0.625  & 0.688 & 0.880 & $+0.193$ \\
0.500  & 0.647 & 0.882 & $+0.235$ \\
0.375  & 0.585 & 0.873 & $+0.287$ \\
0.250  & 0.535 & 0.877 & $+0.343$ \\
0.125  & 0.458 & 0.830 & $+0.372$ \\
0.0625 & 0.190 & 0.610 & $+0.420$ \\
\bottomrule
\end{tabular}
\end{table}

\begin{figure}[t]
\centering
\includegraphics[width=0.95\textwidth]{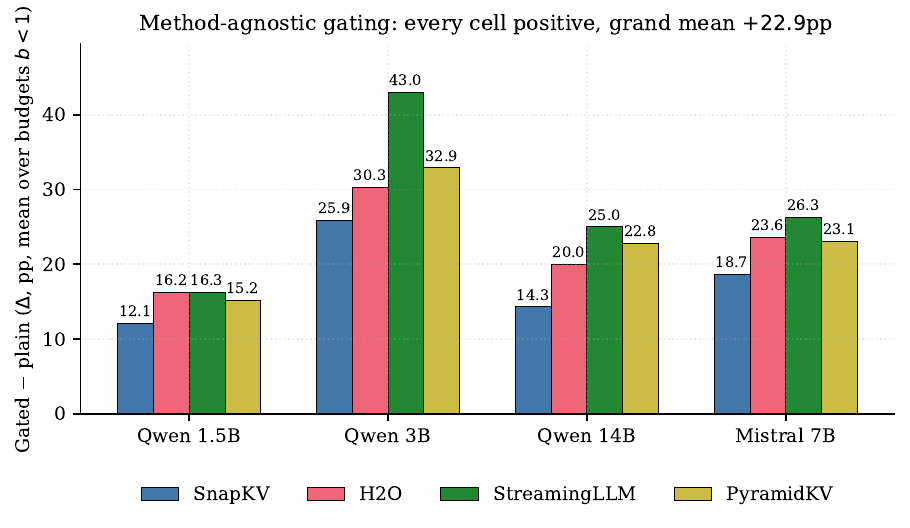}
\caption{The headline $4 \times 4$ method-agnostic matrix.
Bars are gated $-$ plain accuracy averaged over budgets. The wrapper is positive
in every one of $16$ cells with a single $\tau = 0.07$ and no per-base-evictor
tuning. Grand mean $+22.9$pp; excluding NIAH-MK3, $+4.5$pp
(Table~\ref{tab:per-task-delta}).}
\label{fig:matrix}
\end{figure}

\subsubsection{The task-label oracle}
\label{subsec:task-oracle}

Section~\ref{subsec:oracle} summarizes the task-label oracle comparison.
Table~\ref{tab:task-oracle} gives the full per-cell results.

It has three parts: (i) in $8$ of $16$ cells \method
reproduces the oracle \textit{exactly}, so there the gate is a perfect task
classifier and $D$ buys label-freeness rather than per-input resolution,
(ii) in $5$ cells \method \textit{beats} the oracle, by up to $+0.198$, which no
task-level policy can do, since the surplus comes from closing on
individual dilution-prone inputs that the oracle wrongly opens, (iii) in $3$
Mistral cells \method loses by at most $0.009$ through false positives on
MK3. Overall the oracle reaches $+20.1$pp against \method's $+22.9$pp.

So most of the matrix $\Delta$ is a task-level effect that a cheating
baseline could match, and we say so. The per-input claim rests on the
cells where \method exceeds the oracle. Llama-3.1-8B NIAH-MK3, where the mean
$D = 0.071$ sits just \textit{above} the fixed $\tau$ so that a task-level
policy keyed on that mean would open everywhere and gain nothing, does not
support this argument, for two reasons. First, the oracle gap on that cell is
\textit{negative} under both thresholds we report. \method reaches $+0.089$
against the oracle's $+0.162$ at $\tau = 0.07$, and $+0.125$ against $+0.162$
under the $z$-scored threshold, so the cheating baseline wins there rather than
loses. Second, the premise itself is not robust. The margin is $0.0012$ in $D$,
and whether the task mean falls above or below the $z$-scored threshold flips
under three of four plausible definitions of the unlabeled pilot, which
Section~\ref{subsec:calibration} does not pin down. Both quantities are recomputable
from the released per-input log. The reconstruction reproduces the deployed
gate exactly on all $1200$ evicting rows.

\subsubsection{Scorer-independence: full tables}
\label{subsec:scorer-independence-app}

Section~\ref{subsec:scorer-independence} summarizes the scorer-independence result.
The full tables follow.

The single-mask ManifoldKV variant used in the main text is not a fair test of a geometry scorer's full
capability, so we also built a faithful ManifoldKV+Ada-KV
implementation~\citep{manifoldkv,Adakv}: score $s_i =
\lVert k_i - \mu \rVert_2$ against the per-(layer, kv-head) key centroid
$\mu$, with Ada-KV head-wise budget reallocation (safeguard floor
$\texttt{floor\_alpha} = 0.5$ plus a global top-$k$ over the remaining
per-head scores). Ragged per-head keep counts are realized by masking
evicted keys to $-\infty$ per kv-head rather than physically pruning,
which reproduces exactly the logits a ragged compressed cache would
produce with no over-allocation, and the measured kept fraction matches
the nominal budget to four decimals. We add KeyDiff~\citep{keydiff}
(negative cosine to the same centroid) as the published attention-free
scorer.

Before trusting a negative result from this scorer, we first confirm it
is competent. Table~\ref{tab:scorer-independence}(a) is the control that
makes the negative result interpretable: on two-key NIAH the geometry
scorer reproduces exactly the multi-key advantage its authors report,
reaching $0.72$ at $b = 0.5$ where SnapKV is at $0.16$. A scorer that can
do this is not failing MK3 for want of a better scoring rule, and indeed
every scorer we test still collapses on MK3.
Table~\ref{tab:scorer-independence}(b) shows that at $b \le 0.25$,
SnapKV, KeyDiff, and ManifoldKV (post- and pre-RoPE, uniform and Ada-KV
allocation) are all at or near $0.00$ against a full-cache $0.66$. The
geometry scorer does not beat a properly per-head SnapKV at any budget
here and is weaker at $b = 0.5$ ($0.06$-$0.10$ versus $0.32$). The
conclusion we draw is the one the gate is built on, stated with the scope
Section~\ref{sec:limitations} establishes. On this cell, on capacity-bound
inputs, the recoverable information is not present in any $25\%$ subset of the
cache, so none of the scoring rules we test (attention-based, geometric, or
attention-free) recovers it, and the only safe action is not to evict. That
floor is substrate-dependent rather than universal. The same geometry scorer
reaches $0.990$ at $b = 0.5$ on Llama-3.1-8B at 8K, and the budget below which
every scorer fails ranges from $b \le 0.125$ to $b \le 0.25$ across the models
we ran.

\begin{table}[t]
\centering
\caption{\method against a task-label oracle that closes on NIAH-MK3 and opens
elsewhere. $\Delta$ is gated minus plain over budgets $b < 1.0$. A gap of
exactly $0.000$ means the gate reproduced the oracle's decisions on every
input; a positive gap requires within-task per-input variation.}
\label{tab:task-oracle}
\small
\begin{tabular}{llccc}
\toprule
Base evictor & Model & \method $\Delta$ & oracle $\Delta$ & gap \\
\midrule
SnapKV & Qwen2.5-1.5B & $+0.121$ & $+0.121$ & $0.000$ \\
       & Qwen2.5-3B   & $+0.259$ & $+0.183$ & $+0.076$ \\
       & Qwen2.5-14B  & $+0.143$ & $+0.143$ & $0.000$ \\
       & Mistral-7B   & $+0.187$ & $+0.195$ & $-0.008$ \\
\midrule
H2O & Qwen2.5-1.5B & $+0.162$ & $+0.162$ & $0.000$ \\
    & Qwen2.5-3B   & $+0.303$ & $+0.230$ & $+0.073$ \\
    & Qwen2.5-14B  & $+0.200$ & $+0.200$ & $0.000$ \\
    & Mistral-7B   & $+0.236$ & $+0.243$ & $-0.007$ \\
\midrule
StreamingLLM & Qwen2.5-1.5B & $+0.163$ & $+0.163$ & $0.000$ \\
             & Qwen2.5-3B   & $+0.430$ & $+0.232$ & $\mathbf{+0.198}$ \\
             & Qwen2.5-14B  & $+0.250$ & $+0.250$ & $0.000$ \\
             & Mistral-7B   & $+0.263$ & $+0.247$ & $+0.015$ \\
\midrule
PyramidKV & Qwen2.5-1.5B & $+0.152$ & $+0.152$ & $0.000$ \\
          & Qwen2.5-3B   & $+0.329$ & $+0.221$ & $+0.108$ \\
          & Qwen2.5-14B  & $+0.228$ & $+0.228$ & $0.000$ \\
          & Mistral-7B   & $+0.231$ & $+0.240$ & $-0.009$ \\
\midrule
\textbf{grand mean} & & $\mathbf{+0.229}$ & $\mathbf{+0.201}$ & $+0.028$ \\
\bottomrule
\end{tabular}
\end{table}

\begin{table}[h]
\centering
\caption{Scorer-independence of the capacity limit (Qwen2.5-1.5B, RULER 4K,
$N = 50$). (a) The geometry scorer reproduces the published two-key advantage,
so it is competent. (b) On three-key NIAH-MK3 every scorer collapses below
$25\%$ budget. Accuracy at nominal budget $b$.}
\label{tab:scorer-independence}
\small
(a) two-key NIAH, the positive control \\[2pt]
\begin{tabular}{lccc}
\toprule
Scorer & full & $b = 0.5$ & $b = 0.25$ \\
\midrule
SnapKV                 & 0.76 & 0.16 & 0.04 \\
ManifoldKV (post-RoPE) & 0.76 & \textbf{0.72} & \textbf{0.56} \\
ManifoldKV (pre-RoPE)  & 0.76 & \textbf{0.76} & \textbf{0.52} \\
\bottomrule
\end{tabular}

\vspace{6pt}

(b) three-key NIAH-MK3 \\[2pt]
\begin{tabular}{lccccc}
\toprule
Scorer (allocation) & full & $b = 0.5$ & $b = 0.25$ & $b = 0.125$ & $b = 0.0625$ \\
\midrule
SnapKV (uniform per-head)         & 0.66 & 0.20 & 0.02 & 0.00 & 0.00 \\
SnapKV (Ada-KV)                   & 0.66 & 0.32 & 0.00 & 0.00 & 0.00 \\
KeyDiff (uniform)                 & 0.66 & 0.02 & 0.00 & 0.00 & 0.00 \\
KeyDiff (Ada-KV)                  & 0.66 & 0.02 & 0.00 & 0.00 & 0.00 \\
ManifoldKV L2 post-RoPE (uniform) & 0.66 & 0.02 & 0.00 & 0.00 & 0.00 \\
ManifoldKV L2 post-RoPE (Ada-KV)  & 0.66 & 0.06 & 0.00 & 0.00 & 0.00 \\
ManifoldKV L2 pre-RoPE (uniform)  & 0.66 & 0.02 & 0.00 & 0.00 & 0.00 \\
ManifoldKV L2 pre-RoPE (Ada-KV)   & 0.66 & 0.10 & 0.00 & 0.00 & 0.00 \\
\bottomrule
\end{tabular}
\end{table}

\subsubsection{Detect-and-retry as an alternative design}
\label{subsec:retry}

The gate is not the only way to handle capacity-bound inputs. A natural
objection is that one could evict unconditionally and simply re-run the
inputs that fail, avoiding the gate and its attention-exposing pass
entirely. The comparison is analytic, so we state it rather than measure
it. Writing $p_{\mathcal{C}}$ for the proportion of capacity-bound inputs
and $c_{\mathrm{pre}}$, $c_{\mathrm{dec}}$ for prefill and decode cost,
\method pays $c_{\mathrm{gate}}$ on every input plus a memory cost of
$(1 - p_{\text{open}})$ full caches, where $c_{\mathrm{gate}}$ is $17$ms
on Qwen2.5-1.5B and $89$ms on Mistral-7B at 4K (Table~\ref{tab:latency}).
Detect-and-retry instead pays nothing up front but incurs a second full
prefill and decode on every input it retries, so its expected overhead is
$(p_{\mathcal{C}} \cdot \mathrm{TPR} + (1 - p_{\mathcal{C}}) \cdot
\mathrm{FPR}) \cdot (c_{\mathrm{pre}} + c_{\mathrm{dec}})$, and every
capacity-bound input its detector misses is returned wrong rather than
merely late. At the rate $p_{\mathcal{C}} \approx 0.25$ measured on
the realistic workload (Table~\ref{tab:realistic-workload}) and with
$c_{\mathrm{gate}}$ two orders of magnitude below a full
prefill-plus-decode, retry breaks even only if its detector is both
near-perfect in recall and has a false-positive rate below roughly
$0.03$, since each false positive costs a doubled request. The two
designs also differ in kind rather than degree: the gate spends bounded
memory to make a failure impossible, while retry spends unbounded tail
latency to make it recoverable, and only the gate is available to a
system that cannot re-run a request. We note this as a design point
rather than a claim, since we did not build a decode-side distress
detector, and a good one would be a contribution in its own right.

\begin{table}[ht]
\centering
\caption{Cross-architecture SnapKV-only runs on two Llama-architecture models.
These use the four most aggressive budgets $\{0.0625, 0.125, 0.25, 0.5\}$
rather than the eight-budget grid of Table~\ref{tab:matrix}, so the
$\Delta$ values are not directly comparable to the matrix rows
(aggressive-budget-only averaging inflates $\Delta$).}
\label{tab:crossarch}
\begin{tabular}{lccc}
\toprule
Model & plain SnapKV & gated SnapKV & $\Delta$ \\
\midrule
Yi-1.5-9B    & 0.464 & 0.854 & $+0.390$ \\
Llama-3.1-8B & 0.509 & 0.618 & $+0.109$ \\
\bottomrule
\end{tabular}
\end{table}
\begin{table}[ht]
\centering
\caption{Best plain evictor vs.\ best gated evictor per model, over
the four base evictors of Table~\ref{tab:matrix} at the same budget grid.}
\label{tab:bestplain}
\begin{tabular}{lccc}
\toprule
Model & Best plain & Best gated & $\Delta$ \\
\midrule
Qwen2.5-1.5B & 0.438 (SnapKV)    & 0.570 (PyramidKV) & $+0.132$ \\
Qwen2.5-3B   & 0.580 (SnapKV)    & 0.865 (PyramidKV) & $+0.285$ \\
Qwen2.5-14B  & 0.709 (SnapKV)    & 0.868 (H2O)       & $+0.159$ \\
Mistral-7B    & 0.623 (SnapKV)    & 0.823 (PyramidKV) & $+0.201$ \\
\bottomrule
\end{tabular}
\end{table}

\paragraph{Static per-route policy} A serving team with a known
capacity-bound endpoint has a third option that costs nothing at all: route by
endpoint and disable eviction on it, with no attention pass and no detector.
Where that is available it is the right baseline, and by
Table~\ref{tab:task-oracle} a task-level policy already captures roughly
$88\%$ of what \method achieves on our matrix, so the honest comparison is
narrow. \method earns its cost only where routing is unavailable or
insufficient: mixed endpoints that serve both classes, traffic whose class is
not known at admission, and the within-task variation of
Section~\ref{subsec:pareto}, where inputs of the same nominal task fall on
opposite sides. We did not evaluate a routed deployment, and a reader whose
traffic is cleanly separable by endpoint should treat static routing as the
default and the gate as the fallback for what routing cannot reach.

\subsubsection{Composition beyond attention-score evictors}

The comparison above assumes the gate's signal generalizes beyond
attention-score scoring, an assumption worth testing directly. To check
composition beyond attention-score-based bases, we add a ManifoldKV-style
scorer~\citep{manifoldkv}: per-(layer, head) Euclidean distance of
each key to the head's key centroid, averaged into a single shared keep
mask (a single-mask variant of their per-layer design, forced by our
uniform-cache harness). On the Qwen2.5-1.5B 4K mixed suite ($N =
100$/task, four budgets), the same gate at the same $\tau = 0.07$ lifts
it by $\Delta = +16.2$pp with no negative (task, budget) cell, and on
NIAH-MK3 the gate recovers full-KV accuracy ($0.00 \to 0.65$) at every
budget. Two observations follow. The gate signal is attention-based and
scorer-independent, so composition extends outside the attention-score
family, but our single-mask variant does \textit{not} reproduce
ManifoldKV's reported multi-key rescue (plain accuracy $0.01$ at $b =
0.5$). Since their per-layer, per-head design may behave differently, we
treat their published numbers as the authoritative claim and scope ours
accordingly (Section~\ref{sec:limitations}).

\subsubsection{Cross-architecture, scale, and best-plain results}

Table~\ref{tab:crossarch} extends the SnapKV gate to two Llama-architecture models, and
Table~\ref{tab:bestplain} compares the best plain vs.\ best gated evictor per model.

On architecture, Yi-1.5-9B ($\Delta = +39.0$pp) and Llama-3.1-8B
($\Delta = +10.9$pp, Table~\ref{tab:crossarch}) are Llama-architecture
descendants distinct from Qwen and Mistral. Only the SnapKV column was
run for compute, and on a smaller, more aggressive budget grid (see
caption). The sign is positive on both, but on these two models part of
the gain comes from the gate closing conservatively on tasks where
plain SnapKV collapses, rather than from a correctly recovered
partition. The per-task breakdown and the shared Llama-arch family
caveat (niah\_multivalue mis-classifies on both Yi and Llama) are in
Section~\ref{subsec:drops-ordering} and Appendix~\ref{subsec:layer-profile}.

On scale, a SnapKV-only run on Qwen2.5-32B-Instruct (RULER 4K mixed
suite, $N = 100$) confirms the predictor scales but also illustrates the
saturation arm of the scaling formula. The drop ordering transfers
cleanly (NIAH-MK3 smallest at $D = 0.025$, then FWE $0.052$, VT
$0.079$, QA\_1 $0.103$). At the same time the 32B model is
near-saturated at 4K ($A_{\mathrm{full}} = 1.00$ on VT and NIAH-MK3,
$0.93-0.94$ on FWE/QA\_1), so by Eq.~\ref{eq:scaling} there is
essentially no dilution headroom, and the gate's value is
capacity-bound protection rather than recovery: plain SnapKV collapses
at $b = 0.0625$ (NIAH-MK3 $0.00$, VT $0.00$, FWE $0.09$) while the gate
holds the full cache on the low-drop tasks, giving mean $\Delta =
+0.335$ at $\tau = 0.07$ (driven by NIAH-MK3 $+0.88$ and FWE $+0.44$)
and $+0.237$ at the z-scored threshold with more compression ($0.45$
vs.\ $0.63$ kept-KV). We report this as a scale point for the partition
and predictor, not as additional dilution evidence.

The best-plain-versus-best-gated comparison still leaves open why the
matrix's gains differ across base evictors rather than landing at one
number, which the per-evictor behaviour explains. StreamingLLM keeps only
$36$ tokens regardless of budget, so on the dilution-prone fraction it is
near zero plain and near zero gated, and gating contributes only via the
gate-closed NIAH-MK3 inputs that go from $0$ to full-KV accuracy. SnapKV
is budget-scaled and degrades more gracefully, so it has less room for
gating to recover. PyramidKV's depth-weighted scoring biases toward
bottom layers but produces similar absolute gains to SnapKV. The pattern
is that gating $\Delta$ tracks how much the base evictor collapses on
the capacity-bound fraction. Two implementation caveats apply: our
PyramidKV is a depth-weighted variant of the published per-layer budget
schedule, and under two-pass prefill scoring (used on Qwen2.5-14B for
memory) the H2O score only sees the last $w$ queries, so the 14B H2O
column degenerates toward SnapKV-style scoring. Its $\Delta$ there
should be read as a near-duplicate of the SnapKV cell rather than an
independent base evictor.

\subsubsection{Head-to-head tables}

\paragraph{Head-to-head with CapKV on LongBench}
CapKV~\citep{CapKV} is the closest published method that
reports on LongBench. No public CapKV codebase exists yet, since their
preprint is dated 2026-04-28, so we re-implement CapKV's per-(layer, kv-head)
statistical-leverage score $s_i = w_i \cdot v_i^\top A^{-1} v_i$ with
$A = I + \sum_i w_i v_i v_i^\top$ and $w_i = \exp(\langle k_i,
\mu_q\rangle \beta / \sqrt d)$ from their equations,
then wrap it in our gate. Two
LongBench subtasks, Qwen2.5-3B-Instruct, budget $b = 0.5$, CapKV
temperature $\beta = 5$ (their default; we write $\beta$ for their
$\tau$ to avoid clashing with our gate threshold). Results in
Table~\ref{tab:capkv}.

\begin{table}[h]
\centering
\caption{Gated vs.\ plain CapKV (our re-implementation) on two
LongBench subtasks, Qwen2.5-3B-Instruct, $b = 0.5$.}
\label{tab:capkv}
\begin{tabular}{lcccc}
\toprule
Subtask & N & plain CapKV & gated CapKV & $\Delta$ \\
\midrule
LongBench/qasper   & 68 & 0.132 & 0.147 & $+0.015$ \\
LongBench/triviaqa & 24 & 0.875 & 0.958 & $+0.083$ \\
\bottomrule
\end{tabular}
\end{table}

Both columns are positive at matched budget. Conditional on the gate
firing (drop $\geq 0.07$, the same $\tau = 0.07$ we use everywhere),
plain and gated agree by construction. On the gate-closed subset where
the wrapper falls back to full KV, the difference is $+4.3$pp on qasper
($N_{\text{closed}} = 23$) and $+18.2$pp on triviaqa
($N_{\text{closed}} = 11$). The headline numbers are not a reproduction
of CapKV's Table 1 (different model, different metric, 16K eager cap on
prompt length) and we report them as a directional head-to-head: at
these sample sizes ($N = 68, 24$) wrapping CapKV in the gate is
weakly positive overall and concentrates its gain on the gate-closed
subset, consistent with the orthogonality claim of
Section~\ref{subsec:matrix-results}. We do not promote ``never hurts'' to a
formal claim at these sample sizes, and larger LongBench runs are clean
follow-up work.

\paragraph{Head-to-head with DBTrimKV on RULER}
\citet{DBtrimKV} (DBTrimKV) is the published trained-gate evictor
the paper's mechanism section adopts. The authors ship pretrained gates
on Hugging Face for several Qwen3 variants (no Qwen2.5 checkpoint
exists, so the head-to-head is necessarily on a different model
family). We run DBTrimKV-Qwen3-4B-Instruct-2507 on RULER 4K, $N = 30$
inputs per task on (NIAH-MK3, FWE, VT, QA\_1), at three memory budgets
$M \in \{128, 256, 512\}$. The wrapper uses the same fixed
$\tau = 0.07$ calibrated on Qwen2.5-1.5B, with no per-model tuning.
Results in Table~\ref{tab:dbtrimkv}.

\begin{table}[h]
\centering
\caption{Gated vs.\ plain DBTrimKV (pretrained Qwen3-4B-Instruct-2507
gates) on RULER 4K at three memory budgets. The gate fires on only
$1.7\%$ of records, so the win is dominated by the full-KV fallback.}
\label{tab:dbtrimkv}
\begin{tabular}{lcccc}
\toprule
 & N & plain DBTrimKV & gated DBTrimKV & $\Delta$ \\
\midrule
all (3 budgets $\times$ 4 tasks) & 360 & 0.617 & 0.850 & $\boldsymbol{+0.233}$ \\
\midrule
per task: NIAH-MK3 & 90 & 0.467 & 1.000 & $+0.533$ \\
per task: FWE      & 90 & 0.222 & 0.600 & $+0.378$ \\
per task: VT       & 90 & 0.989 & 1.000 & $+0.011$ \\
per task: QA\_1    & 90 & 0.789 & 0.800 & $+0.011$ \\
\midrule
per budget: $M=128$ & 120 & 0.475 & 0.850 & $+0.375$ \\
per budget: $M=256$ & 120 & 0.617 & 0.850 & $+0.233$ \\
per budget: $M=512$ & 120 & 0.758 & 0.850 & $+0.092$ \\
\bottomrule
\end{tabular}
\end{table}

\paragraph{Threshold behaviour: the gate fires rarely on Qwen3 at $\tau = 0.07$}
At our default threshold the gate is open on $6/360 = 1.7\%$ of
records. The remaining $\sim 98\%$ trigger the full-KV fallback, which
explains the large $\Delta$ on NIAH-MK3 ($+0.533$): the gate closing
turns off DBTrimKV's compress step entirely. As a consequence, cache
occupancy under our wrapper is $\sim 14.5 \times$ plain DBTrimKV at the
same nominal $M$. The $+0.233$ is therefore a \textit{quality} win at
matched budget, not a \textit{memory-budget} win at matched cache size.
The rare gate firing here is the same Qwen3 scale mismatch analyzed in
Section~\ref{subsec:calibration}: Qwen3-4B's drop distribution ($\mu = 0.034$)
puts the fixed $\tau = 0.07$ at $z = +2.0$, so a per-model z-scored
threshold ($\tau_{\text{Qwen3}} = 0.022$) is the right deployment
default and, applied to the same suite, restores a genuinely
compressing gate ($43\%$ kept-KV, $\Delta = +0.235$). The headline is that wrapping the strongest
published trained-gate evictor with our partition-aware gate
strictly improves matched-budget accuracy, with the gain concentrated
on the capacity-bound task as the mechanism predicts.

\subsubsection{Random-eviction control}
\label{app:random}

We replace SnapKV scoring with uniform-random per-position scoring at every budget on a 4K, 4-key
synthetic NIAH test-case ($N = 150$, Qwen2.5-1.5B). To be clear about the source, this 4-key test-case is our
own construction and is \textit{not} a RULER task: the NIAH-MK$k$ names used elsewhere in the paper refer to
RULER's \texttt{niah\_multikey\_$k$}~\citep{Ruler}, which tops
out at three keys. The score-driven sweep retains accuracy at moderate
budgets, while the uniform-random sweep collapses immediately
(Table~\ref{tab:random-control}).

\begin{table}[h]
\centering
\caption{SnapKV vs.\ uniform-random scoring at matched budget on the 4K
4-key synthetic NIAH test-case. Random eviction destroys accuracy long before
the budget gets aggressive; the dilution effect requires informed scoring.}
\label{tab:random-control}
\begin{tabular}{rccc}
\toprule
budget $b$ & SnapKV acc & Random acc & $\Delta$ \\
\midrule
1.000  & 0.827 & 0.827 & $+0.000$ \\
0.875  & 0.820 & 0.400 & $+0.420$ \\
0.750  & 0.820 & 0.173 & $+0.647$ \\
0.625  & 0.827 & 0.040 & $+0.787$ \\
0.500  & 0.833 & 0.033 & $+0.800$ \\
0.375  & 0.807 & 0.007 & $+0.800$ \\
0.250  & 0.760 & 0.000 & $+0.760$ \\
0.125  & 0.207 & 0.000 & $+0.207$ \\
0.0625 & 0.033 & 0.000 & $+0.033$ \\
\bottomrule
\end{tabular}
\end{table}

The per-input recovery rate $\rho$ under random scoring is $0.027$
(vs.\ $0.013$ under SnapKV). The four ``recovered'' inputs under random
are isolated successes among an otherwise destroyed distribution, and the
control rules out a generic regularisation explanation: any-eviction-helps
is falsified.

\subsubsection{Per-layer eviction ablation}
\label{app:per-layer}

The default gating uses a single key-mask shared across layers. A per-layer
variant derives each layer's keep mask from that layer's own last-window
attention with the same total budget per layer (uniform cache shape).
Table~\ref{tab:perlayer} reports $\rho$ on the four dilution-prone RULER tasks
where per-layer might plausibly help, on Qwen2.5-1.5B at 4K and 16K.

\begin{table}[h]
\centering
\caption{Per-layer vs.\ single-mask SnapKV on Qwen2.5-1.5B. Per-layer wins
clearly on aggregation tasks at 4K (FWE $1.8\times$, niah\_multivalue
$2.1\times$); ties at 16K and on sequential-integration tasks.}
\label{tab:perlayer}
\begin{tabular}{lcccc}
\toprule
Task & Context & Single-mask $\rho$ & Per-layer $\rho$ & Ratio \\
\midrule
VT               & 4K  & 0.060 & 0.040 & 0.67 \\
FWE              & 4K  & 0.060 & \textbf{0.110} & 1.83 \\
QA\_1            & 4K  & 0.080 & 0.070 & 0.88 \\
niah\_multivalue & 4K  & 0.070 & \textbf{0.150} & 2.14 \\
VT               & 16K & 0.190 & 0.190 & 1.00 \\
FWE              & 16K & 0.040 & 0.050 & 1.25 \\
QA\_1            & 16K & 0.140 & 0.140 & 1.00 \\
niah\_multivalue & 16K & 0.100 & 0.100 & 1.00 \\
\bottomrule
\end{tabular}
\end{table}

\subsection{Deployment Details}
\label{app:deployment-details}

This subsection collects the tables and secondary analyses behind the
deployment constraints of Section~\ref{sec:limitations}.

\subsubsection{Deferred experiment tables}

\paragraph{The 8-cell scaling table} Table~\ref{tab:scaling} reports the
gating $\Delta$ on the mixed suite across four models and two context
lengths. The pattern matches the context-length scaling formula: at 4K,
$\Delta$ tracks both $1 - A_{\mathrm{full}}$ and $\alpha_{\mathcal{R}}$; at
16K, $\Delta$ shrinks where the gate misfires (Qwen 3B 16K) and grows where
headroom returns (Mistral 16K).

\begin{table}[h]
\centering
\caption{Cross-(model, context) gating $\Delta$. Same $\tau = 0.07$
across cells; column $\Delta@\tau_{\mathrm{cal}}$ shows the recipe-calibrated
$\Delta$ on the four cells where it has been computed (``--'' otherwise).}
\label{tab:scaling}
\begin{tabular}{lccccc}
\toprule
Model & Context & plain mean & gated@$\tau\!=\!0.07$ & $\Delta$ &
$\Delta@\tau_{\mathrm{cal}}$ \\
\midrule
Qwen2.5-1.5B   & 4K  & 0.438 & 0.559 & $+0.121$ & $+0.121$ \\
Qwen2.5-3B     & 4K  & 0.580 & 0.839 & $+0.259$ & $+0.213$ \\
Qwen2.5-14B    & 4K  & 0.709 & 0.853 & $+0.143$ & --       \\
Mistral-7B      & 4K  & 0.623 & 0.810 & $+0.187$ & $+0.223$ \\
\midrule
Qwen2.5-1.5B   & 16K & 0.372 & 0.404 & $+0.032$ & --       \\
Qwen2.5-3B     & 16K & 0.530 & 0.564 & $+0.034$ & $+0.054$ \\
Qwen2.5-14B    & 16K & 0.796 & 0.867 & $+0.071$ & --       \\
Mistral-7B      & 16K & 0.535 & 0.651 & $+0.116$ & --       \\
\bottomrule
\end{tabular}
\end{table}

\paragraph{Pre-registered 32K predictions} Table~\ref{tab:prereg32k}
gives the full per-task predictions and measurements behind the
out-of-range test of Section~\ref{subsec:prereg32k}.

\begin{table}[h]
\centering
\caption{Pre-registered 32K predictions vs.\ measurements
(Qwen2.5-1.5B, $N = 100$ per task). Predicted intervals = 4K/16K ratio
band $\times$ measured 32K headroom, frozen before the sweep.}
\label{tab:prereg32k}
\resizebox{\textwidth}{!}{%
\begin{tabular}{lcccc}
\toprule
Task & $A_{\mathrm{full}}$(32K) & predicted $\rho$ & measured $\rho$ (Wilson 95\%) & verdict \\
\midrule
VT  & 0.81 & $[0.048, 0.106]$ & $0.14$ $[0.085, 0.221]$ & outside (above) \\
FWE & 0.49 & $[0.021, 0.263]$ & $0.05$ $[0.022, 0.112]$ & inside \\
NIAH-MK3 & 0.14 & $\le 0.02$ & $0.04$ $[0.016, 0.098]$ & outside (above; CI overlaps) \\
\bottomrule
\end{tabular}%
}
\end{table}

\subsubsection{Achieved compression and batching decay}

\paragraph{Achieved compression of the gated deployment} Table~\ref{tab:achieved-compression}
reports the per-model gate-open fraction and realized kept-KV behind the
compression figures quoted in the abstract and Section~\ref{sec:intro}.
Realized compression at $b = 0.0625$ ranges from $1.8\times$ on Qwen2.5-3B,
whose gate-open fraction is lowest at $0.49$ and whose fallback therefore
costs most, to $3.4\times$ on Qwen2.5-1.5B and Qwen2.5-14B, with a mean of
$2.9\times$. The $3.4\times$ endpoint alone understates this spread. We
report the range and the mean instead, since the spread is a property of
the workload the gate sees rather than measurement noise.

\begin{table}[t]
\centering
\caption{Achieved compression of the gated deployment at $\tau=0.07$ (SnapKV,
4K mixed suite). Kept-KV is the mean over inputs of the retained cache
fraction; ``grid mean'' averages the eight nominal budgets of
Table~\ref{tab:matrix}. Plain SnapKV keeps exactly the nominal fraction
($0.0625$ at $b=0.0625$; grid mean $0.453$). The gate-open fraction is a
per-input property, independent of $b$, and sets the compression floor
$1-p_{\text{open}}$.}
\label{tab:achieved-compression}
\begin{tabular}{lccccc}
\toprule
 & & \multicolumn{3}{c}{gated kept-KV fraction} & compression \\
\cmidrule(lr){3-5}
Model & $p_{\text{open}}$ & $b{=}0.0625$ & $b{=}0.25$ & grid mean & @ $b{=}0.0625$ \\
\midrule
Qwen2.5-1.5B & $0.75$ & $0.297$ & $0.437$ & $0.584$ & $3.4\times$ \\
Qwen2.5-3B   & $0.49$ & $0.541$ & $0.632$ & $0.728$ & $1.8\times$ \\
Qwen2.5-14B  & $0.75$ & $0.297$ & $0.437$ & $0.584$ & $3.4\times$ \\
Mistral-7B   & $0.72$ & $0.327$ & $0.462$ & $0.602$ & $3.1\times$ \\
\bottomrule
\end{tabular}
\end{table}

\paragraph{Batching decay under static provisioning}
Table~\ref{tab:batching-decay} evaluates the compression decay per model
under static batch provisioning, where a batch is provisioned for its largest resident cache.

\begin{table}[t]
\centering
\caption{Compression under static batch provisioning at $\tau = 0.07$
(SnapKV, $b = 0.0625$). $p_{\text{open}}$ and $k_{\text{open}}$ are from
Table~\ref{tab:achieved-compression}. Compression decays toward $1\times$
as batch size grows because a single gate-closed sequence forces the
whole batch to the full cache.}
\label{tab:batching-decay}
\begin{tabular}{lccccc}
\toprule
Model & $B{=}1$ & $B{=}4$ & $B{=}8$ & $B{=}16$ & $B{=}32$ \\
\midrule
Qwen2.5-1.5B & $3.37\times$ & $1.42\times$ & $1.10\times$ & $1.01\times$ & $1.00\times$ \\
Qwen2.5-3B   & $1.85\times$ & $1.06\times$ & $1.00\times$ & $1.00\times$ & $1.00\times$ \\
Qwen2.5-14B  & $3.37\times$ & $1.42\times$ & $1.10\times$ & $1.01\times$ & $1.00\times$ \\
Mistral-7B   & $3.06\times$ & $1.33\times$ & $1.07\times$ & $1.00\times$ & $1.00\times$ \\
\bottomrule
\end{tabular}
\end{table}

\subsubsection{Per-task drop ordering}
\label{subsec:drops-ordering}

Table~\ref{tab:drops} reports the mean drop per task ($N = 100$ inputs
per task on the Qwen and Mistral columns, $N = 50$ on Yi and Llama;
Mistral 4K qa\_1 has $N = 78$ and Mistral 16K niah\_multivalue $N = 98$
due to run truncation). NIAH-MK3 is the smallest drop in every Qwen and
Mistral cell, while on the two Llama-architecture models it ranks second-smallest,
so the \textit{ordering} of the partition is fully preserved across Qwen
and Mistral and partially preserved on the Llama-architecture models, even when the
absolute threshold needs slight tuning. This is
what gives the calibration recipe in Eq.~\ref{eq:tau-app} its robustness.

\begin{table}[h]
\centering
\caption{Per-task mean head-agreement drop $D$ across three
architecture families (Qwen, Mistral, and Llama-architecture: Yi-1.5 and
Llama-3.1). This table uses the strict-thirds bins, early
$= [0, \lfloor L/3 \rfloor)$ and late $= [\lfloor 2L/3 \rfloor, L)$; the
deployed gate uses equal-size bins and differs by at most $0.004$ where
$3 \nmid L$ (Appendix~\ref{app:repro}). Bold marks the smallest drop
per column. NIAH-MK3 is
the smallest drop on Qwen and Mistral cells; on Yi-1.5 and Llama-3.1
it ranks 2nd-smallest behind niah\_multivalue (Llama-arch family
artifact; see prose). Yi-1.5-9B
serves as a parallel Llama-class architecture with a distinct training
corpus (different tokenizer, deeper layer geometry).}
\label{tab:drops}
\resizebox{\textwidth}{!}{%
\begin{tabular}{lccccccc}
\toprule
Task & Qwen 1.5B 4K & Qwen 3B 4K & Qwen 3B 16K & Mistral 4K & Mistral 16K & Yi-1.5-9B 4K & Llama-3.1-8B 4K \\
\midrule
QA\_1            & 0.218 & 0.105  & 0.099  & 0.096 & 0.132 & 0.080         & 0.097 \\
QA\_2            & 0.191 & 0.097  & 0.092  & 0.087 & --    & 0.079         & 0.101 \\
VT               & 0.152 & 0.076  & 0.069  & 0.095 & 0.096 & 0.065         & 0.083 \\
niah\_multivalue & 0.107 & 0.045  & 0.070  & 0.079 & 0.075 & \textbf{$-0.013$} & \textbf{0.064} \\
FWE              & 0.098 & 0.038  & 0.046  & 0.079 & 0.091 & 0.043         & 0.114 \\
NIAH-MK3         & \textbf{0.040} & \textbf{$-0.004$} & \textbf{$-0.021$} & \textbf{0.060} & \textbf{0.053} & 0.019 & 0.073 \\
\bottomrule
\end{tabular}%
}
\end{table}

\paragraph{Llama-class cross-architecture transfer is partial} On
Yi-1.5-9B and Llama-3.1-8B (both Llama-architecture descendants),
niah\_multivalue ranks smallest by $D$ (Yi: $D = -0.013$; Llama: $D =
+0.064$), and NIAH-MK3 ranks 2nd-smallest. On the five Qwen and
Mistral cells, NIAH-MK3 is the smallest. We attribute the swap to a
Llama-arch artifact: the multi-value retrieval task happens to have
unusually concentrated early-layer attention on these models, which
inverts the early-vs-late drop sign. The capacity-bound side of the
partition (small $D$, gate closes, full-KV fallback) still operates on
both models, but the fixed $\tau = 0.07$ does \textit{not} transfer as a
threshold. On Llama the mis-classification is confined to
niah\_multivalue, where the wrapper holds full KV instead of
compressing, which forfeits the eviction speed-up but does not damage
accuracy, because plain SnapKV is already robust there. The
Llama-3.1-8B gated SnapKV $\Delta$ is $+10.9$pp on the
matrix-comparable mixed suite, driven almost entirely by NIAH-MK3
($+43.5$pp) where plain SnapKV collapses to $0\%$ and gated holds
$\sim 44\%$ (still well below the full-KV $1.00$). The other tasks
have $\Delta \approx 0$ because Llama's plain SnapKV is robust on this
budget range ($\geq 0.86$ at $b = 0.0625$ on QA\_1). The NIAH-MK3
gate-open fraction on Llama is $0.56$ despite mean $D = 0.071$ sitting
above $\tau$: per-input $D$ dispersion closes the gate on the other
$\sim 44\%$ of NIAH-MK3 inputs, which is what rescues the cell. This cell does
\textit{not} support a per-input claim over a task-level policy,
since the task-label oracle outperforms \method here under both thresholds
(Section~\ref{subsec:pareto}). The dispersion explains the cell's $\Delta$, not
an advantage over knowing the label. On Yi
the mis-classification is broader: at $\tau = 0.07$ the gate also
stays closed on FWE and VT (gate-open fractions $0.00$ and $0.10$), so
most of Yi's $+39.0$pp comes from the full-KV fallback firing on tasks
where plain SnapKV collapses. That outcome is accuracy-preserving but
it is \textit{not} evidence that the raw threshold transfers to Yi, since the
drop-vs-$\rho$ prediction is inverted there. The z-scored variant of
Section~\ref{subsec:calibration} resolves this: with a per-model standardization
from an unlabeled pilot, Yi compresses for real (kept-KV
$0.22 - 0.32$) while still gaining $\Delta = +0.145$, and Llama
improves to $\Delta = +0.174$. We state the fixed-$\tau$ failure
plainly in Section~\ref{subsec:signal-validation}.

\subsubsection{Scaling formula verification}

Table~\ref{tab:scaling-formula}
reports the headroom $1 - A_{\mathrm{full}}$, the observed $\rho$, and the
ratio. A through-origin fit on the dilution-prone rows gives a slope of
$0.34$. The Qwen2.5-3B 4K FWE row (ratio $0.04$) is excluded from that fit:
at $N = 100$, $\rho = 0.01$ is a single recovery, Wilson $[0.002, 0.054]$,
and is better read as counting noise than as part of the trend. On
NIAH-MK3 the ratio is $\le 0.02$ regardless of headroom, so the partition is
visible in the ratio alone.

\begin{table}[h]
\centering
\caption{Verification of the scaling formula. The ratio
$\rho / (1 - A_{\mathrm{full}})$ is approximately constant on
dilution-prone tasks across (model, context) cells and near zero on
NIAH-MK3 regardless of headroom.}
\label{tab:scaling-formula}
\begin{tabular}{lccc}
\toprule
Sweep & $1 - A_{\mathrm{full}}$ & $\rho$ & ratio \\
\midrule
Qwen2.5-1.5B 4K  VT    & 0.18 & 0.06 & 0.33 \\
Qwen2.5-1.5B 16K VT    & 0.75 & 0.19 & 0.25 \\
Qwen2.5-3B   4K  VT    & 0.00 & 0.00 & --   \\
Qwen2.5-3B   16K VT    & 0.09 & 0.05 & 0.55 \\
Qwen2.5-3B   4K  FWE   & 0.24 & 0.01 & 0.04 \\
Qwen2.5-3B   16K FWE   & 0.62 & 0.32 & 0.52 \\
Qwen2.5-1.5B 4K  NIAH-MK3 & 0.35 & 0.01 & 0.03 \\
Qwen2.5-1.5B 16K NIAH-MK3 & 0.74 & 0.00 & 0.00 \\
\bottomrule
\end{tabular}
\end{table}

\begin{figure}[t]
\centering
\includegraphics[width=0.7\textwidth]{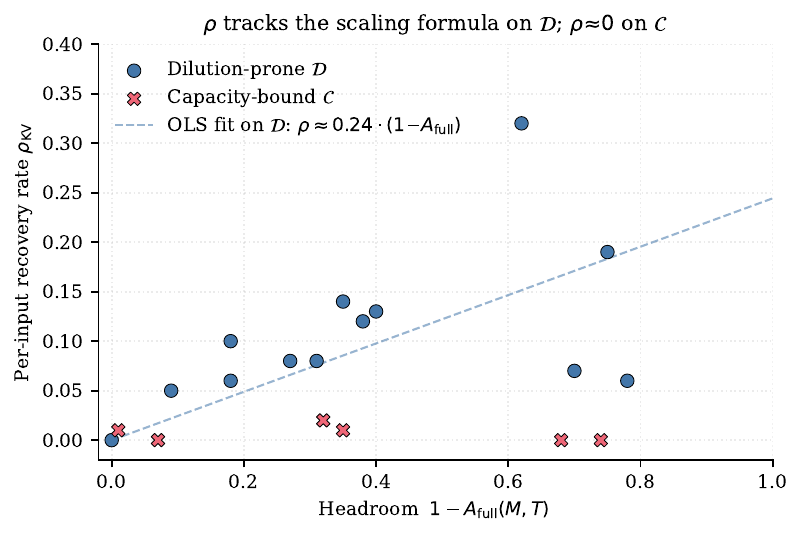}
\caption{Per-input recovery rate $\rho$ vs.\ accuracy headroom $1 -
A_{\mathrm{full}}$ across the $(model, context, task)$ cells of
Table~\ref{tab:scaling-formula}. Dilution-prone tasks (filled circles)
lie on a band with a through-origin fitted slope of $0.34$. NIAH-MK3
(open triangles) lies on the $\rho = 0$ axis regardless of headroom. The
two populations are visibly separated.}
\label{fig:scaling}
\end{figure}

\subsubsection{Cost of the gate and end-to-end efficiency}

Table~\ref{tab:latency} reports the end-to-end efficiency of gated eviction.

\begin{table}[t]
\centering
\caption{End-to-end efficiency of gated eviction (RULER 4K, batch 1,
bf16, greedy, 128 decode steps, medians; A100-80GB). ``Gated expected''
is the gate-open-weighted mix $0.75 \cdot (b{=}0.125) + 0.25 \cdot \mathrm{full}$ at
the measured mixed-suite gate-open fraction ($\tau{=}0.07$); ``gated
measured'' actually gates each input ($15/20$ open, matching the
$N{=}400$ calibration fraction of $0.75$). On the 1.5B model (2 KV
heads), 4K decode is weight-bound and eviction buys \textit{memory}
($8\times$ plain, $2.9\times$ gated expected), not speed; on
Mistral-7B (8 KV heads, 513\,MiB KV) eviction already yields a
$1.29\times$ decode speedup at 4K. Measured KV matches the analytic
$L \cdot H_{\mathrm{kv}} \cdot d \cdot T \cdot 4$\,B formula to
$0.1$\,MiB, validating the 16K/32K scaling rows.}
\label{tab:latency}
\small
\setlength{\tabcolsep}{4pt}
\resizebox{\textwidth}{!}{%
\begin{tabular}{llrrr}
\toprule
Model & Config & Decode tok/s & KV cache (MiB) & Peak decode (MiB) \\
\midrule
\multirow{5}{*}{\shortstack[l]{Qwen2.5-1.5B\\($N{=}20$)}}
 & Full KV                   & 37.9 & 106.1 & 3090 \\
 & SnapKV $b{=}0.125$        & 37.8 & 13.2 ($8.0\times$)  & 2976 \\
 & SnapKV $b{=}0.0625$       & 37.5 & 6.6 ($16.1\times$)  & 2968 \\
 & Gated (measured, $\tau{=}0.07$) & 37.7 & 13.2 / 110.5 (open/closed) & 2976 \\
 & Gated (expected mix)      & 37.9 & 36.5 ($2.9\times$)  & 3004 \\
\midrule
\multirow{2}{*}{\shortstack[l]{Mistral-7B-v0.3\\($N{=}12$)}}
 & Full KV                   & 18.9 & 512.8 & 14429 \\
 & SnapKV $b{=}0.125$        & 24.5 ($1.29\times$) & 64.1 ($8.0\times$) & 13929 \\
\bottomrule
\end{tabular}%
}

\vspace{0.6em}

\small
\setlength{\tabcolsep}{4pt}
\resizebox{\textwidth}{!}{%
\begin{tabular}{lrlr}
\toprule
\multicolumn{2}{l}{Prefill-time overhead (Qwen2.5-1.5B, 4K, median ms)} &
\multicolumn{2}{l}{Analytic KV at $b{=}0.125$ gated (0.75 open)} \\
\midrule
Prefill (eager)            & 344  & Qwen2.5-1.5B, 16K / 32K & 448 $\to$ 154 / 896 $\to$ 308 MiB \\
SnapKV scoring             & 1.2  & Mistral-7B, 16K / 32K   & 2.0 $\to$ 0.69 / 4.0 $\to$ 1.38 GiB \\
Head-agreement gate        & 17.2 & \multicolumn{2}{l}{(gate on Mistral-7B: 88.5\,ms, $O(LH^2k)$ head-pair growth)} \\
Eviction + cache rebuild   & 3.2  & & \\
\bottomrule
\end{tabular}%
}
\end{table}

\paragraph{Deployment cost at scale} We measured the two-pass gate on
Qwen2.5-32B (64 layers, 40 heads): the gate costs $350$ms at 4K ($145$ms
re-forward $+ 204$ms head-agreement drop, $\approx 18\%$ of the $1.9$s prefill),
and the 16K two-pass run \textit{OOMs} on a contended 80GB card (weights $+$ KV
$+$ the $w \times T$ attention transient). The drop cost is quadratic in head
count: a head-count sweep gives an exponent of $2.0 - 2.4$ (per-model floors
$2.1$ms on Qwen-1.5B, $27$ms Mistral-7B, $87$ms Qwen-32B, extrapolating to
$\approx 291$ms on a 70B-class model), confirming the $O(L H^2 k)$ complexity.
The $w \times T$ attention transient is $L H w T \cdot 2$ bytes ($5$GiB at 32K
on Qwen-32B, $\approx 10$GiB on 70B-class). A 70B model at 32K is therefore not
viable in a paged-FlashAttention stack: the two-pass re-forward needs contiguous
(non-paged) KV and a non-fused eager kernel, and materializes a per-sequence
transient that paged attention exists to avoid. \method is best read as an
offline/prefill-time gate, not a fused-serving drop-in.


\begin{figure}[ht]
\centering
\includegraphics[width=0.95\linewidth]{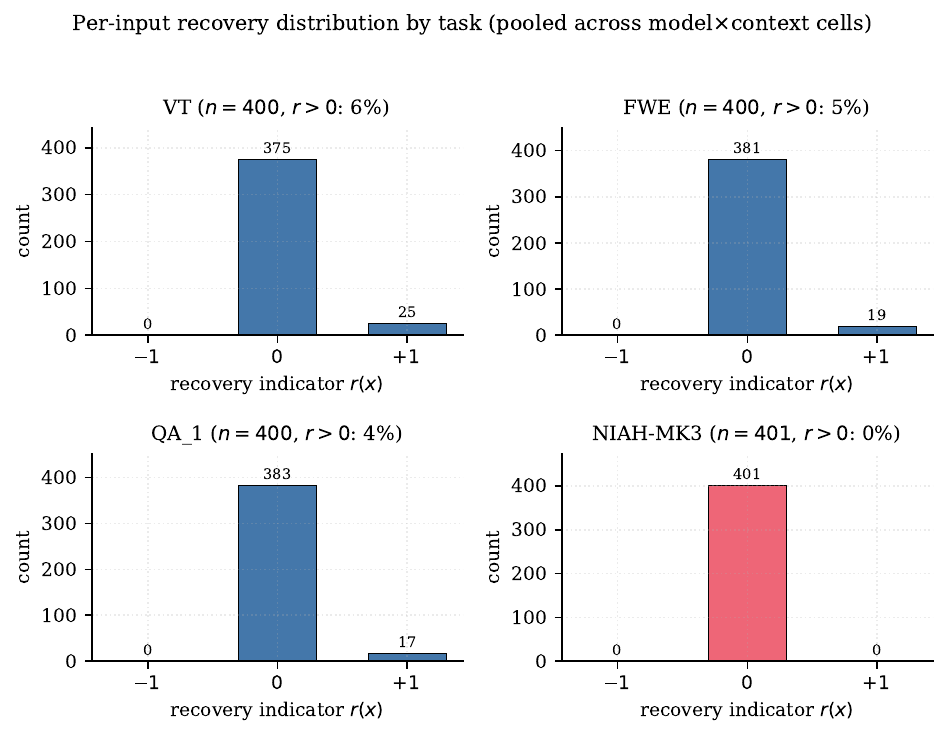}
\caption{Per-input recovery indicator $r(x)$ pooled across model and
context cells, faceted by RULER task. The three dilution-prone tasks show
a non-trivial right tail at $r{=}{+}1$; NIAH-MK3 is concentrated entirely
at $r{=}0$, consistent with the capacity-bound regime. The $r{=}{-}1$ bin
is empty by construction, not by measurement, because the budget sweep
includes $b = 1.0$; see the identity argument above and
Table~\ref{tab:harm-rate} for the harm rate this figure cannot report.}
\label{fig:reco-by-task}
\end{figure}

\begin{figure}[ht]
\centering
\includegraphics[width=0.98\linewidth]{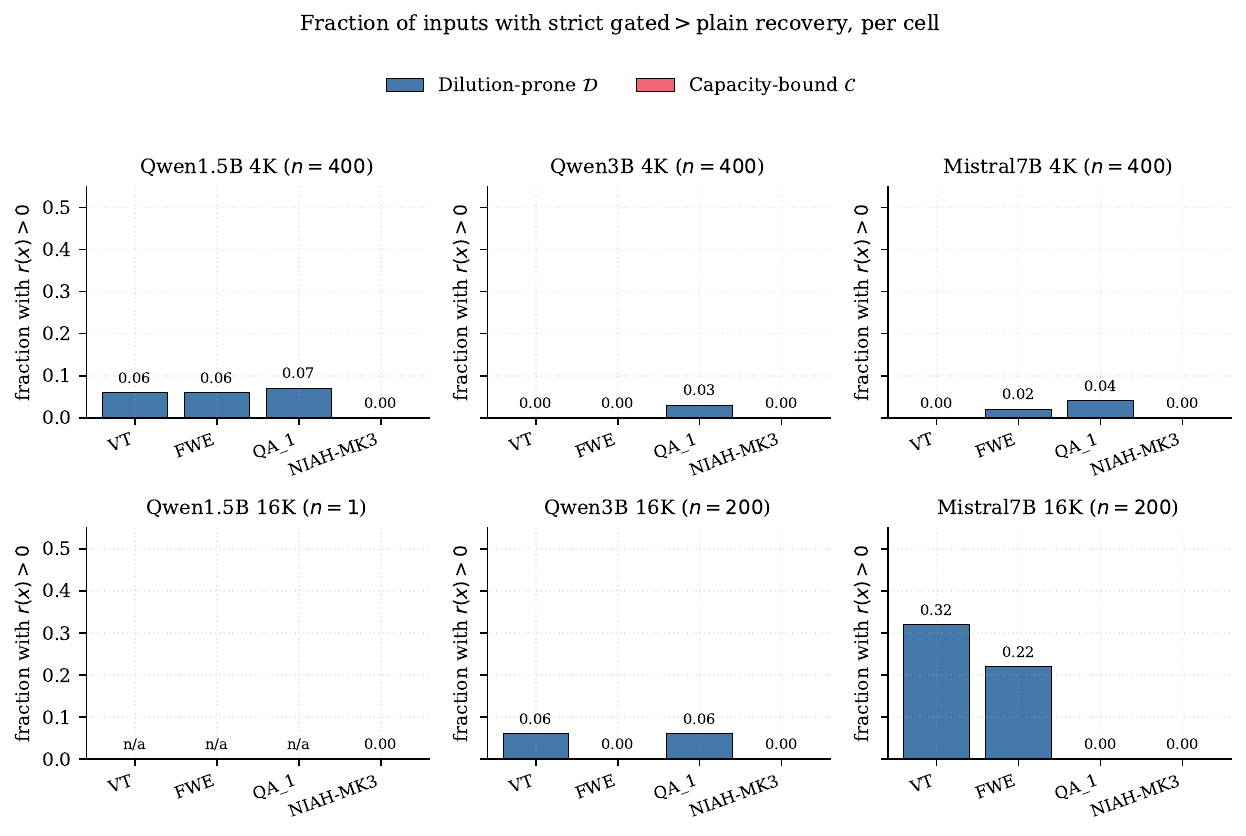}
\caption{Fraction of inputs with strict recovery $r(x) > 0$ per task,
faceted by $(\text{model}, \text{context})$ cell. Dilution-prone bars are
non-trivial in every cell with data, while NIAH-MK3 stays at or near zero,
visualising the partition input by input.}
\label{fig:reco-by-cell}
\end{figure}

\begin{table}[ht]
\centering
\caption{Fixed-budget harm rate $\Pr_x[A(b, x) < A_{\mathrm{plain}}(1, x)]$,
gated / plain at matched nominal budget, pooled over the seven SnapKV cells
(Qwen2.5-1.5B/3B/14B and Mistral-7B at 4K; Qwen2.5-1.5B/3B and Mistral-7B at
16K), $N = 500$ per (task, budget), $\tau = 0.07$. Unlike $r(x)$ above this
is not empty: the gate can and does harm inputs. On the capacity-bound task
it cuts harm by $29\times$ at $16\times$ compression; on QA\_1 it is exactly
inactive (identical rates, the gate opens on every input); on VT at $16\times$
both arms are harmed almost equally, so the gate does not help there.}
\label{tab:harm-rate}
\begin{tabular}{lcccc}
\toprule
Task & $b = 0.0625$ & $b = 0.125$ & $b = 0.25$ & $b = 0.5$ \\
\midrule
NIAH-MK3 & \textbf{0.026} / 0.750 & \textbf{0.026} / 0.744 & \textbf{0.024} / 0.724 & \textbf{0.018} / 0.546 \\
FWE      & \textbf{0.236} / 0.414 & \textbf{0.162} / 0.298 & \textbf{0.118} / 0.224 & \textbf{0.064} / 0.116 \\
QA\_1    & 0.110 / 0.110          & 0.064 / 0.064          & 0.030 / 0.030          & 0.010 / 0.010 \\
VT       & 0.684 / 0.694          & 0.108 / 0.110          & 0.034 / 0.036          & 0.026 / 0.026 \\
\bottomrule
\end{tabular}
\end{table}

\subsubsection{Additional limitations}

Two further limitations concern which baselines appear where. DBTrimKV's
pretrained gates exist only for Qwen3 variants, so
Section~\ref{subsec:head-to-heads} runs on Qwen3-4B-Instruct-2507 rather than
the Qwen2.5-keyed matrix, which is why DBTrimKV stays a standalone
head-to-head rather than a matrix row. The matrix itself covers only
training-free policies for the same reason: trained or learned evictors
(DBTrimKV, CapKV, IndexMem, DynamicKV) are cited but not included in it.

The last limitation concerns how confidently NIAH-MK3 can stand in for
the capacity-bound class generally. It is the only task where empirical
capacity-boundedness (accuracy collapse under eviction) and the gate's
capacity-bound classification coincide. On LongBench passage\_count at
Qwen2.5-14B the gate closes on all inputs (mean $D = -0.013$) yet
eviction is empirically \textit{safe} (accuracy flat under $16\times$
eviction, $\rho = 0$): a benign \textit{false positive} (over-conservative,
costing compression not accuracy), the dual of the MK2 false negative
(Appendix~\ref{app:distractor-detail}). The head-agreement drop is therefore a
one-sided predictor, reliable at firing on dilution-prone inputs, not a
perfect classifier. Multilingual and code tasks remain out of scope. We
expect the mechanism to transfer, since it is architecture-level, but we
have not measured it.

\subsection{Statistical Validation}
\label{app:statistical-validation}

\subsubsection{Per-input recovery distributions}

The mean $\rho$ hides the shape of the per-input recovery, so we also
report its distribution. For each input $x$ we form the recovery indicator
$r(x) = \mathrm{sign}\!\left(\max_b A_{\mathrm{gated}}(b, x) - A_{\mathrm{plain}}(1, x)\right)
\in \{-1, 0, +1\}$, where the maximum is over the budget sweep and
$A_{\mathrm{plain}}(1, x)$ is the full-cache plain accuracy. Figure~\ref{fig:reco-by-task}
pools all $(M, T)$ cells per task, and Figure~\ref{fig:reco-by-cell} reports the
fraction of strictly recovered inputs ($r(x) > 0$) within each $(M, T)$ cell.

\paragraph{What $r(x)$ can and cannot show} The budget sweep includes
$b = 1.0$, and at $b = 1.0$ the gated arm keeps the entire cache and
reproduces the plain full-cache outcome. We audit this on all $16$ matrix
cells plus the four $16$K cells: at $b = 1.0$, $n_{\text{kept}} = T$ and
$A_{\mathrm{gated}} = A_{\mathrm{plain}}$ on every one of the $5{,}120$
rows. Hence
$\max_b A_{\mathrm{gated}}(b, x) \ge A_{\mathrm{gated}}(1, x) = A_{\mathrm{plain}}(1, x)$
and $r(x) \ge 0$ is an \textit{identity}, not a measurement. The empty
$r = -1$ bin is therefore true by construction and carries no information
about whether gating ever harms an input. $r(x)$ should instead be read only as an
\textit{oracle-budget recovery} measure, that is, how often some budget in the
sweep beats the full cache. The question it cannot answer, whether a
\textit{deployed} gate at a \textit{fixed} budget ever loses an answer the full
cache would have got, is answered by the harm rate below.

\paragraph{Fixed-budget harm rate} Holding $b$ fixed makes ``worse than the
full cache'' a real case. Table~\ref{tab:harm-rate} reports
$\Pr_x[A(b, x) < A_{\mathrm{plain}}(1, x)]$ for the gated and plain arms at
matched $b$, pooled over the seven SnapKV $(\text{model}, \text{context})$
cells ($N = 500$ per task and budget).

The $16$K Qwen2.5-1.5B cell is \texttt{gated\_16k\_qwen15b\_sdpa.jsonl}, the
SDPA-prefill run reported throughout the paper. An earlier aborted eager run
left a nine-row NIAH-MK3-only stub under a similar name; pooling that stub
instead drops $50$ inputs per task from FWE, QA\_1 and VT and shifts the
column above to $0.029 / 0.796$, $0.291 / 0.444$, $0.118 / 0.118$ and
$0.751 / 0.756$. We state this because the two files are easy to confuse and
the released analysis script pins the correct one.

The histograms show that $\rho$ is not driven by a few outliers: across
the three dilution-prone tasks we recover a few percent to over ten
percent of inputs in many cells. NIAH-MK3 sits at $r{=}0$ on
essentially every input, confirming that the capacity-bound regime is a
genuine null rather than a noisy average. This input-level view matches
the partition diagnosed by the scaling formula and rules out the case
that gating helps only a small number of inputs by chance. It does not license
any claim that gating is harmless: Table~\ref{tab:harm-rate} shows that at
a fixed budget the gate harms $2.6\%$ of NIAH-MK3 inputs and $23.6\%$ of
FWE inputs at $16\times$ compression.

\subsubsection{Confidence intervals for the headline numbers}
\label{app:ci}

Tables~\ref{tab:matrix-ci} and~\ref{tab:other-ci} report $95\%$ intervals
for every headline delta, computed as described in the uncertainty
paragraph of Section~\ref{sec:experiments}. Full per-cell derivations ship
with the released logs.

\begin{table}[ht]
\centering
\caption{95\% confidence intervals for the $16$ gating deltas of
Table~\ref{tab:matrix}. Cluster bootstrap over inputs: we resample the
$(\text{task}, \text{id})$ pairs with replacement ($10{,}000$ replicates);
each input contributes its per-budget paired gated$-$plain differences,
preserving the pairing and the within-input correlation across budgets;
intervals are $2.5-97.5$ percentiles. All gated outcomes are the exact
$\tau = 0.07$ post-hoc reconstruction of Appendix~\ref{app:repro}.
$N = 400$ inputs per cell ($4$ tasks $\times$ $100$), $N = 200$ for
Qwen2.5-14B. \textit{Every one of the $16$ intervals excludes zero}; the
smallest lower bound is $+0.093$ (Qwen2.5-1.5B, SnapKV). These are $16$
simultaneous comparisons, so we also adjust for multiplicity.
Under Benjamini-Hochberg at $q = 0.05$, with per-cell $p$-values from the same
cluster bootstrap, \textit{all $16$ cells survive}, the least significant at
$p \le 5 \times 10^{-4}$ against a threshold of $0.050$. The adjustment does
not change any conclusion drawn from this table, which is why we report it.}
\label{tab:matrix-ci}
\begin{tabular}{llccc}
\toprule
Model & Base evictor & $\Delta$ & 95\% CI & $N$ \\
\midrule
Qwen2.5-1.5B & SnapKV       & $+0.121$ & $[+0.093,\,+0.151]$ & 400 \\
Qwen2.5-1.5B & H2O          & $+0.162$ & $[+0.126,\,+0.199]$ & 400 \\
Qwen2.5-1.5B & StreamingLLM & $+0.163$ & $[+0.128,\,+0.200]$ & 400 \\
Qwen2.5-1.5B & PyramidKV    & $+0.152$ & $[+0.118,\,+0.188]$ & 400 \\
\midrule
Qwen2.5-3B   & SnapKV       & $+0.259$ & $[+0.225,\,+0.295]$ & 400 \\
Qwen2.5-3B   & H2O          & $+0.303$ & $[+0.261,\,+0.347]$ & 400 \\
Qwen2.5-3B   & StreamingLLM & $+0.430$ & $[+0.383,\,+0.477]$ & 400 \\
Qwen2.5-3B   & PyramidKV    & $+0.329$ & $[+0.286,\,+0.372]$ & 400 \\
\midrule
Qwen2.5-14B  & SnapKV       & $+0.143$ & $[+0.108,\,+0.181]$ & 200 \\
Qwen2.5-14B  & H2O          & $+0.200$ & $[+0.152,\,+0.250]$ & 200 \\
Qwen2.5-14B  & StreamingLLM & $+0.250$ & $[+0.190,\,+0.310]$ & 200 \\
Qwen2.5-14B  & PyramidKV    & $+0.228$ & $[+0.175,\,+0.285]$ & 200 \\
\midrule
Mistral-7B   & SnapKV       & $+0.187$ & $[+0.154,\,+0.220]$ & 400 \\
Mistral-7B   & H2O          & $+0.236$ & $[+0.196,\,+0.278]$ & 400 \\
Mistral-7B   & StreamingLLM & $+0.263$ & $[+0.220,\,+0.305]$ & 400 \\
Mistral-7B   & PyramidKV    & $+0.231$ & $[+0.191,\,+0.272]$ & 400 \\
\bottomrule
\end{tabular}
\end{table}

\begin{table}[t]
\centering
\caption{95\% CIs for the remaining headline deltas. Cross-architecture rows
(Table~\ref{tab:crossarch}) use the four-budget grid and the same cluster
bootstrap; 16K rows are the second half of Table~\ref{tab:scaling}; the
DBTrimKV row is Table~\ref{tab:dbtrimkv} ($N = 120$ inputs, budgets
$\{128, 256, 512\}$). The Mistral NIAH-MK3 showcase entries are Wilson
score intervals on $N = 100$. The only headline delta whose interval
includes zero is Qwen2.5-3B at 16K, the gate-misfire cell that
Section~\ref{subsec:matrix-results} already reports as a shrunken effect.}
\label{tab:other-ci}
\begin{tabular}{lccc}
\toprule
Quantity & Point & 95\% CI & $N$ \\
\midrule
Yi-1.5-9B $\Delta$ (SnapKV, 4-budget)      & $+0.390$ & $[+0.338,\,+0.444]$ & 200 \\
Llama-3.1-8B $\Delta$ (SnapKV, 4-budget)   & $+0.109$ & $[+0.069,\,+0.154]$ & 200 \\
\midrule
Qwen2.5-1.5B 16K $\Delta$                  & $+0.032$ & $[+0.013,\,+0.053]$ & 200 \\
Qwen2.5-3B 16K $\Delta$                    & $+0.034$ & $[-0.002,\,+0.071]$ & 200 \\
Qwen2.5-14B 16K $\Delta$                   & $+0.071$ & $[+0.034,\,+0.107]$ & 120 \\
Mistral-7B 16K $\Delta$                    & $+0.116$ & $[+0.079,\,+0.156]$ & 200 \\
\midrule
DBTrimKV wrapper $\Delta$ (RULER 4K)       & $+0.233$ & $[+0.172,\,+0.297]$ & 120 \\
\midrule
Mistral 4K NIAH-MK3, gated acc.\ @ $b{=}0.0625$ & $0.89$ & $[0.814,\,0.937]$ & 100 \\
Mistral 4K NIAH-MK3, plain acc.\ @ $b{=}0.0625$ & $0.00$ & $[0.000,\,0.037]$ & 100 \\
\bottomrule
\end{tabular}
\end{table}

\subsubsection{Seed replicates}
\label{subsec:seed-variance}

The intervals above are cluster bootstraps over a fixed set of inputs, so they
quantify sampling within one draw and say nothing about the draw itself. We
therefore re-ran five SnapKV cells under three draws of the RULER split, varying
only the shuffle seed (Table~\ref{tab:seed-variance}). Decoding is greedy, so
there is no decoding stochasticity to average over.

Seed SD ranges from $0.004$ to $0.011$ against a grand-mean $\Delta$ of
$+0.229$, so draw-to-draw spread runs roughly an order of magnitude below the
effect it measures. The published SnapKV entries of Figure~\ref{fig:matrix} fall
inside the replicate range for Qwen2.5-1.5B, Qwen2.5-3B and Mistral-7B. The
Qwen2.5-14B entry, $+0.143$, sits $0.001$ above the largest replicate, which we
attribute to $N$: that cell is $N = 200$ in Table~\ref{tab:matrix} and
$N = 400$ here.
Restricted to NIAH-MK3, where nearly all of $\Delta$ lives, the spread widens to
$0.009 - 0.045$ on a much larger effect. Qwen2.5-1.5B is the loosest cell at
$0.045$ on a mean near $+0.45$, still an order of magnitude below its own
effect.

All fifteen runs used one host, and the replicates of a cell are never split
across machines. That matters because the reproduction gate
(Appendix~\ref{app:wrapper-provenance}) finds a small but real host effect on mean $D$, which
would otherwise appear in this table as seed variance.

\begin{table}[h]
\centering
\caption{Seed replicates of the SnapKV cells. Three draws of the RULER 4K split,
identical in every other respect: greedy decoding, $\tau = 0.07$, $\Delta$
averaged over budgets $b < 1.0$, $N = 400$ ($4$ tasks $\times$ $100$). The
lower block restricts to NIAH-MK3.}
\label{tab:seed-variance}
\small
\begin{tabular}{lccccccc}
\toprule
Model & seed 1 & seed 2 & seed 3 & mean & SD & range \\
\midrule
\multicolumn{7}{l}{\textit{All four tasks}} \\
Qwen2.5-1.5B & $+0.113$ & $+0.099$ & $+0.122$ & $+0.111$ & 0.011 & 0.022 \\
Qwen2.5-3B   & $+0.255$ & $+0.267$ & $+0.275$ & $+0.266$ & 0.010 & 0.019 \\
Qwen2.5-14B  & $+0.133$ & $+0.140$ & $+0.142$ & $+0.138$ & 0.005 & 0.010 \\
Mistral-7B   & $+0.187$ & $+0.194$ & $+0.183$ & $+0.188$ & 0.006 & 0.011 \\
Llama-3.1-8B & $+0.078$ & $+0.071$ & $+0.078$ & $+0.076$ & 0.004 & 0.007 \\
\midrule
\multicolumn{7}{l}{\textit{NIAH-MK3 only}} \\
Qwen2.5-1.5B & $+0.454$ & $+0.398$ & $+0.486$ & $+0.446$ & 0.045 & 0.089 \\
Qwen2.5-3B   & $+0.738$ & $+0.747$ & $+0.729$ & $+0.738$ & 0.009 & 0.019 \\
Qwen2.5-14B  & $+0.530$ & $+0.559$ & $+0.568$ & $+0.552$ & 0.020 & 0.037 \\
Mistral-7B   & $+0.709$ & $+0.710$ & $+0.669$ & $+0.696$ & 0.023 & 0.041 \\
Llama-3.1-8B & $+0.310$ & $+0.282$ & $+0.311$ & $+0.301$ & 0.016 & 0.029 \\
\bottomrule
\end{tabular}
\end{table}


\section{Theory: Mechanism, Partition, and Scaling Derivations}
\label{app:theory-derivations}

Section~\ref{sec:theory} states the mechanism, the partition it implies, and the scaling formula in summary form. This section provides the full derivations, followed by the sufficient-condition bridge connecting the mechanism to the observable predictor. At the end of each part, an explicit scope statement defines what is and is not established, so the claims can be read exactly. Notation follows \citet{DBtrimKV} $\alpha_{\mathcal{R}}$ the pre-eviction signal mass, $\gamma$ the noise-retention ratio, $\rho_{\mathcal{R}},\rho_{\mathcal{N}}$ the retained-mass fractions, $\delta_t$ the dilution, $D$ the head-agreement drop, $\sigma_z$ the pre-softmax logit-noise scale, $m$ the logit margin, $\kappa:=m/\sigma_z$, and $H,L,T$ the head, layer, and length counts.
Throughout, the $T$ key positions split as $T = R + N$ into a relevant set $\mathcal{R}$ ($|\mathcal{R}| = R$) and a noise set $\mathcal{N}$ ($|\mathcal{N}| = N$), with $R \ll N$. Two distinct noise scales appear and should not be conflated: $\sigma_z$ is the pre-softmax logit-noise scale (Appendix~\ref{app:sufficient}), and $\sigma_v$ is the per-distractor value-vector variance under isotropy (Appendix~\ref{app:proof-scaling}).

\paragraph{Mechanism and partition}
We adopt as stated the dilution mechanism from \citet{DBtrimKV}, which
suffices to motivate our gating predictor. Fix a frozen decoder language
model $M$. A single attention head reads from positions $t \in \{1, \dots, T\}$, each carrying
a value $v_t$ and a pre-eviction probability $\alpha_t$ with $\sum_t
\alpha_t = 1$. Partition positions into a relevant set $\mathcal{R} = \{t:
r_t = 1\}$ ($|\mathcal{R}| = R$) and a noise set $\mathcal{N}$
($|\mathcal{N}| = N$), with $R \ll N$. Let $U_t \subseteq \mathcal{R}$ denote the \textit{useful} set at
step $t$ (the positions whose retention preserves the next-token margin),
and define dilution as the attention mass \textit{outside} $U_t$:
\begin{equation}
\label{eq:dilution}
\delta_t := 1 - \sum_{i \in U_t} \alpha_{t,i}.
\end{equation}
Their Proposition~3.1 establishes that near-tie distractors
($\alpha_t$ comparable across $\mathcal{N}$) force $\delta_t$ to be large, and their Corollary~3.2 shows that preferential retention with $\rho_U \geq \rho_D$ (where $\rho_U, \rho_D$ are the retained mass on useful and
distractor positions) reduces $\delta_t$.

The partition follows from this mechanism. Define $\alpha_{\mathcal{R}} := \sum_{t \in \mathcal{R}} \alpha_t$ as the pre-eviction signal mass. Two regimes emerge:
\begin{itemize}
  \item \textit{Dilution-prone} ($\mathcal{D}$): $\alpha_{\mathcal{R}} \ll 1$. Most attention is on noise; eviction of noise positions monotonically improves the head SNR.
  \item \textit{Capacity-bound} ($\mathcal{C}$): $\alpha_{\mathcal{R}} \to 1$. Attention is concentrated on the few relevant positions. Imperfect scoring can only \textit{lose} signal; eviction strictly degrades the head.
\end{itemize}
The partition we observe on RULER is the empirical version of $\alpha_{\mathcal{R}}$ onto the task structure: NIAH-MK3 forces a single concentrated signal ($\alpha_{\mathcal{R}} \to 1$), while VT, FWE, QA, and niah\_multivalue all spread signal across many positions ($\alpha_{\mathcal{R}} \ll 1$).

\paragraph{Context-length scaling}
Fix the task and its relevant-set size $R$. Let $T = R + N$. Assume
distractor attention is approximately uniform on $\mathcal{N}$:
$\alpha_t \approx (1 - \alpha_{\mathcal{R}}) / N$ for $t \in \mathcal{N}$.
Let $\sigma_v^2$ denote the per-distractor value-vector variance under
isotropy (a different object from the pre-softmax logit noise scale
$\sigma_z$ used in App.~\ref{app:sufficient}). Then the head SNR before
eviction is
\[
\mathrm{SNR} = \frac{\alpha_{\mathcal{R}}^2}{\sigma_v^2 \sum_{t \in \mathcal{N}}
\alpha_t^2} \approx \frac{\alpha_{\mathcal{R}}^2 \cdot N}{\sigma_v^2 (1 -
\alpha_{\mathcal{R}})^2}.
\]
If $\alpha_{\mathcal{R}}$ decreases with $T$ (more distractors compete for
softmax mass), the pre-eviction SNR decreases and the noise level grows.

Let the kept set be $\mathcal{K}$ and define $\rho_{\mathcal{R}} :=
\sum_{t \in \mathcal{K} \cap \mathcal{R}} \alpha_t / \alpha_{\mathcal{R}}$
(the retained fraction of relevant mass) and analogously $\rho_{\mathcal{N}}$.
Under noise-retention ratio $\gamma \in (0, 1)$ (i.e.\ $|\mathcal{K} \cap
\mathcal{N}| = \gamma N$), the post-eviction SNR ratio
$\mathrm{SNR}_{\mathrm{post}} / \mathrm{SNR}_{\mathrm{pre}}$ is a
monotone-increasing function of the \textit{gain argument} (Eq.~\ref{eq:scaling-app}).
\begin{equation}
\label{eq:gain-arg}
g := \rho_{\mathcal{R}}^2 / \gamma
\end{equation}
The squaring survives the renormalisation step, and consequently, $g$ is not a function of $\rho_{\mathcal{R}} / \gamma$. 
The accuracy headroom $1 - A_{\mathrm{full}}(M, T)$ scales with $T$ on dilution-prone tasks.

The single-budget improvement event factorises into a headroom indicator $A:=\mathbf{1}[\hat A(\mathrm{SNR}_{\max})<1]$ and a recovery indicator $B:=\mathbf{1}[\,\exists b:\mathrm{SNR}_{\mathrm{post}}>\mathrm{SNR}_{\mathrm{pre}}]$. The recovery probability is $p_{\mathrm{recov}}(g):=\Pr_x[B]$, with $p_{\mathrm{recov}}(g)=0$ for $g\le1$ and non-decreasing above. Since the operational $b^*(x)$ is the argmax over the budget grid, and a calibrated (nested top-$k$) scorer ensures recovery budgets are not near-disjoint, the union-bound factor reduces to unity. The per-input expected accuracy improvement at a single budget $b$ thus obeys the exact scaling relation:
\begin{equation}
\label{eq:exact-scaling-bounded}
\rho_{\mathrm{KV}}
= (1-A_{\mathrm{full}})\,p_{\mathrm{recov}} + \mathrm{Cov}_x(A,B),
\qquad
(1-A_{\mathrm{full}})\,p_{\mathrm{recov}}
\;\le\;\rho_{\mathrm{KV}}\;\le\;
\min\{1-A_{\mathrm{full}},\,p_{\mathrm{recov}}\}.
\end{equation}
The lower endpoint uses $\mathrm{Cov}_x(A,B)\ge0$, which holds because both indicators are monotone in the single driver $\alpha_{\mathcal{R}}$ (Chebyshev's association inequality~\citep{Harris_1960}); the upper endpoint is Fr\'echet--Hoeffding. Eq.~\ref{eq:scaling} (main text) predicts three observed effects in one formula: context-length amplification on dilution-prone tasks (Qwen2.5-1.5B VT goes $\rho = 0.06 \to 0.19$ as $T = 4\mathrm{K} \to 16\mathrm{K}$, in lockstep with $1 - A_{\mathrm{full}} = 0.18 \to 0.75$); larger-model saturation at short context (Qwen2.5-3B 4K VT has $1 - A_{\mathrm{full}} \approx 0.00$, so $\rho \approx 0$); and re-entry of larger models into the dilution regime at long context (Qwen2.5-3B 16K FWE has $1 - A_{\mathrm{full}} = 0.62$ and $\rho = 0.32$). Empirically, the ratios $\rho / [(1-A_{\mathrm{full}}) \cdot p_{\mathrm{recov}}]$ cluster around $0.25 - 0.55$ on dilution-prone tasks.

\textit{Scope.} This derivation establishes the identity and two-sided bound with constant $1$ \textit{inside} the deterministic-SNR surrogate. It does not establish the surrogate itself (accuracy as a fixed function of head SNR is assumed); the covariance is pinned only to a sign, not a value. Where headroom is driven by $\alpha_{\mathcal{R}}$-independent factors (the QA parametric-knowledge cells) both bounds pinch $\rho\!\to\!0$ despite headroom, matching the observed QA exceptions as a sign prediction, not a proof about those tasks.

\paragraph{Connecting the predictor to the mechanism: a sufficient
condition for non-trivial dilution}
\label{subsec:conjecture}
The mechanism uses the unobservable $\alpha_{\mathcal{R}}$. Our gating predictor uses the head-agreement drop $D$ (Eq.~\ref{eq:D}), which is observable. We sketch a proof connecting the two in Appendix~\ref{app:sufficient}, and
leave the formal version with explicit, tight constants to future work.
Under noise isotropy, sub-Gaussian pre-softmax noise of scale
$\sigma_z$, a bounded logit margin between $\mathcal{R}$ and
$\mathcal{N}$, and good scoring with noise-retention ratio $\gamma \in
(0, 1)$, there exist a model-specific constant $c_1$ and a
$\gamma$-dependent constant $c_2(\gamma)$, with an additional $O(H)$
multiplicative slack absorbed in both, such that
\begin{equation}
\label{eq:bridge}
D \;\ge\; c_1 \cdot \log(1/\gamma)
\;\Longrightarrow\;
\mathbb{E}[\delta_t] \;\ge\; c_2(\gamma) \cdot (1 - A_{\mathrm{full}}),
\end{equation}
with probability $\geq 1 - \beta$ over the pre-softmax noise, where
$\delta_t$ is the dilution term of Eq.~\ref{eq:dilution}. The converse (capacity-bound $\Rightarrow$ small
$D$) holds under the same surrogate given a $\kappa$-separation assumption
$m\ge2\sigma_z\sqrt{\log(2H^2LT^2/\beta)}$: then, by the top-$k$ concentration
result stated and proved as Lemma~\ref{lem:topk} in
Section~\ref{app:sufficient}, late-layer heads share a common
top-$k$ set, $a_{\mathrm{late}}\!\to\!1$, and $\Pr[D\ge\tau]\le\beta=2H^2LT^2\,
e^{-\kappa^2/4}$, so observing $D\ge\tau$ certifies the input is not
capacity-bound except on a noise tail.

\textit{Measured margin (why the bridge does not certify the gate).} The converse is
only informative when $\beta<1$, which needs $\kappa\ge10.11$ at
$(H,L,T)=(12,28,{\approx}4\mathrm{K})$. We measured $\kappa$ directly on
Qwen2.5-1.5B RULER 4K NIAH-MK3: median $\kappa=3.33$, about $3\times$ too small.
The gap is robust: sweeping the top-$k$ split over $k\in\{8,16,32,64\}$ gives
$\kappa=4.01,3.68,3.33,2.99$ (monotone, none rescues the bound); a second
architecture gives $3.71$ (Mistral-7B); and, decisively, a dilution-prone task
gives $\kappa=3.40$ (VT), close to the capacity-bound $3.33$ and so
the margin itself is not what distinguishes the classes (the gate separates them through $D$, not $\kappa$). 
A factor of three in $\kappa$ is a factor $>10^9$ in $\beta$. The margin condition is a union bound over $T^2 \cdot HL$ failure events and is correspondingly loose, so its failure leaves the top-$k$ concentration \textit{uncertified rather than false}: the bound is therefore
\textit{uninformative at the paper's own parameters}, and the class separation the gate
relies on is empirical (Table~\ref{tab:drops}), not certified by the surrogate.

\textit{Sign vs.\ ordinal.} As stated the converse predicts $D\le0$ on
capacity-bound inputs, which Table~\ref{tab:drops} contradicts in five of seven
cells (measured $D$ is mildly positive, $+0.02$ to $+0.07$). Relaxing the
common-set idealisation to an $O(1)$ needle neighbourhood weakens the conclusion
to $\Pr[D\le o(1)]\ge1-\beta-o(1)$, which the data satisfy but which no longer
distinguishes the classes by sign. We therefore rely on the ordering based separation
only.

\textit{Scope.} Establishes the sufficient direction and, under the named
$\kappa$-separation and common-set assumptions, the converse and its
$e^{-\kappa^2/4}$ rate; also (App.~H content) an $O(H^2\varepsilon)$ approximate
identification of $U_t$ with $\mathcal{I}^{(\ell^*)}$ under $\varepsilon$-close
heads plus a consensus--usefulness assumption, which degrades exactly in the
single-decisive-head (retrieval-head, \citealp{wu2025retrieval}) regime where the
paper already declines to apply the bridge. Does not establish that real
attention satisfies the margin assumption (measured false, $\kappa\!\approx\!3$
vs.\ $\ge10.1$ needed), that capacity-bound heads literally share one common set,
that the $O(H)$ pairwise-to-$H$-wise slack or the heuristic $\log(1/\gamma)$
factor are tight, or a stand-alone necessity proof. The dilution mechanism motivates the bridge, not the reverse. The bridge is the sufficient-condition span from the mechanism to the observable $D$, and it does not reach that observable at the measured margin. Neither the bridge nor its converse carries the empirical claims, which rest on the measured partition (Table~\ref{tab:drops}) and the direct gating results. We retain this analysis as motivation for why an early-minus-late agreement difference is a reasonable signal, while noting that the signal's validity is established empirically, not by this bridge.

\section{Proof sketch of the scaling proposition}
\label{app:proof-scaling}

We expand the sketch from Section~\ref{sec:theory}.

\paragraph{Setup} Let $T = R + N$ with $R$ relevant and $N$ noise positions. Assume noise isotropy: $\alpha_t \approx (1 - \alpha_{\mathcal{R}}) / N$ for $t \in \mathcal{N}$, with $\alpha_{\mathcal{R}} := \sum_{t \in \mathcal{R}} \alpha_t$, and let $\sigma_v^2$ be the per-distractor value-vector variance under isotropy (distinct from the pre-softmax logit noise scale $\sigma_z$ of App.~\ref{app:sufficient}). The pre-eviction head SNR is
%
\[
\mathrm{SNR}_{\mathrm{pre}} = \frac{\alpha_{\mathcal{R}}^2}
{\sigma_v^2 \sum_{t \in \mathcal{N}} \alpha_t^2}
\approx \frac{\alpha_{\mathcal{R}}^2 N}{\sigma_v^2 (1 - \alpha_{\mathcal{R}})^2}.
\]
Eviction with noise-retention ratio $\gamma \in (0, 1)$ keeps a fraction
$\rho_{\mathcal{R}}$ of the relevant mass and a fraction $\rho_{\mathcal{N}}$
of the noise mass. After softmax renormalisation over the kept set, the
post-eviction relevant mass is
$\widetilde{\alpha}_{\mathcal{R}} = \rho_{\mathcal{R}} \alpha_{\mathcal{R}} / Z$
with $Z = \rho_{\mathcal{R}} \alpha_{\mathcal{R}} + \rho_{\mathcal{N}}
(1 - \alpha_{\mathcal{R}})$, and the retained noise positions number
$\gamma N$ with per-position mass $\rho_{\mathcal{N}}(1-\alpha_{\mathcal{R}})/(Z
\gamma N)$. Substituting,
\begin{equation}
\label{eq:scaling-app}
\mathrm{SNR}_{\mathrm{post}}
= \frac{\widetilde{\alpha}_{\mathcal{R}}^2}
{\sigma_v^2 \sum_{t \in \mathcal{K} \cap \mathcal{N}} \widetilde{\alpha}_t^2}
\approx \frac{\rho_{\mathcal{R}}^2 \alpha_{\mathcal{R}}^2 \, \gamma N}
{\sigma_v^2 \, \rho_{\mathcal{N}}^2 (1 - \alpha_{\mathcal{R}})^2}
= \frac{\rho_{\mathcal{R}}^2}{\rho_{\mathcal{N}}^2 \, \gamma^{-1}} \cdot
\mathrm{SNR}_{\mathrm{pre}}.
\end{equation}
Under good scoring ($\rho_{\mathcal{R}} \ge \rho_{\mathcal{N}}$ with
$\rho_{\mathcal{N}} \approx \gamma$ on a calibrated scorer), the gain factor
reduces to $\rho_{\mathcal{R}}^2 / \gamma$, so the post-eviction SNR strictly
exceeds the pre-eviction SNR. The renormalisation $Z$ is the step that turns
the raw fractions $(\rho_{\mathcal{R}}, \rho_{\mathcal{N}})$ into a usable
SNR ratio. Without it the argument would collapse to $\rho_{\mathcal{R}}^2/\gamma$
absent the squaring.

\paragraph{Factorisation} We work inside the surrogate, where the
true accuracy indicator $\mathbf{1}[A(b, x) > A(b_{\max}, x)]$ is replaced
by the surrogate indicator
$\mathbf{1}[\hat A(\mathrm{SNR}(b, x)) > \hat A(\mathrm{SNR}(b_{\max}, x))]$
with $\hat A: [0, \infty) \to [0, 1]$ strictly increasing and saturating at
$1$. This replacement is the surrogate's content: it asserts that accuracy
is a deterministic function of head SNR. Within the surrogate the
improvement event factorises as
\[
\mathbf{1}[\hat A(\mathrm{SNR}_{\mathrm{post}}) > \hat A(\mathrm{SNR}(b_{\max}, x))]
= \mathbf{1}[\hat A(\mathrm{SNR}(b_{\max}, x)) < 1] \cdot
\mathbf{1}[\hat A(\mathrm{SNR}_{\mathrm{post}}) > \hat A(\mathrm{SNR}_{\mathrm{pre}})].
\]
Both monotonicity and saturation are needed: if $\hat A$ were not strictly
increasing the second indicator could fire without an SNR gain, and if
$\hat A$ did not saturate at $1$ the case $\hat A(\mathrm{SNR}(b_{\max}, x)) = 1$
would not zero out the left-hand side. The first indicator is the headroom
in the surrogate (a population-level proxy for the true full-cache failure
probability $1 - A_{\mathrm{full}}$). The second depends only on
$\rho_{\mathcal{R}}^2/\gamma$ via Eq.~\ref{eq:scaling-app} once
$\rho_{\mathcal{N}} \approx \gamma$ is folded in. The transfer from
surrogate event to operational accuracy is mediated by the
surrogate-fits-data assumption: we treat $\hat A$ as a good fit to the
expected accuracy conditional on SNR, and deviations are absorbed into the
empirical constant, fitted at $0.34$. Marginalising over the prompt
distribution and treating the two indicators as independent gives
\[
\rho_{\mathrm{KV}}(\mathsf{T}, M, T) \approx
\bigl(1 - A_{\mathrm{full}}(M, T)\bigr) \cdot
p_{\mathrm{recov}}(\rho_{\mathcal{R}}^2 / \gamma),
\]
which is Eq.~\ref{eq:scaling}. The boundary conditions on $p_{\mathrm{recov}}$
follow from the surrogate, and both are statements about the gain argument
$g = \rho_{\mathcal{R}}^2/\gamma$ of Eq.~\ref{eq:gain-arg} rather than about
$\rho_{\mathcal{R}}$ and $\gamma$ separately. At $\rho_{\mathcal{R}} = 0$ no
relevant mass survives, so $\mathrm{SNR}_{\mathrm{post}} = 0$, recovery is
impossible and $g = 0$. At $\gamma = 1$ (no eviction)
$\mathrm{SNR}_{\mathrm{post}} = \mathrm{SNR}_{\mathrm{pre}}$ identically, the
second indicator never fires, and $g = \rho_{\mathcal{R}}^2 \le 1$. More
generally the second indicator requires $\mathrm{SNR}_{\mathrm{post}} >
\mathrm{SNR}_{\mathrm{pre}}$, which by Eq.~\ref{eq:scaling-app} is exactly
$g > 1$. Hence $p_{\mathrm{recov}}(g) = 0$ on the whole dead zone $g \le 1$,
and is non-decreasing above it. Stating the two endpoints separately as
$p_{\mathrm{recov}}(0) = p_{\mathrm{recov}}(1) = 0$ with monotonicity on
$[0, 1]$ would force $p_{\mathrm{recov}} \equiv 0$ there, which is the dead
zone rather than a prediction of no recovery anywhere.

\paragraph{From surrogate to operational $\rho_{\mathrm{KV}}$} The
factorisation argues about the improvement event at a single budget $b$,
whereas the operational $\rho_{\mathrm{KV}}$ (Eq.~\ref{eq:rho}) is the
per-input strict-Pareto frequency over the best budget $b^{*}(x)$. The
two are related by a union bound over the discrete budget grid $\mathcal{B}
= \{b_1, \dots, b_K\}$:
\begin{align*}
\Pr_x\bigl[A(b^{*}(x), x) > A(b_{\max}, x)\bigr]
&= \Pr_x\bigl[\exists\, b \in \mathcal{B} : A(b, x) > A(b_{\max}, x)\bigr] \\
&\le \sum_{b \in \mathcal{B}} \Pr_x\bigl[A(b, x) > A(b_{\max}, x)\bigr].
\end{align*}
Under monotone recovery ($p_{\mathrm{recov}}$ non-decreasing in
$\rho_{\mathcal{R}}^2/\gamma$, with $\rho_{\mathcal{R}}$ determined by the
scorer's top-$k$ at each $b$), the right-hand side is bounded by
$K \cdot \Pr_x[A(\widetilde b, x) > A(b_{\max}, x)]$ for the budget
$\widetilde b$ that maximises the per-budget improvement probability.
Under the independence approximation of the previous paragraph, this
maximiser term equals Eq.~\ref{eq:scaling}, so
\[
\rho_{\mathrm{KV}}(\mathsf{T}, M, T)
\;\le\; K \cdot \bigl(1 - A_{\mathrm{full}}(M, T)\bigr)
\cdot p_{\mathrm{recov}}\bigl(\rho_{\mathcal{R}}^2 / \gamma\bigr).
\]
This is a one-sided upper bound at the cost of a multiplicative grid factor
$K$. \textbf{The factor is an artifact of the union step and is not needed.}
Appendix~\ref{app:sharpened} shows that on any budget grid whose improvement
events are nested, which is what a calibrated scorer produces, the correct
leading constant is exactly $1$, and characterises the nested-versus-disjoint
boundary at which a real factor $K$ could appear. We therefore state the
bound with constant $1$ and keep the $K$ form only to show where the loose
version comes from.

\paragraph{Independence bound via covariance} Let $A$ and $B$ denote the
headroom and recovery indicators viewed as Bernoulli random variables under
the prompt-distribution measure. The exact decomposition is
$\Pr[A \cdot B] = \Pr[A]\,\Pr[B] + \mathrm{Cov}(A, B)$, so the independence
step incurs only the covariance residual. Both indicators are monotonically
decreasing in $\alpha_{\mathcal{R}}$: larger $\alpha_{\mathcal{R}}$ raises
$\mathrm{SNR}_{\mathrm{pre}}$, pushing $A_{\mathrm{full}}$ toward $1$ (less
headroom, so $A$ trends to $0$) and simultaneously saturating
$\mathrm{SNR}_{\mathrm{post}}$ earlier under (A1) noise-isotropy (less
recovery margin, so $B$ trends to $0$). Coupling $A$ and $B$ through the
single scalar $\alpha_{\mathcal{R}}$ makes them comonotone in this driver, so
Chebyshev's association (sum) inequality~\citep{Harris_1960} gives
$\mathrm{Cov}(A, B) \ge 0$, i.e.\ $\Pr[A \cdot B] \ge \Pr[A]\,\Pr[B]$. The
lattice-theoretic FKG inequality is not needed here and does not apply
directly: both indicators are monotone functions of one scalar driver, which
is the elementary Chebyshev/Harris setting. Lemma~\ref{lem:assoc} in
App.~\ref{app:sharpened} proves the same step, including the case where $A$
carries an auxiliary driver independent of $\alpha_{\mathcal{R}}$. The Cauchy--Schwarz / Bernoulli-variance
bound $|\mathrm{Cov}(A, B)| \le \sqrt{\mathrm{Var}(A)\,\mathrm{Var}(B)} \le 1/4$
then yields an absolute, distribution-free residual. Combining with the
union-bound paragraph above,
\[
\rho_{\mathrm{KV}}(\mathsf{T}, M, T)
\;\le\; K \cdot \bigl(1 - A_{\mathrm{full}}(M, T)\bigr)
\cdot p_{\mathrm{recov}}\bigl(\rho_{\mathcal{R}}^2 / \gamma\bigr)
\;+\; K/4.
\]
The $K/4$ term is rigorous but uninformative at our $K \le 7$, since a bound of
$1.75$ on a probability says nothing. Both $K$ factors disappear under the
sharpened treatment of Appendix~\ref{app:sharpened}, which replaces this
paragraph's conclusion with the exact identity
$\rho_{\mathrm{KV}} = (1 - A_{\mathrm{full}})\,p_{\mathrm{recov}}
+ \mathrm{Cov}(A, B)$ and a covariance term that is signed rather than
merely bounded. The derivation above is included because it isolates the step
that introduces the slack, though every empirical claim in the paper rests on the
sharpened version.

\section{Proof sketch: sufficient condition for the partition}
\label{app:sufficient}

We sketch a proof of the bridge inequality Eq.~\ref{eq:bridge}: under the
simplified attention model below, a large head-agreement drop $D$
implies a non-trivial lower bound on Bui's dilution
$\delta_t$~\citep{DBtrimKV}. The argument is structured as five
numbered steps followed by a converse heuristic, an empirical-validation
paragraph, and an explicit list of what is \textit{not} proved. We give
closed forms for the constants $c_1$ and $c_2(\gamma)$ rather than
leave them implicit. The constants carry an $O(H)$ multiplicative slack
from the pairwise-to-$H$-wise Jaccard step in Step~3, and tightening the
constants is left to future work.


\paragraph{Per-head logit model}
 With $\mathcal{R}, \mathcal{N}$ the relevant and noise sets as above, for layer
 $\ell \in \{1, \dots, L\}$ and head $h \in \{1, \dots, H\}$ the pre-softmax
 logits are $z^{(\ell)}_{h, i} = (q^{(\ell)}_h)^\top k_i  \eta^{(\ell)}_{h, i}$,
 where $\eta^{(\ell)}_{h, i}$ is zero-mean $\sigma_z$-sub-Gaussian noise
 independent across $(\ell, h, i)$. The subscript $z$ distinguishes the pre-softmax logit noise scale used here
from the value-vector variance $\sigma_v$ used in
Appendix~\ref{app:proof-scaling}. We write $\alpha^{(\ell)}_{h, i}$ for
 the post-softmax attention. Bui's dilution $\delta_t = 1 - \sum_{i \in
 U_t} \alpha_{t, i}$ is taken at the decoding query at step $t$. {We drop the $t$ subscript and write $\delta := \delta_t$.}

\paragraph{Step 1 (the surrogate model)}
We use the multi-head lift of the single-head surrogate from
Appendix~\ref{app:proof-scaling}:
\begin{enumerate}
  \item[\textbf{(A1)}] \textit{Signal coherence.} For each layer $\ell$ and head
  $h$, there is a head-private relevant subset $\mathcal{R}^{(\ell)}_h
  \subseteq \mathcal{R}$ such that for $i \in \mathcal{R}^{(\ell)}_h$ the
  deterministic part of the logit satisfies
  $(q^{(\ell)}_h)^\top k_i \geq m_+$, and for $i \in \mathcal{N}$,
  $(q^{(\ell)}_h)^\top k_i \leq m_-$, with $m_+ - m_- \geq m > 0$.
  This $m$ is the \textit{logit margin}; it is bounded above by $\log T$ in
  the well-trained regime.
  \item[\textbf{(A2)}] \textit{Sub-Gaussian noise.}
  $\eta^{(\ell)}_{h, i}$ is zero-mean and $\sigma_z^2$-sub-Gaussian,
  independent across $(\ell, h, i)$. This is the standard noise model
  used in attention-as-feature-aggregation analyses.
  \item[\textbf{(A3)}] \textit{Capacity-bound vs dilution-prone regimes.}
  In the capacity-bound regime $\mathcal{C}$, all heads at all layers share a
  common relevant set: $\mathcal{R}^{(\ell)}_h = \mathcal{R}$ for every
  $(\ell, h)$. In the dilution-prone regime $\mathcal{D}$, late-layer head
  subsets diverge: there exists a layer index $\ell^* > L/2$ such that for
  any head pair $(h, h')$ at layer $\ell^*$,
  $|\mathcal{R}^{(\ell^*)}_h \cap \mathcal{R}^{(\ell^*)}_{h'}|
  / |\mathcal{R}^{(\ell^*)}_h \cup \mathcal{R}^{(\ell^*)}_{h'}|
  \leq J_{\text{late}}$, while early-layer heads share their subsets,
  $|\mathcal{R}^{(1)}_h \cap \mathcal{R}^{(1)}_{h'}| /
  |\mathcal{R}^{(1)}_h \cup \mathcal{R}^{(1)}_{h'}| \geq J_{\text{early}}$,
  with $J_{\text{early}} > J_{\text{late}}$. Both $J_{\text{early}}$
  and $J_{\text{late}}$ are model-specific population constants of the
  surrogate, not functions of $L$. This layer structure is architecture-dependent, not a universal claim: on Qwen2.5-1.5B the late-plateau ordering holds (capacity-bound NIAH-MK3 sits at $0.283$, above
  dilution-prone VT at $0.227$), whereas on Llama-3.1-8B the per-layer profile
  is non-monotone and the separation between regimes collapses (Appendix~\ref{subsec:layer-profile},
  Figure~\ref{fig:layer-profile}). The results downstream of (A3) are stated
  within this surrogate and are not claimed to hold where the measured profile departs from it.
  \item[\textbf{(A4)}] \textit{Identification of $U_t$.} We identify Bui's
  useful set $U_t$ (a decoding-step quantity defined at the
  residual-stream level) with the head-common signal set
  $\mathcal{I}^{(\ell^*)} := \bigcap_h \mathcal{R}^{(\ell^*)}_h$ at the
  divergence layer $\ell^*$ (a single-layer attention quantity). The
  identification is a modelling choice (positions every head agrees are
  signal). It collapses time-step and layer-index axes, and within the surrogate the two are assumed to coincide rather than shown to. We flag it in the ``What is not
  proved'' list below.
\end{enumerate}

The empirical content of (A3) is the layer-dependence we measure: early
layers respond to local context (similar $\mathcal{R}^{(\ell)}_h$ across
heads, regardless of task), late layers respond to task structure (head
divergence on dilution-prone tasks, head convergence on capacity-bound
tasks). See \S\ref{subsec:algorithm} for the corresponding operational
definitions.

\paragraph{Step 2 (head-agreement drop as Jaccard divergence)}
The head-agreement drop $D$ in Eq.~\ref{eq:D} is the early-late
difference of mean pairwise top-$k$ Jaccard agreement, with $k$ chosen
to be commensurate with $R$ (in the experiments $k = 32 \sim |\mathcal{R}|$).
In the high-SNR limit of the surrogate (Lemma~\ref{lem:topk} below),
each head's top-$k$ set concentrates on its own private
$\mathcal{R}^{(\ell)}_h$ with high probability. Hence at layer $\ell$,
\begin{equation}
a_\ell \;\approx\; \frac{1}{\binom{H}{2}} \sum_{h < h'}
\frac{|\mathcal{R}^{(\ell)}_h \cap \mathcal{R}^{(\ell)}_{h'}|}
{|\mathcal{R}^{(\ell)}_h \cup \mathcal{R}^{(\ell)}_{h'}|}.
\end{equation}
Binning into early third and late third and subtracting,
\begin{equation}
\label{eq:D-jaccard}
D \;=\; a_{\text{early}} - a_{\text{late}}
\;\approx\;
\overline{J}_{\text{early}} - \overline{J}_{\text{late}},
\end{equation}
the difference of mean pairwise Jaccard between layers in the early and
late bins. $D$ is therefore a direct measurement of the head-set
divergence in (A3).

\begin{lemma}[Top-$k$ concentration on the private set]
\label{lem:topk}
Under (A1), (A2) with $m \geq \sigma_z \sqrt{4 \log(2 H^2 L T^2 / \beta)}$,
the top-$k$ set of head $(\ell, h)$ equals $\mathcal{R}^{(\ell)}_h$
jointly across all $HL$ heads with probability $\geq 1 - \beta$, when
$k = |\mathcal{R}^{(\ell)}_h|$.
\end{lemma}

\textit{Proof.} A standard sub-Gaussian maximal inequality: for any pair
$(i, j)$ with $i \in \mathcal{R}^{(\ell)}_h$ and $j \in \mathcal{N}$,
$\Pr[z_{h, i}^{(\ell)} < z_{h, j}^{(\ell)}] \leq \exp(-m^2 /
(4 \sigma_z^2))$. There are $R \cdot N \leq T^2 / 4$ such relevant-vs-noise
pairs at each layer-head. Substituting the threshold $m \geq
\sigma_z \sqrt{4 \log(2 H^2 L T^2 / \beta)}$ gives a per-pair tail of
$\beta / (2 H^2 L T^2)$. Union-bounding over up to $T^2$ pairs per
layer-head and over all $HL$ layer-heads gives a total tail of $HL \cdot
T^2 \cdot \beta / (2 H^2 L T^2) = \beta / (2 H) \leq \beta$.
\hfill $\square$

The threshold $m \geq \sigma_z \sqrt{4 \log(2 H^2 L T^2 / \beta)}$ is the
\textit{margin condition}. If it holds, every step below holds with
probability $\geq 1 - \beta$ over the noise.

\paragraph{Step 3 (head-set divergence forces union expansion)}
For each layer $\ell$ define the union of head-private sets
$\mathcal{U}^{(\ell)} := \bigcup_{h=1}^{H} \mathcal{R}^{(\ell)}_h$ and
the intersection $\mathcal{I}^{(\ell)} := \bigcap_{h=1}^{H}
\mathcal{R}^{(\ell)}_h$. The mean pairwise Jaccard $\overline{J}_\ell$
at layer $\ell$ is a \textit{pairwise} object, whereas
$|\mathcal{U}^{(\ell)}|$ and $|\mathcal{I}^{(\ell)}|$ are $H$-wise
objects, so the ratio $|\mathcal{U}|/|\mathcal{I}|$ is not equal to $1 -
\overline{J}$ in general. Inclusion-exclusion only gives
$|\mathcal{U}|/|\mathcal{I}| \leq 1 + (H-1)(1 - \overline{J})$ in the
best case, and the union-vs-intersection gap can scale with $H$ in the
worst case. We therefore avoid an $H$-wise identity and use the weaker
pairwise inequality. Under (A3), $\overline{J}_{\ell^*} \leq
J_{\text{late}}$, so there exists a head pair $(h, h')$ at layer
$\ell^*$ with $|\mathcal{R}^{(\ell^*)}_h \cap
\mathcal{R}^{(\ell^*)}_{h'}| \leq J_{\text{late}} \cdot
|\mathcal{R}^{(\ell^*)}_h \cup \mathcal{R}^{(\ell^*)}_{h'}|$. Picking
the head $h^\dagger$ with the largest private residual set
$\mathcal{R}^{(\ell^*)}_{h^\dagger} \setminus \mathcal{I}^{(\ell^*)}$,
\begin{equation}
\label{eq:union-bound}
|\mathcal{U}^{(\ell^*)} \setminus \mathcal{I}^{(\ell^*)}|
\;\geq\;
|\mathcal{R}^{(\ell^*)}_{h^\dagger} \setminus \mathcal{I}^{(\ell^*)}|
\;\geq\;
(1 - J_{\text{late}}) \cdot \max_h |\mathcal{R}^{(\ell^*)}_h|.
\end{equation}
This is strictly weaker than treating $(|\mathcal{U}| -
|\mathcal{I}|)/|\mathcal{U}|$ as if it were equal to $1 -
\overline{J}$, but it is correct, and the slack is absorbed in $c_1$ below.

Under (A4), Bui's useful set is $U_t = \mathcal{I}^{(\ell^*)}$, so
positions in $\mathcal{U}^{(\ell^*)} \setminus \mathcal{I}^{(\ell^*)}$
are attended by some head but not by all, and from the perspective of
the next layer's query they contribute attention mass outside $U_t$.

\paragraph{Step 4 (bridge to Bui's $\delta_t$)}
Let $\bar\alpha$ denote the average per-head, per-position attention on a
position in $\mathcal{R}^{(\ell^*)}_h$ at layer $\ell^*$. Under (A1)
with margin $m$, the softmax concentrates and $\bar\alpha \geq
(1 - (N/R)e^{-m})/R$ on private positions, tight when $e^m \gg N/R$ (an
earlier draft omitted the $N/R$ factor, which is not negligible in the
paper's own $R \ll N$ regime). Treating per-head value
contributions as additive after the output projection (a modelling
choice we flag below), the attention mass \textit{outside}
$\mathcal{I}^{(\ell^*)}$ contributed by the maximal-residual head is at
least
\begin{equation}
\sum_{i \in \mathcal{R}^{(\ell^*)}_{h^\dagger} \setminus
\mathcal{I}^{(\ell^*)}} \alpha^{(\ell^*)}_{h^\dagger, i}
\;\geq\;
\bar\alpha \cdot (1 - J_{\text{late}}) \cdot \max_h
|\mathcal{R}^{(\ell^*)}_h|,
\end{equation}
using Eq.~\ref{eq:union-bound}. Under eviction with noise-retention
ratio $\gamma$ and a top-$k$ scorer that does not a priori distinguish
signal-side residuals from noise positions (an additional modelling
assumption beyond (A1)-(A4) which we make explicit here), positions in
$\mathcal{U}^{(\ell^*)} \setminus \mathcal{I}^{(\ell^*)}$ inherit the
same effective per-position retention rate $\gamma$ as the noise set.
Those retained for only a subset of heads appear as distractors to the
remaining heads after the cache is merged. Conditional on this
retention,
\begin{equation}
\label{eq:Edelta-lb}
\mathbb{E}[\delta_t]
\;\geq\;
\bar\alpha \cdot (1 - J_{\text{late}}) \cdot
\max_h |\mathcal{R}^{(\ell^*)}_h|
\cdot (1 - \gamma),
\end{equation}
because the $(1 - \gamma)$ fraction of un-retained positions in the
residual private sets no longer balance the heads that did keep them,
leaving a residual distractor mass proportional to $(1 - \gamma)$. This
is a new bound derived in our surrogate. It is structurally analogous
to Bui's Corollary~3.2 (which relates pre- and post-eviction dilution
via $\delta' = (\rho_D/\rho_U)\delta / [(1 - \delta) +
(\rho_D/\rho_U)\delta]$), but the input here is the
head-disagreement-induced distractor set rather than a pre-existing
dilution.

\paragraph{Step 5 (closing the bridge: solve for $D$)}
Eq.~\ref{eq:Edelta-lb} uses $(1 - \overline{J}_{\text{late}})$, which we
relate to the observable $D$ via Eq.~\ref{eq:D-jaccard}: $D =
\overline{J}_{\text{early}} - \overline{J}_{\text{late}}$. Since
$\overline{J}_{\text{early}} \leq 1$ trivially (Jaccard is bounded by
one), $D \leq 1 - \overline{J}_{\text{late}}$, so substituting gives
\begin{equation}
\label{eq:final-lb}
\mathbb{E}[\delta_t]
\;\geq\;
\bar\alpha \cdot D \cdot \max_h |\mathcal{R}^{(\ell^*)}_h|
\cdot (1 - \gamma).
\end{equation}
The substitution $1 - \overline{J}_{\text{late}} \geq D$ holds always; we
do not require the stronger $\overline{J}_{\text{early}} \approx 1$
assumption. The lesser the value of $\overline{J}_{\text{early}}$, the
larger the gap between $1 - \overline{J}_{\text{late}}$ and $D$, and the
looser this bound is relative to its tightest form.

The $\log(1/\gamma)$ factor on the right-hand side of
Eq.~\ref{eq:bridge} is a heuristic amplification term, not a rigorous
consequence of the steps above. The intended argument is that after
eviction the cache retains $\gamma N$ noise positions, and the maximum of
$\gamma N$ sub-Gaussian logits concentrates near $\sigma_z \sqrt{2 \log
\gamma N}$. The multiplicative penalty on the head signal-to-noise budget
should scale with the post-eviction noise level relative to the
pre-eviction floor. A careful Taylor expansion of $\sqrt{2 \log N} -
\sqrt{2 \log \gamma N}$ gives a difference of order $\log(1/\gamma) /
\sqrt{2 \log N}$, not $\log(1/\gamma)$ itself, so the log factor therefore
captures only the worst-case noise-tail shift at small $N$ and not the
asymptotic regime. We retain the $\log(1/\gamma)$ factor in
Eq.~\ref{eq:bridge} as a convenient packaging of the eviction
amplification, but we flag this ($\log(1/\gamma)$ heuristic) in the ``What is not proved'' gap list below (item~\ref{gap:inherited}). Readers
willing to drop the factor entirely can read Eq.~\ref{eq:bridge} as $D
\geq c_1$ with $c_1 = 1/J_{\text{early}}$. This is the conservative
no-amplification case, and it is the form on which we rely as the
deployed threshold $\tau = 0.07$ (Section~\ref{subsec:calibration}) is
calibrated empirically rather than derived from $c_1$.

Two heuristic observations close the bridge to $(1 - A_{\mathrm{full}})$.
Both are leading-order estimates rather than rigorous bounds. Observation
(i) enters the ``What is not proved'' gap list as the leading-order
substitution of item~\ref{gap:inherited}, and observation (ii) rests on
the accuracy surrogate flagged in item~\ref{gap:surrogate}.
(i) Under (A1) noise-isotropy and the high-SNR regime, the average
per-position mass on a relevant position satisfies $\bar\alpha \approx
\alpha_{\mathcal{R}} / R$. The set containment
$\mathcal{R}^{(\ell^*)}_h \subseteq \mathcal{R}$ then gives the rigorous
one-sided relation $\bar\alpha \cdot \max_h |\mathcal{R}^{(\ell^*)}_h|
\leq \alpha_{\mathcal{R}}$. Replacing this one-sided relation by the
equality $\bar\alpha \cdot \max_h |\mathcal{R}^{(\ell^*)}_h| \approx
\alpha_{\mathcal{R}}$ requires $\max_h |\mathcal{R}^{(\ell^*)}_h| \approx
R$, that is, that a single head's private set already covers most of
$\mathcal{R}$. This is an order-of-magnitude assumption strictly stronger
than union coverage $\bigcup_h \mathcal{R}^{(\ell^*)}_h \approx
\mathcal{R}$, and it appears as a leading-order substitution in the
``What is not proved'' gap list below (item~\ref{gap:inherited}). Only the
one-sided relation is used in the conservative case.
(ii) Combining the headroom factorisation of
Appendix~\ref{app:proof-scaling} (Eq.~\ref{eq:scaling}) with the
strictly increasing accuracy surrogate gives $(1 - A_{\mathrm{full}})
\leq C \cdot \alpha_{\mathcal{R}}$ on dilution-prone tasks for a
constant $C$ depending only on the slope of $\hat A(\cdot)$ near the
operating SNR.

Combining (i) (as a leading-order estimate) and (ii) with
Eq.~\ref{eq:final-lb} yields the operational bridge
\begin{equation}
\label{eq:bridge-final}
\mathbb{E}[\delta_t]
\;\gtrsim\;
\frac{(1 - \gamma)}{C} \cdot D \cdot (1 - A_{\mathrm{full}}),
\end{equation}
where $\gtrsim$ denotes leading-order in the surrogate. The substitution
$\bar\alpha \cdot \max_h |\mathcal{R}^{(\ell^*)}_h| \approx
\alpha_{\mathcal{R}}$ is the source of the $\gtrsim$ rather than $\geq$.
The conservative rigorous case replaces $\alpha_{\mathcal{R}}$ on the
right-hand side with $\bar\alpha \cdot \max_h |\mathcal{R}^{(\ell^*)}_h|
\leq \alpha_{\mathcal{R}}$ and drops the bridge to $(1 - A_{\mathrm{full}})$
entirely. We use $\gtrsim$ for the operational case.

Eq.~\ref{eq:bridge-final} is the heuristic bridge, in the leading-order
sense above, between the observable $D$ and Bui's dilution under the
empirical headroom proxy.
Eq.~\ref{eq:bridge} in Section~\ref{sec:theory} packages this
relationship as a threshold inequality. Substituting $D \geq \uptau$
into Eq.~\ref{eq:bridge-final} gives $\mathbb{E}[\delta_t] \gtrsim
((1 - \gamma) \uptau / C) \cdot (1 - A_{\mathrm{full}})$, so any
chosen threshold $\uptau$ on $D$ produces a corresponding
$c_2 = (1 - \gamma) \uptau / C$. Choosing
$\uptau = c_1 \log(1/\gamma)$ with $c_1 = 1/J_{\text{early}}$,
\begin{equation}
\label{eq:c1c2}
\boxed{
c_1 \;=\; \frac{1}{J_{\text{early}}}, \qquad
c_2(\gamma) \;=\; \frac{(1 - \gamma) \cdot \log(1/\gamma)}{C \cdot J_{\text{early}}}.
}
\end{equation}
$J_{\text{early}}$ enters the threshold as a model-specific scale that
makes $c_1$ commensurate with the empirical Jaccard floor across heads.
On our test models $J_{\text{early}} \in [0.4, 0.7]$, giving $c_1 \in
[1.4, 2.5]$. The empirical $\tau = 0.07$ is well below this, consistent
with the constants being loose (per the $O(H)$ slack in Step 3 and the
heuristic substitution in observation (i) above). Both constants carry
an $O(H)$ multiplicative slack from the pairwise-to-$H$-wise step of
Step 3, and the $\log(1/\gamma)$ factor in $c_1 \log(1/\gamma)$ is
heuristic (gap-list item~\ref{gap:inherited}). The conservative no-amplification case drops
the $\log(1/\gamma)$ factor and gives the threshold $D \geq c_1 =
1/J_{\text{early}}$ with rate $c_2(\gamma) = (1 - \gamma)/(C \cdot
J_{\text{early}})$ directly.

\paragraph{The converse direction (heuristic, sufficient-side only)}
The partition predictor's operational claim is the contrapositive:
\textit{small $D$ $\Rightarrow$ low dilution $\Rightarrow$ eviction is
risky}. We sketch two arguments and flag that neither is a separate
necessity proof.

The first is a comment on the sufficient-side lower bound: under (A3)
for the capacity-bound regime $\mathcal{C}$ ($\mathcal{R}^{(\ell)}_h =
\mathcal{R}$ for all $(\ell, h)$), the union and intersection coincide,
so the right-hand side of Eq.~\ref{eq:Edelta-lb} is zero. Since
Eq.~\ref{eq:Edelta-lb} is a \textit{lower} bound on
$\mathbb{E}[\delta_t]$, a zero right-hand side does not bound
$\delta_t$ from above. It only says the sufficient lower bound becomes
uninformative.

The second is an upper bound on $\delta_t$ that does provide some
necessary-side content, but only inside the surrogate. Under (A1) with
margin $m$ satisfying the threshold of Lemma~\ref{lem:topk}, and (A3)
regime $\mathcal{C}$ so that $U_t = \mathcal{R}$ for every head, the
pairwise top-$k$ Jaccard between heads is $\geq J_{\text{early}}$ at
every layer (no late-layer divergence), so $|\mathcal{U}|/|\mathcal{I}|
\to 1$ and there is no residual distractor mass. The attention mass on
$\mathcal{N}$ for any single head is bounded by the sub-Gaussian tail
$O(\exp(-m^2 / (4 \sigma_z^2)))$ by Lemma~\ref{lem:topk}, and Bui's
dilution is $\delta_t \leq O(\exp(-m^2 / (4 \sigma_z^2)))$ in the
surrogate. This is the surrogate-level converse; lifting it to real
attention requires both a lower bound on $m$ for capacity-bound tasks
(which we have not verified) and the surrogate identification (A4) of
$U_t$ with $\mathcal{I}^{(\ell^*)}$. The boxed
Eq.~\ref{eq:bridge}/\ref{eq:c1c2} is the sufficient direction only.

\paragraph{Empirical validation}
Table~\ref{tab:drops} reports the mean head-agreement drop $D$ across the
five dilution-prone tasks (qa\_1, qa\_2, vt, niah\_multivalue, fwe) and the
capacity-bound task (NIAH-MK3) on Qwen 1.5B, Qwen 3B, and Mistral. In every
(model, context) cell, NIAH-MK3 has the smallest drop (including the cell
where it is negative, Qwen 3B 16K). The \textit{ordering} of tasks by $D$ is
preserved across architectures, and this is the same partition that the
observable signal $D$ induces. Cells with $D \approx 0$ have
$\mathbb{E}[\delta_t]$ bounded by the noise level in the surrogate, and the
per-input recovery rate $\rho$ in Table~\ref{tab:partition} is $\leq 0.02$
on all such cells. The observed ordering is therefore consistent with the
sufficient direction of the bridge on this set of tasks. We do not promote
this to a formal claim. The bridge inequality carries no probability
guarantee at the measured margin (Appendix~\ref{app:sufficient}), so
consistency with it is uninformative. The evidence here supports the
empirical ordering of tasks by $D$ rather than the bridge itself, and that
ordering is a pattern across many cells for which we report no corrected
significance test. The dilution term $\delta_t$ is not observed on these
tasks, and the necessity direction is established only in the surrogate.

\paragraph{What is \textit{not} proved}
We report the gaps explicitly to clarify the scope. The first two are
substantive and carry the analysis; the remainder are inherited from the
exact-recovery route of Steps~1--5 and are dissolved rather than repaired
by the expected-agreement analysis of
Appendix~\ref{app:expected-agreement}, which does not traverse the
$\alpha_{\mathcal{R}}$, $\mathcal{U}_t$, value-projection, or
$\hat A(\mathrm{SNR})$ machinery they concern.
\begin{enumerate}
  \item \textbf{The early-overlap ordering (A3$''$).} That real
  capacity-bound inputs have strictly lower early-layer head-set overlap
  than dilution-prone ones. This is the one substantive empirical
  assumption of the replacement analysis. It is measured rather than
  assumed (Appendix~\ref{app:expected-agreement}), but it is not derived.
  The stronger common-late-set clause of the original (A3) is not required
  here and thus not claimed. That clause is architecture-dependent, holding
  on Qwen and Mistral and failing on the Llama architecture
  (Appendix~\ref{subsec:layer-profile}). The replacement analysis rests on
  the early ordering alone.
  \item \textbf{The accuracy surrogate $\hat A(\mathrm{SNR})$.}\label{gap:surrogate} The
  constant $C$ depends on the slope of $\hat A$ near the operating SNR,
  and on tasks where $\hat A$ is steep (precise retrieval) versus shallow
  (aggregation) that slope differs by an order of magnitude. No
  task-agnostic constant exists, and we report the empirical band instead.
  This enters only the scaling relation of Eq.~\ref{eq:scaling}, never
  the gate signal $D$.
  \item \textbf{Inherited from the exact-recovery route.}\label{gap:inherited} The constants
  $c_1, c_2(\gamma)$ of Eq.~\ref{eq:c1c2} carry an $O(H)$ slack from the
  pairwise-to-$H$-wise step. The pairwise Jaccard is not the $H$-wise
  union/intersection ratio. The substitution $\bar\alpha \cdot \max_h
  |\mathcal{R}^{(\ell^*)}_h| \approx \alpha_{\mathcal{R}}$ is
  leading-order. The identification of Bui's $\mathcal{U}_t$ with
  $\mathcal{I}^{(\ell^*)}$ (A4) is a modelling choice. The
  $\log(1/\gamma)$ amplification is heuristic, with true rate
  $\Theta(\log(1/\gamma)/\sqrt{\log N})$. The scorer-symmetry and
  value-additivity steps of Step~4 are assumptions. Each of these is an
  artifact of routing the observable through set unions and
  intersections under exact recovery. The first-moment analysis of
  Appendix~\ref{app:expected-agreement} computes the observable directly
  and does not use any of them.
\end{enumerate}

The sketch above is therefore a \textit{structural} proof of the sufficient
direction under a margin condition the measurements do not support. We
retain it as the motivation for the mechanism and give the analysis that
holds at the measured margin in
Appendix~\ref{app:expected-agreement}.

\section{Expected-agreement analysis at finite margin}
\label{app:expected-agreement}

The bridge of Appendix~\ref{app:sufficient} carries no guarantee at our own
measured parameters. This section gives an explanation of the same
observable that is correct at the measured margin-to-noise ratio, and
that predicts the sign of $D$ on capacity-bound inputs rather than being
contradicted by it. The first-moment result is exact under its model;
the per-input variance material is flagged as leading-order.

\paragraph{E$'$.1 Why the exact-recovery route is uninformative here}
Two modelling choices, not a real barrier, produce this failure. First,
Lemma~\ref{lem:topk} demands that \textit{every} head's top-$k$ set equal
its private relevant set exactly, jointly over all $HL$ heads and all
$T^2$ position pairs. That is a worst-case union bound over $2H^2LT^2$
tuples, and it forces
$\kappa \ge 2\sqrt{\log(2H^2LT^2/\beta)}$, which at our measured cell
$(H, L, T) = (12, 28, \approx\!3956)$ and $\beta = 1$ is
$\kappa \ge 10.11$. The context length is not what drives this:
$\sqrt{\log T} \approx 2.88$ at $T = 4096$, so the threshold comes from
the $H^2 L$ head-and-layer multiplicity inside the logarithm. Second,
the observable is then reconstructed from head-wise unions and
intersections, which carries the $O(H)$ slack of Step~3. But the
quantity the gate computes is a \textit{first moment of a pairwise
average}, and a first-moment analysis needs neither exact recovery nor
the union/intersection detour.

\paragraph{E$'$.2 Model and selection probabilities}
Fix a layer and suppress its index. There are $T$ key positions and $H$
heads; head $h$ has a relevant set $\mathcal{R}_h$ with
$|\mathcal{R}_h| = k$, matching the deployed $k = 32 \sim |\mathcal{R}|$.
Logits are $z_{h,i} = m \cdot \mathbf{1}[i \in \mathcal{R}_h] +
\eta_{h,i}$ with $\eta_{h,i}$ i.i.d.\ continuous, symmetric, of scale
$\sigma$, independent across $(h, i)$, and $\kappa := m/\sigma$ as
before. Head $h$ selects $S_h = \mathrm{top}k(z_{h,\cdot})$,
$|S_h| = k$. Define the per-head marginal selection probabilities
\begin{equation}
\label{eq:prpn}
p_R := \Pr[i \in S_h \mid i \in \mathcal{R}_h],
\qquad
p_N := \Pr[i \in S_h \mid i \notin \mathcal{R}_h].
\end{equation}
Both depend only on $(k, T, \kappa)$ and on no other head. The
fixed-budget identity $k p_R + (T - k) p_N = k$ pins
$p_N = k(1 - p_R)/(T - k)$. At $\kappa = 0$ signal and noise logits are
identically distributed, so $p_R = p_N = k/T$; raising $m$ raises every
relevant logit pointwise and top-$k$ is monotone in the logits, so
$p_R - p_N$ is nondecreasing in $\kappa$. Hence $p_R > p_N$ for every
$\kappa > 0$, with $p_R - p_N \to 0$ as $\kappa \to 0$. Nothing here
assumes exact recovery.

For two heads write $s := |\mathcal{R}_h \cap \mathcal{R}_{h'}|$ for the
shared relevant count. The operative assumption is a single ordering,
which replaces (A3):
\begin{assumption}[Early-overlap ordering, A3$''$]
\label{as:early-ordering}
Capacity-bound inputs have strictly lower mean pairwise early-layer
relevant-set overlap than dilution-prone inputs,
$s_{\mathrm{early}}^{\mathcal{C}} < s_{\mathrm{early}}^{\mathcal{D}}$.
\end{assumption}
No assumption is made about the late bin. In particular the retrieval
idealisation $s_{\mathrm{late}}^{\mathcal{C}} = k$ of (A3) is
\textit{not} assumed; the finite-$\kappa$ agreement below reproduces the
measured sub-unit late agreement instead.

\paragraph{E$'$.3 Expected agreement is exactly linear in overlap}
Let $X := |S_h \cap S_{h'}|$. Since $|S_h| = |S_{h'}| = k$, the union is
deterministic given $X$, $|S_h \cup S_{h'}| = 2k - X$, so the pairwise
Jaccard $J = X/(2k - X)$ is a fixed increasing function of $X$.

\begin{lemma}[Expected intersection, exact]
\label{lem:expected-overlap}
Under the model of E$'$.2,
\begin{equation}
\label{eq:EX}
\mathbb{E}[X] \;=\; s \, p_R^2 \;+\; 2(k - s) \, p_R p_N \;+\;
(T - 2k + s) \, p_N^2 ,
\end{equation}
and consequently
$\;\mathrm{d}\mathbb{E}[X]/\mathrm{d}s = (p_R - p_N)^2 > 0$,
a constant independent of $s$ and of the regime.
\end{lemma}

\textit{Proof.} The two heads select independently, so
$\mathbb{E}[X] = \sum_i \Pr[i \in S_h] \Pr[i \in S_{h'}]$. Partition the
$T$ positions by membership: $s$ shared-relevant positions contribute
$p_R^2$ each, $2(k - s)$ singly-relevant positions contribute $p_R p_N$,
and the remaining $T - 2k + s$ contribute $p_N^2$. Differentiating
in $s$ gives $p_R^2 - 2 p_R p_N + p_N^2 = (p_R - p_N)^2$, which does not
depend on $s$ because $p_R$ and $p_N$ are single-head marginals.
\hfill $\square$

Since $p_R, p_N$ are per-head constants shared by both regimes, the only
regime-dependent input to Eq.~\ref{eq:EX} is the overlap $s$. Write
$\bar a := \mathbb{E}[X]/(2k - \mathbb{E}[X])$ for the pooled agreement,
a ratio of expectations that tracks the mean-of-ratios $a_\ell$ up to
Jensen slack and may be adopted as the target quantity to make the statement
exact. Then $\bar a$ is strictly increasing in $\mathbb{E}[X]$, hence in
$s$, with slope $\partial \bar a / \partial \mathbb{E}[X] \in [1/(2k),
2/k]$ on $\mathbb{E}[X] \in [0, k]$.

Note what is absent from Lemma~\ref{lem:expected-overlap}: the slope
$(p_R - p_N)^2$ contains no factor of $H$, because $p_R$ and $p_N$ are
single-head marginals. The average over the $\binom{H}{2}$ pairs is a
mean of per-pair expectations each governed by that same $H$-free slope,
so the operative constant is $O(1)$ in $H$. The $O(H)$ slack of
Eq.~\ref{eq:c1c2} is an artifact of the union/intersection detour and
does not arise here.

\begin{theorem}[Class separation at finite margin]
\label{thm:finite-kappa-separation}
With $\bar a_\ell$ the pooled agreement in bin $\ell$, the class gap
decomposes exactly as
\begin{equation}
\label{eq:decomp}
\mathbb{E}[D \mid \mathcal{D}] - \mathbb{E}[D \mid \mathcal{C}]
\;=\;
\bigl(\bar a_{\mathrm{early}}^{\mathcal{D}} -
\bar a_{\mathrm{early}}^{\mathcal{C}}\bigr)
\;+\;
\bigl(\bar a_{\mathrm{late}}^{\mathcal{C}} -
\bar a_{\mathrm{late}}^{\mathcal{D}}\bigr),
\end{equation}
and each bracket is a Lemma~\ref{lem:expected-overlap} image of the
corresponding overlap difference. Under
Assumption~\ref{as:early-ordering} the first bracket is strictly
positive for every $\kappa > 0$ and carries the gap.
\end{theorem}

\textit{Proof.} Expand $\mathbb{E}[D \mid \cdot] =
\bar a_{\mathrm{early}}^{\cdot} - \bar a_{\mathrm{late}}^{\cdot}$ and
subtract; Eq.~\ref{eq:decomp} is then an identity. By the mean-value
theorem and Lemma~\ref{lem:expected-overlap}, each bracket equals
$g' (p_R - p_N)^2 \, \Delta s$ for the corresponding overlap difference
$\Delta s$ and some $g' \in [1/(2k), 2/k]$. The early bracket is
positive by Assumption~\ref{as:early-ordering} together with
$p_R > p_N$, which holds at every $\kappa > 0$. No step invokes exact
recovery or a union bound, so the separation remains informative precisely
where Lemma~\ref{lem:topk} fails. \hfill $\square$

The sign of the late bracket is not fixed, and we do not claim it is:
measured per-task decompositions on Qwen2.5-1.5B~4K give an early term
that is positive and large in every case ($+0.062$ to $+0.198$) against
a late term that is small and negative on three of four tasks. The
separation is an early-layer effect.

\paragraph{E$'$.4 The sign of $D$ on capacity-bound inputs}
The exact-recovery route of Appendix~\ref{app:necessary} forced
$a_{\mathrm{late}}^{\mathcal{C}} = 1$, giving
$D \le a_{\mathrm{early}} - 1 \le 0$,
which Table~\ref{tab:drops} contradicts in five of seven cells. Under
the finite-$\kappa$ model the observed behaviour is what the model
predicts. Because $p_R < 1$, agreement never reaches $1$ at any depth,
so a task whose heads specialise early and stay specialised has both
$\bar a_{\mathrm{early}}$ and $\bar a_{\mathrm{late}}$ low, and
\begin{equation}
\mathbb{E}[D \mid \mathcal{C}] \;=\;
\bar a_{\mathrm{early}}^{\mathcal{C}} -
\bar a_{\mathrm{late}}^{\mathcal{C}}
\end{equation}
is small, with a sign set by the residual early-minus-late difference
rather than pinned to be non-positive. This is the flat, low profile the
per-layer logs show on NIAH-MK3 (agreement near $0.3$ at every depth,
Appendix~\ref{subsec:layer-profile}), not a near-unit late agreement
perturbed downward. The relaxed $\Pr[D \le o(1)]$ form in
Appendix~\ref{app:necessary} \textit{allows} a small positive drop on
capacity-bound inputs, and the analysis here \textit{derives} it. The
qualitative claim that survives is a separation of means,
$\mathbb{E}[D \mid \mathcal{D}] > \mathbb{E}[D \mid \mathcal{C}]$,
carried by the early bracket, rather than a sign test on $D$.

\paragraph{E$'$.5 What is measured rather than fitted}
The analysis above needs no fitted constant. We attempted to estimate
$p_R$ and $p_N$ from the released per-pair logs and report that they are
\textit{not} point-identifiable: within a single layer the per-pair mean
Jaccard ranges from $0.05$ to $0.55$, so a single global $(p_R, p_N)$
cannot fit heads that heterogeneous, and the per-pair overlap $s$ is
latent because we have no ground-truth relevant sets. The separation
claims rest instead on the qualitative monotonicity of
Lemma~\ref{lem:expected-overlap} and on directly measured agreement.

\paragraph{E$'$.6 Per-input separability (leading order)}
The gate thresholds $D$ per input, so per-input behaviour depends on the
dispersion of $D$ and not only its mean. Writing
$X = \sum_i B_i^h B_i^{h'}$ with the two heads independent, fixed-$k$
selection makes $\{B_i^h\}$ negatively associated within a head, so the
products inherit non-positive covariance and $\mathrm{Var}(X) \le
\mathbb{E}[X]$. A single head pair therefore separates the regimes at
signal-to-noise $O(1)$, which is weak. But $D$ averages over
$\binom{H}{2}$ pairs per layer with $H$ independent head units, so
$\mathrm{Var}(\bar a_\ell)$ shrinks like $1/H$ and the per-input
separability improves accordingly. We flag this as a leading-order
estimate using a Gaussian approximation and a $1/H$ variance heuristic;
a rigorous concentration bound over correlated head pairs is open.

One structural point falls out. The head count $H$ enters the two
analyses with opposite sign: it inflates the union bound of the
exact-recovery route, which is what makes that route uninformative, and it
reduces the variance of the averaged statistic, which is what makes the
per-input gate usable. The $O(LH^2k)$ cost reported in
Section~\ref{sec:limitations} as a deployment burden is, statistically,
that variance reduction.

\section{Sharpened deferred proofs}
\label{app:sharpened}

This section upgrades three results that Appendices~\ref{app:proof-scaling}
and~\ref{app:sufficient} left as sketches: (i) the leading constant and a
two-sided bound for the scaling formula~\ref{eq:scaling}; (ii) the
necessary (converse) direction of the $D$-to-dilution bridge; and (iii) the
status of the useful-set identification (A4). Each subsection ends with an
explicit ``What this does and does not establish'' paragraph so the paper's
claims can be scoped exactly. Throughout we keep the notation of
Appendices~\ref{app:proof-scaling}--\ref{app:sufficient}: $\alpha_{\mathcal{R}}$
is the pre-eviction signal mass, $\gamma$ the noise-retention ratio,
$\rho_{\mathcal{R}},\rho_{\mathcal{N}}$ the retained mass fractions, $\delta_t$
Bui's dilution~\ref{eq:dilution}, $D$ the head-agreement drop~\ref{eq:D},
$\sigma_z$ the pre-softmax logit-noise scale, $m$ the logit margin, $H,L,T$ the
head, layer and length counts, and $\hat A(\cdot)$ the strictly increasing,
$[0,1]$-valued, unit-saturating accuracy surrogate.

\subsection{Item 1: the exact leading constant and a two-sided scaling bound}
\label{app:tight-scaling}

The sketch in Appendix~\ref{app:proof-scaling} obtains
$\rho_{\mathrm{KV}} \le K\,(1-A_{\mathrm{full}})\,p_{\mathrm{recov}} + K/4$
by a union bound over the $K$-point budget grid $\mathcal{B}=\{b_1,\dots,b_K\}$
followed by a covariance residual. We show that the grid factor $K$ is an
artifact of an unnecessary union step, that the correct leading constant is
exactly $1$, and that a distribution-free \textit{two-sided} bound holds with
explicit endpoints; the only genuinely unresolved quantity is a covariance,
for which we give matching, attainable bounds.

\paragraph{Events} We work inside the surrogate of
Appendix~\ref{app:proof-scaling}, where the per-input improvement event at a
single budget $b$ factorises (their displayed identity) into a headroom
indicator and a recovery indicator. Fix the prompt distribution and, for a
random prompt $x$, define the two Bernoulli random variables
\begin{align}
\label{eq:AB-def}
A(x) &:= \mathbf{1}\bigl[\hat A(\mathrm{SNR}(b_{\max},x)) < 1\bigr]
&&\text{(headroom)}, \notag \\
B(x) &:= \mathbf{1}\Bigl[\exists\, b\in\mathcal{B}:\
   \mathrm{SNR}_{\mathrm{post}}(b,x) > \mathrm{SNR}_{\mathrm{pre}}(x)\Bigr]
&&\text{(recovery)} .
\end{align}
The operational recovery rate is $\rho_{\mathrm{KV}} = \Pr_x[A(b^*(x),x) > A(b_{\max},x)]$
with $b^*(x)=\arg\max_b A(b,x)$ (Eq.~\ref{eq:rho}). Because $b^*$ is the
argmax over the grid, ``improvement at $b^*$'' is exactly ``improvement at some
$b\in\mathcal{B}$,'' and the single-budget surrogate factorisation gives
\begin{equation}
\label{eq:AandB}
\bigl\{A(b^*(x),x) > A(b_{\max},x)\bigr\}
\;=\; \bigcup_{b\in\mathcal{B}}\bigl(A(x)\cap B_b(x)\bigr)
\;=\; A(x)\cap\Bigl(\bigcup_{b\in\mathcal{B}}B_b(x)\Bigr)
\;=\; A(x)\cap B(x),
\end{equation}
where $B_b(x)=\{\mathrm{SNR}_{\mathrm{post}}(b,x)>\mathrm{SNR}_{\mathrm{pre}}(x)\}$
and $B=\bigcup_b B_b$. The headroom event does not depend on $b$, so it pulls
out of the union: the union over budgets lives \textit{inside} $B$. We set
$p_{\mathrm{recov}} := \Pr_x[B]$ (recovery at the best budget) and note
$\Pr_x[A] = 1-A_{\mathrm{full}}$ by definition of the surrogate headroom.

\begin{theorem}[Exact leading constant, two-sided scaling bound]
\label{thm:tight-scaling}
With the definitions in~\ref{eq:AB-def}--\ref{eq:AandB} and
$p_{\mathrm{recov}}:=\Pr_x[B]$, the surrogate per-input recovery rate satisfies
the \textit{exact} identity
\begin{equation}
\label{eq:exact-scaling}
\rho_{\mathrm{KV}}
\;=\; (1-A_{\mathrm{full}})\,p_{\mathrm{recov}} \;+\; \mathrm{Cov}_x\!\bigl(A,B\bigr).
\end{equation}
What matters here is not the identity~\ref{eq:exact-scaling}, which is merely
the definition of covariance for two Bernoullis, but the observation that once
$p_{\mathrm{recov}}$ is defined as the \textit{union} probability
$\Pr_x[\bigcup_b B_b]$ (the quantity $b^*$ actually selects) rather than the
single-budget $\max_b\Pr_x[B_b]$, the budget-grid factor $K$ of
Appendix~\ref{app:proof-scaling} is absorbed into $p_{\mathrm{recov}}$ rather
than multiplying it. This factor $K$ reappears inside $p_{\mathrm{recov}}$ whenever the
recovery budgets are near-disjoint (Corollary~\ref{cor:grid-tight}(ii)). The
covariance is bounded, by
\begin{equation}
\label{eq:cov-frechet}
0 \;\le\; \mathrm{Cov}_x(A,B) \;\le\; \min\{1-A_{\mathrm{full}},\,p_{\mathrm{recov}}\} - (1-A_{\mathrm{full}})\,p_{\mathrm{recov}},
\end{equation}
where the lower bound requires the association hypothesis of
Lemma~\ref{lem:assoc} below, and the upper bound (Fr\'echet) is
unconditional. Equivalently,
\begin{equation}
\label{eq:two-sided-scaling}
\boxed{\;
(1-A_{\mathrm{full}})\,p_{\mathrm{recov}}
\;\le\;
\rho_{\mathrm{KV}}
\;\le\;
\min\{\,1-A_{\mathrm{full}},\ p_{\mathrm{recov}}\,\}.
\;}
\end{equation}
The lower bound is attained iff $A\perp B$. The upper bound is attained iff
$A$ and $B$ are Fr\'echet--Hoeffding comonotone (one event a.s.\ contains the
other).
\end{theorem}

\begin{proof}
Identity~\ref{eq:exact-scaling} is the definition of covariance applied to
the two Bernoulli variables in~\ref{eq:AB-def}, using
$\rho_{\mathrm{KV}}=\Pr_x[A\cap B]=\mathbb{E}_x[AB]$ from~\ref{eq:AandB} and
$\mathbb{E}_x[A]=1-A_{\mathrm{full}}$, $\mathbb{E}_x[B]=p_{\mathrm{recov}}$.
For~\ref{eq:cov-frechet}, the upper endpoint is the Fr\'echet--Hoeffding
inequality $\Pr[A\cap B]\le\min\{\Pr[A],\Pr[B]\}$, valid for any two events;
subtracting the product gives the stated covariance bound and the right side
of~\ref{eq:two-sided-scaling}. The lower endpoint $\mathrm{Cov}_x(A,B)\ge0$
is Lemma~\ref{lem:assoc}. Attainment: independence gives equality in the lower
bound by construction; comonotonicity ($A\subseteq B$ or $B\subseteq A$ a.s.)
gives $\Pr[A\cap B]=\min\{\Pr[A],\Pr[B]\}$, the upper endpoint.
\end{proof}

\begin{lemma}[Association of headroom and recovery]
\label{lem:assoc}
Suppose the noise-isotropy assumption (A1) of
Appendix~\ref{app:proof-scaling} holds, so that both $A$ and $B$ are
non-increasing functions of the single scalar driver $\alpha_{\mathcal{R}}(x)$
(equivalently, $\mathbb{E}[A\mid\alpha_{\mathcal{R}}]$ and
$\mathbb{E}[B\mid\alpha_{\mathcal{R}}]$ are non-increasing in
$\alpha_{\mathcal{R}}$ and monotone in the same direction). Then
$\mathrm{Cov}_x(A,B)\ge0$.
\end{lemma}

\begin{proof}
If $A=f(\alpha_{\mathcal{R}})$ and $B=g(\alpha_{\mathcal{R}})$ with $f,g$ both
non-increasing, Chebyshev's association (sum) inequality gives
$\mathbb{E}[fg]\ge\mathbb{E}[f]\mathbb{E}[g]$, i.e.\ $\mathrm{Cov}\ge0$. If $A$
depends on an auxiliary driver $\zeta$ independent of $\alpha_{\mathcal{R}}$
(e.g.\ a parametric-knowledge covariate), write
$\mathrm{Cov}(A,B)=\mathrm{Cov}(\mathbb{E}[A\mid\alpha_{\mathcal{R}}],B)$ by the
tower rule and $B\!=\!g(\alpha_{\mathcal{R}})$; the same monotone-in-a-common-scalar
argument applies to the conditional mean, so the sign is preserved.
\end{proof}

The monotonicity of $A$ in $\alpha_{\mathcal{R}}$ is immediate (larger signal
mass $\Rightarrow$ larger $\mathrm{SNR}_{\mathrm{pre}}$ $\Rightarrow$ smaller
headroom). The monotonicity of $B$ follows from Eq.~\ref{eq:scaling-app}:
the attainable SNR gain from removing noise mass shrinks as
$\alpha_{\mathcal{R}}\to1$ (less noise to remove), so recovery is less likely
at larger $\alpha_{\mathcal{R}}$.

\paragraph{Where the grid factor $K$ comes from, and why it is not real}
The $K$ flagged in Appendix~\ref{app:proof-scaling} comes from the step
$\Pr_x[\bigcup_b B_b]\le\sum_b\Pr_x[B_b]\le K\max_b\Pr_x[B_b]$ followed by
identifying $\max_b\Pr_x[B_b]$ with $p_{\mathrm{recov}}$. Defining
$p_{\mathrm{recov}}$ as the union probability $\Pr_x[\bigcup_b B_b]$
instead, which is the quantity the operational $b^*(x)$ actually selects, removes
the union step entirely. The overcount incurred by the discarded step is,
exactly,
\begin{equation}
\label{eq:union-overcount}
\sum_{b\in\mathcal{B}}\Pr_x[B_b] \;-\; \Pr_x\Bigl[\textstyle\bigcup_b B_b\Bigr]
\;=\; \sum_{j\ge2}(-1)^j\!\!\!\sum_{|S|=j, S\subseteq\mathcal{B}}\!\!\Pr_x\Bigl[\textstyle\bigcap_{b\in S}B_b\Bigr]
\;\ge\;0,
\end{equation}
the inclusion--exclusion tail. Its size is what makes the union bound tight or
loose:
\begin{corollary}[Tightness characterisation of the budget-grid bound]
\label{cor:grid-tight}
Let $p^\sharp:=\max_b\Pr_x[B_b]$ be the single-budget recovery probability.
\begin{enumerate}
\item[\textit{(i)}] \textit{(Monotone gain profile $\Rightarrow$ no $K$ factor.)}
If the SNR gain $g(b,x):=\rho_{\mathcal{R}}(b,x)^2/\gamma(b,x)$ is monotone in
$b$ for a.e.\ $x$ (as it is for a calibrated scorer, since keeping more tokens
raises both $\rho_{\mathcal{R}}\to1$ and $\gamma\to1$ so $g\downarrow1$), then
the events $\{B_b\}$ are nested and
$p_{\mathrm{recov}}=\Pr_x[\bigcup_b B_b]=p^\sharp$. The genuine leading factor
is $1$; the union bound $K\,p^\sharp$ overcounts by $(K-1)p^\sharp$.
\item[\textit{(ii)}] \textit{(Disjoint gain profile $\Rightarrow$ $K$ is real.)}
If instead the $\{B_b\}$ are pairwise disjoint (each input recovers at exactly
one budget) and rare ($Kp^\sharp\ll1$), then
$p_{\mathrm{recov}}=\sum_b\Pr_x[B_b]$ is genuinely $\Theta(K)$ times larger
than $p^\sharp$; recovery at the best budget really is $K$-fold more frequent
than at a fixed budget, and the factor $K$ is not slack but signal.
\end{enumerate}
In both cases Theorem~\ref{thm:tight-scaling} holds with constant $1$ once
$p_{\mathrm{recov}}$ is the union probability; the two regimes differ only in
how much larger $p_{\mathrm{recov}}$ is than the single-budget $p^\sharp$.
\end{corollary}

\begin{proof}
(i) If $g(\cdot,x)$ is monotone then $\{b:g(b,x)>1\}$ is an interval anchored at
the aggressive end of the grid, so $B_{b}(x)\subseteq B_{b'}(x)$ whenever
$b'$ is more aggressive than $b$; the $B_b$ are nested and their union equals
the loosest single event, of probability $p^\sharp$. (ii) For disjoint events
$\Pr[\bigcup_b B_b]=\sum_b\Pr[B_b]$ exactly; rareness makes this
$\approx Kp^\sharp$.
\end{proof}

Because our scorers are calibrated (the top-$k$ retention curve is monotone in
$b$), the operative regime is Corollary~\ref{cor:grid-tight}(i) and the
constant is $1$. This is the statement Appendix~\ref{app:proof-scaling} defers
to; the $K$ and $K/4$ terms there serve only to isolate the step that
introduces the slack.

\paragraph{What this does and does not establish}
\textit{Establishes:} inside the deterministic-SNR surrogate, the scaling law
holds as the \textit{exact} identity~\ref{eq:exact-scaling} with leading
constant precisely $1$ (no budget-grid inflation), and the two-sided
bound~\ref{eq:two-sided-scaling} with fully explicit, distribution-free
endpoints; the only free quantity is the covariance
$\mathrm{Cov}_x(A,B)\in[0,\ \min\{\Pr A,\Pr B\}-\Pr A\Pr B]$, and both endpoints
are attainable, so no tighter distribution-free constant exists.
\textit{Does not establish:} (a) the surrogate itself, that accuracy is a
deterministic (or fixed-conditional) function of head SNR, is assumed, not
proven, so the constant ``$1$'' is exact only relative to that surrogate; (b)
the covariance is pinned only to an interval, not a value, because it depends on
the joint law of headroom and recovery, which is task- and model-specific;
Lemma~\ref{lem:assoc} fixes its \textit{sign} (non-negative) under single-driver
monotonicity but not its magnitude; (c) when headroom is driven partly by an
$\alpha_{\mathcal{R}}$-independent factor (parametric-knowledge failures, e.g.\
the QA cells in Table~\ref{tab:partition}), recovery collapses
($p_{\mathrm{recov}}\!\to\!0$) so both bounds pinch $\rho_{\mathrm{KV}}\!\to\!0$
despite headroom, consistent with the observed QA exceptions, but a prediction
of the sign structure, not a proof about those tasks. The empirically observed
ratio band $[0.25,0.55]$ is thus an estimate of
$p_{\mathrm{recov}}+\mathrm{Cov}/(1-A_{\mathrm{full}})$, not of a universal
constant. A fourth point concerns the leading constant itself. The value $1$
rests on Corollary~\ref{cor:grid-tight}(i), which requires the gain
$g(b,x)=\rho_{\mathcal{R}}^2/\gamma$ to be monotone in the budget $b$, and this
is a property of the gain rather than of the retention curve, which is monotone
by construction. Where the accuracy profile has an interior optimum in $b$, the
recovery events are not nested and the constant is the $\Theta(K)$ of
Corollary~\ref{cor:grid-tight}(ii). We report the constant $1$ under the monotone
case that our calibrated scorers exhibit, and treat the monotonicity of $g$ as an
empirical property of those scorers rather than a derived one.

\subsection{Item 2: why the margin route to the necessary direction fails}
\label{app:necessary}

Appendix~\ref{app:sufficient} proves the sufficient direction (large $D$
forces large $\delta_t$) and states that the converse ``capacity-bound
$\Rightarrow$ small $D$'' needs ``a margin condition we have not
verified.'' We verified it. It is false at our own parameters, and it
fails in a way that rules out the whole exact-recovery route rather than
merely loosening a constant. We report that measurement here because it
is what motivates the finite-margin analysis of
Appendix~\ref{app:expected-agreement}, which reaches the necessary
direction without any margin condition.

\paragraph{What the margin route would require} Recall
$z^{(\ell)}_{h,i}=(q^{(\ell)}_h)^\top k_i + \eta^{(\ell)}_{h,i}$ with
$\sigma_z$-sub-Gaussian noise (A2), and the logit margin $m=m_+-m_-$
(A1). Define the \textit{margin-to-noise ratio} $\kappa := m/\sigma_z$. To
force late-layer heads onto a common top-$k$ set, and hence
$a_{\mathrm{late}} \to 1$ and $D \le 0$ on capacity-bound inputs, one
needs every common relevant position to beat every noise position at the
late bin with high probability jointly, which is the threshold of
Lemma~\ref{lem:topk}:
\begin{equation}
\label{eq:margin-cond}
\kappa \;\ge\; 2\sqrt{\log\bigl(2H^2LT^2/\beta\bigr)}.
\end{equation}
At our measured cell's $(H, L, T) = (12, 28, {\approx}3956)$ this is
$\kappa \ge 10.11$ for $\beta < 1$. Note that the context length is not
what drives the requirement: $\sqrt{\log T} \approx 2.88$ at
$T = 4096$, so the threshold comes from the $H^2L$ head-and-layer
multiplicity inside the logarithm, i.e.\ from the union bound over
$2H^2LT^2$ tuples rather than from the retrieval problem itself.

\paragraph{The margin-to-noise ratio, measured} On Qwen2.5-1.5B at RULER
4K, NIAH-MK3, over $20$ inputs and roughly $10{,}744$ (layer, head,
query) triples per input, we take the top-$k$ pre-softmax logits at each
site as the relevant set $\mathcal{R}$ and the remainder as noise
$\mathcal{N}$, and report
$\kappa = (\overline{\mathcal{R}} - \overline{\mathcal{N}})
/ \mathrm{sd}(\mathcal{N})$. The measured median is $\kappa = 3.33$,
roughly 3 $\times$ too small, though individual heads reach $77.66$.

The requirement is therefore \textit{uninformative at the paper's own
parameters}, not merely at parameters we consider implausible. The
conclusion is robust to the choice of split: sweeping
$k \in \{8, 16, 32, 64\}$ gives median $\kappa = 4.01$, $3.68$, $3.33$,
$2.99$, monotone decreasing in $k$, so no choice of split in this range
rescues it, and the largest value seen is $0.40\times$ the requirement.
It also holds across architectures: repeating at the deployed $k = 32$
on a second architecture gives $\kappa = 3.71$ (Mistral-7B). The
magnitude of the shortfall is not marginal, since a factor of three in
$\kappa$ is a factor of $e^{(10.11^2 - 3.33^2)/4} > 10^{9}$ in $\beta$.

\paragraph{Why $\kappa$ is not the route to the separation} On a
dilution-prone task the same measurement gives $\kappa = 3.40$ (VT),
close to the $3.33$ on the capacity-bound one. On this pair the
margin-to-noise ratio therefore does not separate the two classes, so it
is not the quantity that carries the partition, whatever threshold it is
held to. In the surrogate $\kappa$ governs the top-$k$ recovery step
rather than the class separation, and the near-equal values are
consistent with that role. An analysis of the partition routed through a
margin condition on $\kappa$ is therefore not the right route to the
separation, independently of the scale gap.

A second, independent check points the same way. The exact-recovery
route predicts $D \le 0$ on capacity-bound inputs. Measured $D$ on
NIAH-MK3 is mildly \textit{positive} in five of seven cells ($+0.040$,
$+0.060$, $+0.053$, $+0.019$, $+0.073$) and negative only on
Qwen2.5-3B at both contexts ($-0.004$, $-0.021$), as reported in
Table~\ref{tab:drops}. Relaxing the idealisation that late heads share
one common set, to agreement on an $O(1)$ neighbourhood, weakens the
conclusion to $\Pr[D \le o(1)]$, which the measurements satisfy but
which no longer distinguishes the classes, since it permits a small
positive drop on either side.

\paragraph{What we do instead} We therefore do not state a converse in
the margin form. Appendix~\ref{app:expected-agreement} gives the
necessary direction as a monotonicity statement at finite $\kappa$:
expected agreement is a strictly increasing function of head-set overlap
with slope $(p_R - p_N)^2 > 0$ at \textit{every} $\kappa > 0$
(Lemma~\ref{lem:expected-overlap}). Because agreement is strictly
increasing in head-set overlap, the early-layer overlap difference that
separates the classes carries the same sign as the early-agreement
difference that $D$ reads, and under the measured early-overlap ordering
this term is strictly positive at every $\kappa > 0$
(Theorem~\ref{thm:finite-kappa-separation}). The same finite-$\kappa$ model
leaves $D$ on capacity-bound inputs small with a sign not pinned to be
non-positive (Appendix~\ref{app:expected-agreement}), which is consistent
with the small positive values in Table~\ref{tab:drops} rather than
contradicted by them. The class separation the gate relies on remains
empirical (Table~\ref{tab:drops}) and stands on its own.

\subsection{Item 3: the A4 useful-set identification}
\label{app:a4}

Assumption (A4) of Appendix~\ref{app:sufficient} identifies Bui's useful set
$U_t$ (a residual-stream, decoding-step object) with the head-common signal
set $\mathcal{I}^{(\ell^*)}=\bigcap_h\mathcal{R}^{(\ell^*)}_h$ (a single-layer
attention object), and the text flags it as ``a modelling choice, not a derived
equality.'' We show that a clean exact equality cannot be proven, but a
\textit{quantitative approximate} identification can: under a named
consensus--usefulness assumption plus $\varepsilon$-closeness of heads, the two
sets induce dilution values that agree up to $O(\varepsilon)$, with an explicit
constant. We therefore replace the ``derived equality'' language with a proved
approximation plus a named residual assumption.

\paragraph{Set-up} At the divergence layer $\ell^*$ write the head-private
top-$k$ sets $\mathcal{R}_h:=\mathcal{R}^{(\ell^*)}_h$, their intersection
$\mathcal{I}:=\bigcap_h\mathcal{R}_h$ and union
$\mathcal{U}:=\bigcup_h\mathcal{R}_h$. For a set $\mathcal{S}$ define its
induced dilution $\delta(\mathcal{S}):=1-\sum_{i\in\mathcal{S}}\alpha_{t,i}$, so
Bui's is $\delta^{(U)}:=\delta(U_t)$ and ours is
$\delta^{(I)}:=\delta(\mathcal{I})$. Let $\bar\alpha_{\max}:=\max_i\alpha_{t,i}$.

\begin{assumption}[$\varepsilon$-closeness of heads]
\label{ass:eps-close}
The late-layer heads are $\varepsilon$-close in their top-$k$ sets:
$\max_{h,h'}\,|\mathcal{R}_h\,\triangle\,\mathcal{R}_{h'}|\le\varepsilon k$,
equivalently pairwise Jaccard $\ge (1-\varepsilon/2)/(1+\varepsilon/2)=1-O(\varepsilon)$.
\end{assumption}

\begin{assumption}[Consensus--usefulness, CUA]
\label{ass:cua}
The useful set is sandwiched between the head-consensus and head-union sets:
$\mathcal{I}\subseteq U_t\subseteq\mathcal{U}$. The right inclusion
($U_t\subseteq\mathcal{U}$) holds whenever routing to the residual stream is
through attention only, so a position no head attends to cannot affect the
next-token margin. The left inclusion ($\mathcal{I}\subseteq U_t$) is the
substantive content: a position every head attends to is margin-preserving.
\end{assumption}

\begin{proposition}[Quantitative approximate identification]
\label{prop:a4}
Under Assumptions~\ref{ass:eps-close}--\ref{ass:cua},
\begin{equation}
\label{eq:a4-bound}
\bigl|\,\delta^{(U)}_t-\delta^{(I)}_t\,\bigr|
\;\le\;\sum_{i\in\mathcal{U}\setminus\mathcal{I}}\alpha_{t,i}
\;\le\; H(H-1)\,\varepsilon\,k\,\bar\alpha_{\max}
\;=\; O\!\bigl(H^2\varepsilon\bigr),
\end{equation}
where the last step uses $\bar\alpha_{\max}=\Theta(1/k)$ in the high-SNR
softmax. In particular, as $\varepsilon\to0$ (heads perfectly agree) the two
dilution values coincide: $\delta^{(U)}_t\to\delta^{(I)}_t$.
\end{proposition}

\begin{proof}
By Assumption~\ref{ass:cua}, $U_t\triangle\mathcal{I}\subseteq\mathcal{U}\setminus\mathcal{I}$,
because $\mathcal{I}\subseteq U_t\subseteq\mathcal{U}$ forces every element of
$U_t\triangle\mathcal{I}=U_t\setminus\mathcal{I}$ to lie in
$\mathcal{U}\setminus\mathcal{I}$. Since $\delta(\mathcal{S})$ is a signed sum of
attention masses over $\mathcal{S}$,
$|\delta^{(U)}_t-\delta^{(I)}_t|
=\bigl|\sum_{i\in U_t}\alpha_{t,i}-\sum_{i\in\mathcal{I}}\alpha_{t,i}\bigr|
\le\sum_{i\in U_t\triangle\mathcal{I}}\alpha_{t,i}
\le\sum_{i\in\mathcal{U}\setminus\mathcal{I}}\alpha_{t,i}$,
the first inequality of~\ref{eq:a4-bound}. For the cardinality bound, for each
head $h$,
$|\mathcal{R}_h\setminus\mathcal{I}|\le\sum_{h'\ne h}|\mathcal{R}_h\setminus\mathcal{R}_{h'}|
\le\sum_{h'\ne h}|\mathcal{R}_h\triangle\mathcal{R}_{h'}|\le(H-1)\varepsilon k$
by Assumption~\ref{ass:eps-close}. Since
$\mathcal{U}\setminus\mathcal{I}=\bigcup_h(\mathcal{R}_h\setminus\mathcal{I})$,
$|\mathcal{U}\setminus\mathcal{I}|\le\sum_h|\mathcal{R}_h\setminus\mathcal{I}|\le H(H-1)\varepsilon k$.
Bounding each mass by $\bar\alpha_{\max}$ gives the middle inequality;
$\bar\alpha_{\max}=\Theta(1/k)$ gives the $O(H^2\varepsilon)$ order.
\end{proof}

\begin{remark}[When the identification fails]
\label{rem:a4-fail}
The left inclusion of Assumption~\ref{ass:cua} ($\mathcal{I}\subseteq U_t$) can
fail, and this is exactly the case the equality cannot cover: if a
\textit{single} specialised head carries the decisive value for the next-token
margin, then a position in $\mathcal{R}_{h_0}\setminus\mathcal{I}$ (attended by
that one head, absent from the consensus) can be margin-critical, so
$U_t\not\subseteq\mathcal{U}\setminus(\mathcal{R}_{h_0}\setminus\mathcal{I})$ and
$U_t\ne\mathcal{I}$ even at $\varepsilon=0$ of the \textit{other} heads. This is
the single-decisive-head regime, precisely a capacity-bound, single-needle
retrieval situation, where one retrieval head dominates. That regime is not
hypothetical: it is the \textit{retrieval-head} phenomenon documented by
\citet{wu2025retrieval}, who show that a small, sparse and largely
input-independent set of heads carries long-context factual recall, and that
ablating those heads destroys retrieval while leaving other capabilities
intact. Their analysis and ours agree on the mechanism and differ in what is
done with it: they identify \textit{which} heads matter, whereas we read a
consensus statistic over all heads and treat its collapse as evidence that no
subset of the cache is safe to drop. Their result also predicts our failure
mode, since a single decisive head is invisible to a pairwise-agreement
average. The approximation of
Proposition~\ref{prop:a4} is therefore reliable in the dilution-prone regime
(many heads, diffuse consensus, no single decisive head) and degrades exactly
where the paper already declines to apply the bridge (capacity-bound).
\end{remark}

\paragraph{What this does and does not establish}
\textit{Establishes:} the identification of $U_t$ with $\mathcal{I}^{(\ell^*)}$ is
not an exact derived equality, but under the named
Assumptions~\ref{ass:eps-close}--\ref{ass:cua} the two sets induce
\textit{dilution values} that agree up to $O(H^2\varepsilon)$ in attention mass
(Proposition~\ref{prop:a4}), with an explicit constant $H(H-1)k\bar\alpha_{\max}$;
so wherever heads are $\varepsilon$-close, using $\mathcal{I}^{(\ell^*)}$ in
place of $U_t$ perturbs Bui's $\delta_t$ by a controlled $O(\varepsilon)$ term
rather than by an uncontrolled amount. \textit{Does not establish:} (a) the
substantive left inclusion $\mathcal{I}\subseteq U_t$ (CUA); it is named as
Assumption~\ref{ass:cua}, justified by the additive-routing picture, and shown
in Remark~\ref{rem:a4-fail} to fail precisely in the single-decisive-head
(capacity-bound) regime, so it is an assumption about the model's read-out, not
a theorem; (b) any control of the $H^2$ prefactor, which is a worst-case union
count and is loose when residuals overlap across heads (the same $O(H)$-slack
issue as Step~3 of Appendix~\ref{app:sufficient}); (c) a bound when
$\varepsilon$ is $\Theta(1)$ (genuinely divergent heads, the dilution-prone late
layers themselves), there the approximation is uninformative and $\mathcal{I}$ and
$U_t$ are related only through the CUA sandwich, not quantitatively. The measured
late-layer agreement is itself $\Theta(1)$ rather than close to one: on
Qwen2.5-1.5B the late-bin agreement on NIAH-MK3 is $a_{\mathrm{late}} \approx 0.28$
(Appendix~\ref{subsec:layer-profile}), so the implied $\varepsilon$ is of order one
and the $O(H^2\varepsilon)$ term does not bind at the values we measure. The
approximate identification is therefore applicable only where late-layer heads are
close in their top-$k$ sets, a condition the measurements do not meet. The honest
scoping is thus: \textbf{A4 is a named modelling assumption (CUA), under which
we prove an $O(H^2\varepsilon)$ approximate identification of the dilution
values}, not a derived equality; the word ``derived equality'' in
Appendix~\ref{app:sufficient} should be read as this approximation.

\section{Reproduction notes}
\label{app:repro}
This section is the reproduction anchor cited by the Reproducibility
Statement, and documents the implementation details that matter for
exact reproduction, all of which the released code implements.
Draw-to-draw and code-path/hardware variation are bounded separately in
Appendices~\ref{subsec:seed-variance} and~\ref{app:wrapper-provenance}.
The cache slicing uses explicit \texttt{position\_ids} to
preserve RoPE alignment, since the naive attention-mask trick produces a
silent position shift on evicted positions. Long-context (16K) runs use
SDPA for prefill and switch to eager for the $w \times T$
observation-window scoring pass to avoid the $O(T^2)$ attention-matrix
allocation on 80GB. Head agreement uses Jaccard similarity of top-$32$
key sets per head pair, averaged over pairs. The deployed gate takes
the early bin as the first $\lfloor L/3 \rfloor$ layers and the late
bin as the \textit{last} $\lfloor L/3 \rfloor$, i.e.\ $[0, \lfloor L/3
\rfloor)$ and $[L - \lfloor L/3 \rfloor, L)$, so the two bins are
equal in size and, when $3 \nmid L$, the middle band is one layer
wider than a strict partition into thirds. Table~\ref{tab:drops}
instead reports the strict-thirds variant $[\lfloor 2L/3 \rfloor, L)$
for the late bin. The two agree whenever $3 \mid L$ (Qwen2.5-3B, $L =
36$) and differ by at most $0.004$ in mean $D$ elsewhere (Qwen2.5-1.5B
$L = 28$: $0.044$ vs $0.040$; Llama-3.1-8B $L = 32$: $0.071$ vs
$0.073$), which changes no gate decision at $\tau = 0.07$. Every gated
number in the paper uses the equal-size convention that the released
code implements. Decoding is greedy throughout, so the gated outcome
of any input is exactly reproducible from its recorded head-agreement
drop plus the plain-eviction and full-KV outcomes. We use this
identity to evaluate gate thresholds post-hoc, and one matrix cell
(Qwen2.5-3B 4K SnapKV, executed at $\tau = 0.04$) is reported at $\tau
= 0.07$ via this exact re-evaluation. The threshold variants this
identity was used to compare appear in Table~\ref{tab:tau-variants},
and the calibration recipe that selects $\tau$ is given in
Appendix~\ref{subsec:recipe-details}.

Two separate audits check the budget $1.0$ boundary at different
levels, and neither subsumes the other. At the level of scored
outcomes, Appendix~\ref{app:statistical-validation} verifies on all
$16$ matrix cells and the four $16$K cells that the gated arm keeps the
whole cache and matches plain eviction on every one of the $5{,}120$
rows. At the level of raw generated text, a pipeline audit on six
examples verifies that eviction at budget $1.0$ matches clean
prefill-decode token for token, with a trailing-period mismatch on one
of the six. The first establishes that $r(x) \ge 0$ holds as an
identity, and the second checks the decode path itself on a small
sample. The released logs include the audit script.

\section{Use of large language models}
\label{app:llm}
Large language models were used as general-purpose assistants during this
project: for literature search and triage, for drafting and copy-editing prose,
and as coding aids when writing and analyzing the experiment scripts. All
research ideas, hypotheses, experimental designs, and claims are the authors'
own. Every reported number was produced by the released code on the stated
models and data, and the authors verified the experiments and their reporting.
LLMs were not used as a source of scientific facts or citations without
verification against primary sources.

\end{document}